\documentclass{article}
\usepackage[accepted]{icml2024}

\usepackage{microtype}
\usepackage{graphicx}
\usepackage{subfigure}
\usepackage{booktabs} 

\usepackage{hyperref}

\usepackage{amsmath}
\usepackage{amssymb}
\usepackage{mathtools}
\usepackage{amsthm}
\usepackage{mathrsfs}
\usepackage{nccmath}
\usepackage[textsize=tiny]{todonotes}
\usepackage[capitalize,noabbrev]{cleveref}

\DeclareMathOperator{\mint}{\medmath\int}
\DeclareMathOperator{\miint}{\medmath\int}

\theoremstyle{plain}
\newtheorem{theorem}{Theorem}[section]
\newtheorem{proposition}[theorem]{Proposition}
\newtheorem{lemma}[theorem]{Lemma}
\newtheorem*{lemma*}{Lemma}

\newtheorem{definition}[theorem]{Definition}
\theoremstyle{definition}

\theoremstyle{remark}

\icmltitlerunning{Kernel-Based Evaluation of Conditional Biological Sequence Models}

\begin{document}

\twocolumn[
\icmltitle{Kernel-Based Evaluation of Conditional Biological Sequence Models}

\begin{icmlauthorlist}
\icmlauthor{Pierre Glaser}{gatsby}
\icmlauthor{Steffanie Paul}{hmss}
\icmlauthor{Alissa M. Hummer}{hmss,oxford}
\icmlauthor{Charlotte M. Deane}{oxford}
\icmlauthor{Debora S. Marks}{hmsb}
\icmlauthor{Alan N. Amin}{nyu}
\end{icmlauthorlist}

\icmlaffiliation{gatsby}{Gatsby Computational Neuroscience Unit, London, UK}
\icmlaffiliation{hmss}{Systems Biology, Harvard Medical School, Boston, USA}
\icmlaffiliation{hmsb}{Harvard Medical School, Broad Institute, Boston, USA}
\icmlaffiliation{oxford}{Department of Statistics, University of Oxford, Oxford, UK}
\icmlaffiliation{nyu}{Courant Institute, New York University, New York, USA}

\icmlcorrespondingauthor{Pierre Glaser}{pierreglaser@gmail.com}

\icmlkeywords{Machine Learning, ICML}

\vskip 0.3in
]

\printAffiliationsAndNotice{}  

\begin{abstract}
    We propose a set of kernel-based tools to evaluate the designs and tune
    the hyperparameters of conditional sequence models, with a focus on problems
    in computational biology. The backbone of our tools is a new measure of
    discrepancy between the true conditional distribution and the model's estimate,
    called the Augmented Conditional Maximum Mean Discrepancy (ACMMD). Provided that
    the model can be sampled from, the ACMMD can be estimated unbiasedly from data
    to quantify absolute model fit, integrated within hypothesis tests, and 
    used to evaluate model reliability. We demonstrate the utility of our approach by analyzing a popular protein design model, ProteinMPNN.
    We are able to reject the hypothesis that ProteinMPNN fits its data for various protein families,
    and tune the model's temperature hyperparameter to achieve a better fit.
\end{abstract}

\vspace{-2em}

\section{Introduction} \label{sec:introduction}

Conditional sequence models constitute one of the most prominent model classes
of modern machine learning. Such models have allowed progress in  longstanding
problems in fields ranging from natural language generation to biomedical
applications such as genomics and protein design. Abstracting away the precise
nature of the data, the objective common to many of these problems can be
summarized as the prediction of high-dimensional discrete-valued sequences,
given some possibly high-dimensional input information about the sequence.
For example, in protein design, inverse folding models \cite{Dauparas2022pmpnn}
seek to learn the conditional distribution of amino acid sequences (proteins)
that are likely to fold to a given input protein backbone 3D geometry, or \emph{structure}.

In such problems, it is crucial to evaluate the properties of the trained model.
Model evaluation can help assess the risk of using the model's predictions
in the real world, such as performing in-vitro experiments (a time-intensive process),
guide hyperparameter searches, and deepen one's understanding of
the model's behavior. 
Two properties are particularly important to measure: the first one
is model \emph{accuracy}, which describes how well the model
approximates the true conditional distribution of the target variable given the input.
Models with high accuracy have learned the underlying structure of the data, suggesting
a high potential value in deploying them in real-world applications.
However, in practice, it is likely that models will not be perfectly accurate.
Inaccurate models can still be useful as long as they fall back to conservative guesses (in the extreme case, the prior distribution) when they are uncertain.
From a statistical perspective,
this property is known as \emph{reliability}
\cite{brocker2008some,vaicenavicius2019evaluating,widmann2022calibration},
and will be the second property of interest in this work.

Given a set of real samples, the standard approach to
evaluate models in protein design consists in using log-likelihoods or sequence recovery
\citep{Dauparas2022pmpnn, Hsu2022esmif1, Gao2022PiFold}. However, log-likelihoods
cannot be used to evaluate reliability, and are only \emph{relative} measures of accuracy:
these methods can only be used to compare
models and would not alert the practitioner for example if all models make very poor
predictions.
Instead, to assess how far a model is from being optimally accurate and consistent ---
and thus the potential value in improving it, by for example
collecting more data or increasing its complexity, one should consider \emph{absolute}
rather than \emph{relative} metrics, that is, metrics that not only allow one to compare models
to each other, but also to evaluate a single model's performance without any other point of comparison.
For these metrics to have
practical value, they should come with estimators computable from data samples.
These estimators should be efficiently computable, recover the true
metric as in the large sample size limit (i.e. be \emph{consistent}), and preferably
be centered around the true value of the metric (i.e. be \emph{unbiased}). Finally,
to factor out the statistical error coming from estimating these metrics using a finite
number of samples, these metrics should be integrable into hypothesis tests
built to detect \emph{statistically significant} mismatches between the model and the data.

{\bfseries Contributions}
In this work, we introduce a set of absolute evaluation
metrics for measuring the accuracy and the reliability of
conditional sequence models.
Both our metrics are grounded in a new measure of divergence between
conditional probability distributions, which we call the Augmented Conditional
Maximum Mean Discrepancy (ACMMD), which extends the kernel-based
conditional goodness-of-fit framework of \citet{jitkrittum2020testing, Glaser2023-sj, widmann2022calibration}
to the case of sequence-valued variables.
We analyze the statistical properties of our proposed metrics, which can be
estimated using samples from the data and the model. Under certain
conditions, we show that the ACMMD is able to detect any mismatch
between the model and the data. In addition, we integrate the ACMMD into
hypothesis tests to detect such mismatches from
the model and the data samples.  We showcase
the utility of our methods by using them in an in-depth analysis of
a popular inverse folding model - ProteinMPNN \citep{Dauparas2022pmpnn}. Our
results demonstrate the theoretical properties of our methods, while also
providing insight as to how to gauge the certainty and applicability of
ProteinMPNN for designing proteins of varying topologies and evolutionary
families.

\section{Problem Setting}\label{sec:problem-setting}
We consider the problem of predicting a discrete sequence-valued variable we are designing $Y \in \mathcal  Y$,
for example a biological sequence, conditionally on a variable $X \in \mathcal  X$ at our disposal.
The predicted sequence $ Y $ 
is allowed to have an arbitrary length, e.g. $ \mathcal  Y = \cup_{\ell=1}^{\infty} \mathcal A^{\ell} $,
where $ \mathcal  A $ is a finite set.
In protein design, $ X $ could be the 3D structure of a protein
(e.g. $ \mathcal  X = \cup_{\ell =1 }^{\infty} \mathbb{R}^{3\ell} $)
and $ Y $ the sequence of amino acids making up the protein, in which case $ \mathcal  A $
is the set of amino acids.
Given a large number of i.i.d measurements
 $\{X_i, Y_i\}_{i=1}^{N_T}$ from a distribution $\mathbb P(X, Y)$, for example pairs of sequences and structures from
the Protein Data Bank  \cite{ingraham2019generative}, we train a \emph{predictive} model
$ Q_{|}: x  \longmapsto Q_{|x} $ that takes in a 
value $x$ and outputs a distribution on $Y$, $Q_{|x}(Y)$ that attempts to match the true
conditional $\mathbb P(Y |X=x)$, denoted $\mathbb P_{|x}(Y)$ in this work.
After training,
we are interested in quantifying how accurately $Q_{|}$ approximates
$\mathbb P_{|}$ on average across all values of $x$ after training, using a held-out set of samples
$\{X_i, Y_i\}_{i=1}^{N} \sim \mathbb P(X, Y)$.
Quantifying the accuracy of $ Q_{|} $ is known as the \emph{conditional goodness-of-fit}
problem, and we address it in \cref{sec:augmented_conditional_mmd}.
Furthermore, we will also be interested in quantifying the reliability of $ Q_{|} $,
a task which we address in \cref{sec:reliability}.

\section{Conditional Goodness--of--Fit with ACMMD}\label{sec:augmented_conditional_mmd}
In this section, we propose a metric that quantifies the accuracy of a predictive sequence model.
We will show that this metric satisfies many desirable properties:
first, it is absolute and able to detect any differences between conditional distributions.
Second, it can be unbiasedly and efficiently estimated using samples from the model and the data distribution.
Third, it can be used in hypothesis tests to detect statistically significant mismatches from such samples.

\subsection{The Augmented Conditional MMD}
We now propose a method to quantitatively
evaluate the conditional goodness--of--fit of $ Q_{|} $ to $ \mathbb P_{|} $.
Our approach
consists in constructing a \emph{divergence} $\operatorname{D}(\mathbb P_{|}, Q_{|})$,
between the conditional distribution of $ Y $ given $ X $ and the model $ Q_{|} $.
By definition, this divergence should satisfy:
\begin{equation}\label{eq:cgof-metric}
\begin{aligned}
(i) \,\, &\operatorname{D}(\mathbb P_{|}, Q_{|} ) \geq 0 \\
(ii) \,\,& \operatorname{D}(\mathbb P_{|}, Q_{|})  = 0 \iff \mathbb P_{|x}=Q_{|x} ,\,\, \mathbb P(X)\text{--a.e.}
\end{aligned}
\end{equation}
Combined, these two properties ensure that $\operatorname{D}(\mathbb P_{|}, Q_{|})$
is \emph{absolute}, e.g. assigns the known value lowest value 0 to the best possible model,
and is able to distinguish any mismatch between the model and the data, which is
crucial to prevent blind spots in our evaluation.
We borrow the idea of comparing $Q_{\mid}$ with $\mathbb P_\mid$ by comparing the joint $\mathbb P(X, Y)$
with a joint that keeps the same marginal $\mathbb P(X)$ but swaps $\mathbb P_\mid$ with $ Q_{|} $.
These two joint distributions are equal if and only if $Q_{\mid}$ and $\mathbb P_\mid$ match almost everywhere.
To compare these two distributions, we will use the Maximum Mean Discrepancy (MMD) \cite{gretton2012kernel} given by:
\begin{equation}\label{eq:mmd}
\begin{aligned}
    \mathrm{MMD}(\mathbb{ Q }_1, \mathbb{ Q }_2) = \sup_{ \substack{f \in \mathcal  H_{\mathcal Z} \\ \Vert f \Vert_{\mathcal  H_{\mathcal Z}} \leq 1} } \mathbb E_{{\mathbb Q}_1}[f(Z)] - \mathbb E_{{\mathbb Q}_2}[f(Z)].
\end{aligned}
\end{equation}
Here, $ \mathcal  Z $ is some measurable space, $ \mathbb{ Q }_1 $ and $ \mathbb{ Q }_2 $ are probability measures
on $ \mathcal  Z $, and $ \mathcal  H_{\mathcal  Z} $ is a reproducing kernel Hilbert space
(RKHS) of functions from $\mathcal  Z$ to $\mathbb  R$
with kernel $ k_{\mathcal  Z} $ \cite{berlinet2011reproducing}.
Applying this general definition to the case at hand,
we obtain a measure of accuracy for $ Q_{|} $, defined below.

\begin{definition}[Augmented Conditional MMD]\label{def:mmd_cgof}
    Let $ (X, Y) \in \mathcal X \times \mathcal Y $
    with law $ \mathbb{ P }_{X} \otimes \mathbb{ P }_{|} $.
    Let  $ Q_{|} $ be a conditional probability from $\mathcal X$ to $\mathcal Y$.
    We define the Augmented Conditional $ \operatorname{MMD}
    $ (ACMMD) between $ \mathbb{ P }_{|} $ and $ Q_{|}$ as:
    \begin{equation}\label{eq:mmd_cgof}
    \begin{aligned}
    \operatorname{ACMMD}(\mathbb P_{|}, Q_{|}) \coloneqq  \mathrm{MMD}(\mathbb P_X \otimes \mathbb P_{|}, \mathbb P_X \otimes Q_{|})
    \end{aligned}
    \end{equation}
    where the $ \operatorname{MMD} $ is evaluated with a user-specified kernel
    $ k_{\mathcal  X \times \mathcal  Y} $ on $\mathcal  X \times \mathcal  Y$.
    Here, $ \mathbb{ P }_{X }\otimes \mathbb{ P }_{|} $ 
    is defined by $ (X, Y) \sim \mathbb{ P }_{X
    }\otimes \mathbb{ P }_{|} \iff X \sim \mathbb{ P }_X $, $ (Y|X=x) \sim
    \mathbb{ P }_{|x} $, and similarly for $ \mathbb{ P }_{X }\otimes Q_{|}$.

\end{definition}

{\bfseries Choice of kernel for ACMMD}
The ACMMD requires specifying a kernel on the joint space $\mathcal X \times \mathcal Y$.
In this work, we will focus on the case where $ k_{\mathcal  X \times \mathcal  Y} $ is the
\emph{tensor product} kernel $ k_{\mathcal  X} \otimes k_{\mathcal  Y} $ of two kernels
$ k_{\mathcal  X} $ and $ k_{\mathcal Y} $ on $\mathcal  X$ and $\mathcal  Y$ respectively:
\begin{equation} \label{eq:tensor-product-kernel}
\begin{aligned}
k_{\mathcal  X \times \mathcal  Y}((x, y), (x', y')) = k_{\mathcal  X}(x, x')k_{\mathcal  Y}(y, y')
\end{aligned}
\end{equation}
This choice is popular in practice, and the resulting ACMMD retains
its desirable properties, as we show next.

\paragraph{The ACMMD is a divergence between conditional probabilities}
The ACMMD writes as divergence (which is symmetric, e.g. a distance)
between joint distributions,
while we seek to use it to compare conditional distributions.
The following lemma shows that the same ACMMD can be formulated
in alternative manner that highlights its purpose as a conditional
distribution comparator.

\begin{lemma}\label{lemma:cgof_mmd}
    Under mild integrability conditions, we have:
    \begin{equation*} \label{}
    \begin{aligned}
    \operatorname{ACMMD}(\mathbb P_{|}, Q_{|}) =  \left \| T_{K_{\mathcal
    X}}(\mu_{\mathbb{ P }_{|}} - \mu_{Q_{|}}) \right \|_{\mathcal
    H_{\mathcal  X, \mathcal  H_\mathcal  Y}}
    \end{aligned}
    \end{equation*}
    Where $ \mu_{\mathbb{  P}_{|}} $  and $ \mu_{Q_{|}} $ are the conditional mean
    embeddings \cite{park2020measure} of $ \mathbb P_{|} $ and $ Q_{|} $, $
    K_{\mathcal  X}(x, x') \coloneqq k_{\mathcal X}(x, x') I_{\mathcal
    H_{\mathcal  Y}} $ (here, $ I_{\mathcal  H_{\mathcal  Y}} $ the
    identity operator) is an operator-valued kernel with associated vector-valued
    RKHS $ \mathcal H_{\mathcal X, \mathcal  H_{\mathcal  Y}} \subset L^{2}_{\mathbb{ P
    }_X}(\mathcal  X, \mathcal H_{\mathcal Y}) $, and $ T_{K_{\mathcal  X}} $
    is its associated integral operator from $ L^{2}_{\mathbb{ P }_X}(\mathcal
    X, \mathcal H_{\mathcal Y}) $ to $ \mathcal  H_{\mathcal  X, \mathcal  H_\mathcal Y} $.
    Moreover, if $ k_{\mathcal  X} $ and $
    k_{\mathcal  Y} $ are $ C_0 $-universal \footnote{A kernel $ k $ is $ C_0
        $-universal if the associated RKHS $ \mathcal  H_k $ is dense in $
    C_0(\mathcal  X) $, the space of continuous functions on $ \mathcal  X $
    vanishing at infinity \cite{sriperumbudur2010relation} }, then it holds that:
    \begin{equation*} \label{}
    \begin{aligned}
        \operatorname{ACMMD}(\mathbb{ P }_{|}, Q_{|}) = 0 \iff \mathbb{ P }_{|x} = Q_{|x}, \quad  \mathbb{ P }_X \text{-a.e.}
    \end{aligned}
    \end{equation*}
\end{lemma}

The complete statement (with the full set of assumptions, and the definition of
integral operators) and its proof can be found in \cref{sec:proof_cgof_mmd}.
\cref{lemma:cgof_mmd} shows that
the $ \operatorname{ACMMD} $ can be understood as
the result of a two-step procedure, given by (1) computing the conditional mean
embedding $ \mu_{\mathbb{ P}_{|}}: x  \longmapsto \mathbb{ E  }_{ \mathbb{ P }_{|x}} \left \lbrack k_{\mathcal  Y}(y, \cdot) | X=x
\right \rbrack $ of $ \mathbb{ P }_{|} $ (resp. of $ Q_{|} $), which is a
function from $ \mathcal  X $ to $ \mathcal H_{\mathcal  Y} $, and (2) embed
the difference of these conditional mean embeddings into the
\emph{vector-valued} RKHS $ \mathcal  H_{\mathcal  X, \mathcal  H_{\mathcal
Y}} \subset L^{2}_{\mathbb{ P }_X}(\mathcal  X, \mathcal H_{\mathcal  Y}) $
with kernel $ K_{\mathcal X} $, before returning its associated RKHS norm. The
second part of the lemma gives sufficient conditions for the $ \mathrm{ACMMD} $
to discriminate between any non ($ \mathbb{ P }_{X} $--a.e) equal conditional
distributions, fulfilling the requirements specified in \cref{eq:cgof-metric}:
these conditions are to use \emph{universal} kernels $
k_{\mathcal  X} $ and $ k_{\mathcal  Y} $.  Regarding $ k_{\mathcal  Y} $,
this requirement is not very restrictive, as many universal kernels on sequences
have been shown to be universal \cite{Amin2023-er}.
The difficulty in finding a universal $ k_{\mathcal  X}  $ will depend on the
space $ \mathcal  X $ (unspecified in this work) for the problem at hand.

\paragraph{Estimating the ACMMD from data}

Crucial to this work is the fact that if the model $ Q_{|} $ can be sampled from
for any $ x \in \mathcal  X $, $ \operatorname{ACMMD}\vphantom{a}^{2} $ will admit tractable unbiased
estimators. To see this, we first rewrite $ \operatorname{ACMMD}\vphantom{a}^{2} $ in a
form that will make this property apparent.

\begin{lemma}\label{lemma:cgof_mmd_double_expectation}
    Let $ Z \coloneqq (X, Y, \tilde{Y}) $ the triplet of random variables with law
    \footnote{Identifying $ Q_{|} $ with its
    analogue Markov kernel $ \tilde{Q}_{|} $ from $ (\mathcal  X \times \mathcal  Y,
    \mathscr X \otimes \mathscr  Y)$ such that $ \tilde{Q}_{|(x, y)}(dy') \coloneqq
    Q_{|x}(dy') $.
    }
    $ \mathbb{ P }_X \otimes \mathbb{ P }_{|} \otimes Q_{|} $. Then, under the
    integrability assumptions of \cref{lemma:cgof_mmd}, we have that:
    \begin{equation*} \label{}
    \begin{aligned}
    \operatorname{ACMMD}^{2}(\mathbb{ P }_{|}, Q_{|}) &= \mathbb{ E  }_{Z_1, Z_2} \left \lbrack
    h(Z_1, Z_2) \right \rbrack
    \end{aligned}
    \end{equation*}
    where $ Z_1, Z_2 $ are two independent copies of $ Z $ and
    $ h $ is a symmetric function given by:
    \addtolength{\jot}{0.5em}
    \begin{equation*} \label{}
    \begin{aligned}
    &h(Z_1, Z_2) \coloneqq k_{\mathcal  X}(X_1, X_2) g((Y_1, \tilde{Y}_1), (Y_2, \tilde{Y}_2)) \\
    &g((Y_1, \tilde{Y}_1), (Y_2, \tilde{Y}_2)) \coloneqq k_{\mathcal  Y}(\tilde{Y}_1, \tilde{Y}_2) + k_{\mathcal  Y}(Y_1, Y_2) \\ 
    & \quad - k_{\mathcal  Y}(\tilde{Y}_1, Y_2) - k_{\mathcal  Y}(Y_1, \tilde{Y}_2)
    \end{aligned}
    \end{equation*}
\end{lemma}
\cref{lemma:cgof_mmd_double_expectation}, proved in
\cref{lemma:proof_cgof_mmd_estimator}, expresses $ \operatorname{ACMMD}\vphantom{a}^{2} $ as a double
expectation given two independent \emph{samples} of $ (X, Y, \tilde{Y})\sim \mathbb{ P
}_{X} \otimes \mathbb{ P }_{|} \otimes Q_{|}$. Leveraging this fact, we can
derive an unbiased and consistent estimator for $ \operatorname{ACMMD}\vphantom{a}^{2} $.

\begin{lemma}\label{lemma:cgof_mmd_estimator}
    Let $ \{ X_{i}, Y_{i}, \tilde{Y}_{i} \}_{i=1}^{N}
    \overset{\text{i.i.d}}{\sim} \mathbb{ P }_{X} \otimes \mathbb{ P }_{|} \otimes Q_{|}$
    be samples from the data and the model. Then an unbiased
    estimator
    $\textstyle \widehat{ \operatorname{ACMMD} }{\vphantom{\operatorname{ACMMD}}}^{2}(\mathbb{ P }_{|}, Q_{|})$
    of
    $ \operatorname{ACMMD}^{2}(\mathbb{ P }_{|}, Q_{|}) $ 
    is given by:
    \begin{equation} \label{eq:cgof_mmd_estimator}
    \begin{aligned}
        \frac{2}{N(N-1)} \sum\limits_{ 1 \leq i < j \leq N }
        h((X_{i}, Y_{i}, \tilde{Y}_{i}), (X_{j}, Y_{j}, \tilde{Y}_{j}))
    \end{aligned}
    \end{equation}
\end{lemma}

\vspace{-1em}

\cref{lemma:cgof_mmd_estimator},  proved in \cref{lemma:proof_cgof_mmd_estimator},
shows that it is possible to unbiasedly estimate $ \operatorname{ACMMD}^{2} $ even
when the analytical model expectations are intractable, provided that one can
sample from the model. This estimator takes the form of a U-statistics
\citep[Chapter 5]{serfling2009approximation} with symmetric probability kernel
$ h $, which are well-studied in the statistics literature. In particular, they
provide a generic framework to obtain minimal-variance analogues of unbiased
estimators \citep[Chapter 5, p. 176]{serfling2009approximation}.
$ \textstyle \widehat{ \operatorname{ACMMD} }{\vphantom{\operatorname{ACMMD}}}^{2} $ is a \emph{consistent}
estimator of $
\operatorname{ACMMD}^{2} $: under the integrability assumptions of
\cref{lemma:cgof_mmd} the strong law of large numbers applies
\citep[Section~5.4,~Theorem~A]{serfling2009approximation}, and we have:
$\textstyle \widehat{ \operatorname{ACMMD} }{\vphantom{\operatorname{ACMMD}}}^{2}(\mathbb{ P }_{|}, Q_{|}) \xrightarrow[N  \to \infty]{ a.s. } \operatorname{ACMMD}^{2}(\mathbb{ P }_{|}, Q_{|})$.
We provide a more detailed characterization of the asymptotic distribution of 
$\textstyle \widehat{ \operatorname{ACMMD} }{\vphantom{\operatorname{ACMMD}}}^{2}$
in
\cref{app-sec:cgof_mmd_estimator_asymptotic_distribution}.

\subsection{Testing Conditional Goodness--of--Fit with ACMMD}\label{sec:testing}

In the limit of infinitely many samples, a positive ACMMD means that the model and the data differ. However,
in practice, when only a finite number of samples are available, our estimate
$\textstyle \widehat{ \operatorname{ACMMD} }{\vphantom{\operatorname{ACMMD}}}^{2}$
is only a noisy version of the true $ \operatorname{ACMMD}^2 $, meaning we cannot conclude whether the model fits the data
by directly inspecting its value. Instead, we need a procedure that accounts for the estimation
noise; we achieve this by using the ACMMD as part of a hypothesis test
deciding between two different hypotheses:
\begin{equation*} \label{}
\begin{aligned}
    \begin{cases}
    H_0: \operatorname{ACMMD}(\mathbb{ P }_{|}, Q_{|}) = 0 \\
    H_1: \operatorname{ACMMD}(\mathbb{ P }_{|}, Q_{|}) > 0
    \end{cases} 
\end{aligned}
\end{equation*}
In particular, we construct a test that takes as input a sample 
$ \{ X_{i}, Y_{i}, \tilde{Y}_{i} \}_{i=1}^N $ from the data and the model
and outputs a (binary) decision to reject (or not) the null hypothesis $ H_0 $
based on whether $ \textstyle \widehat{ \operatorname{ACMMD} }{\vphantom{\operatorname{ACMMD}}}^{2}(\mathbb{ P }_{|}, Q_{|}) $ exceeds a certain
threshold. Because of the estimation noise arising from the use of finitely
many samples, such a test cannot systematically output the right decision.
Nonetheless, we build our test to ensure a \emph{false rejection}
(e.g. reject $ H_0 $ while $ \mathbb{ P }_{|} = Q_{|}\,\, \mathrm{a.e} $)
rate of $ \alpha \in (0, 1) $, a common
practice in statistical testing \cite{gretton2012kernel}. To do so, we would like to
set the rejection threshold $ q_{1-\alpha} $ to be an estimate of the
$ 1-\alpha $ quantile of the distribution of
$ \textstyle \widehat{ \operatorname{ACMMD} }{\vphantom{\operatorname{ACMMD}}}^{2}(\mathbb{ P }_{|}, Q_{|}) $
under $ H_{0} $.
However, since $ q_{1-\alpha} $ is not available in closed form, we instead compute an
estimate $ \widehat{q}_{1-\alpha} $ using the wild bootstrap procedure
\cite{arcones1992bootstrap}. This procedure draws $ B $ samples $
\{\widetilde{\operatorname{ACMMD}}{\vphantom{\operatorname{ACMMD}}}^{2}_{b}\}_{b=1}^{B} $
of the form:

\vspace*{-1.5em}

\begin{equation} \label{eq:wild-bootstrap}
\begin{aligned}
    \widetilde{\operatorname{ACMMD}}{\vphantom{\operatorname{ACMMD}}}^{2}_{b} \coloneqq \frac{2}{N(N-1)} \hspace{-0.4em} \sum\limits_{
    1 \leq i < j \leq N }^{ N } \hspace{-0.4em} W^{b}_{i} W^{b}_j h(Z_{i}, Z_{j})
\end{aligned}
\end{equation}

\vspace*{-1.5em}
where $ \{ W^{b}_i \}_{\substack{i=1\dots N}}^{b=1\dots B} $ are i.i.d.
Rademacher random variables independent of the data,
from which we compute a quantile estimate $ \widehat{ q }_{1-\alpha} $ of this distribution
of samples (see  \cref{app-sec:type_I_error_control} for a precise definition of $ \widehat{ q }_{1-\alpha} $).
Importantly, this procedure guarantees an exact control of the false rejection rate
at level $ \alpha $.
We prove this fact in \cref{sec:level_acmmd}, where we cast the wild bootstrap
procedure as a Monte-Carlo estimation of the distribution of
$\widehat{\operatorname{ACMMD}}\vphantom{a}^{2}$ when $ \mathbb{ P }_{|} = Q_{|} $, which is valid
non-asymptotically. Our test, which we call the
ACMMD test, is summarized in Algorithm \ref{alg:mmd_cgof_test}. To the best of
our knowledge, this is the first conditional goodness-of-fit test that is
applicable to sequence models.
\begin{algorithm}[tb]
   \caption{ACMMD Conditional Goodness--of--fit Test}
   \label{alg:mmd_cgof_test}
\begin{algorithmic}
   \STATE {\bfseries Input:} $\{X_{i}, Y_{i},
   \tilde{Y}_{i}\}_{i=1}^N \stackrel{\text{i.i.d.}}{\sim} \mathbb P_{X} \otimes
   \mathbb{ P }_{|} \otimes Q_{|}$
   \STATE \textbf{Parameters:} Level $\alpha$,
   kernel $k_{\mathcal  X}$, kernel $ k_{\mathcal  Y}$
    \STATE \texttt{// Estimate ACMMD using \cref{eq:cgof_mmd_estimator}}
    \STATE $\widehat{\operatorname{ACMMD}}\vphantom{a}^{2}\hspace{-0.2em} \leftarrow
    \frac{2}{N(N-1)}\hspace{-0.5em} \sum\limits_{ \substack{ i, j = 1 \\ i < j } }^{ N }
    \hspace{-0.5em}  h((X_{i}, Y_{i}, \tilde{Y}_{i}), (X_{j}, Y_{j}, \tilde{Y}_{j}))$
    \STATE Sample $\{ \widetilde{\operatorname{ACMMD}}\vphantom{a}_{b}^{2} \}_{b=1}^{B}$ using \cref{eq:wild-bootstrap}
    \STATE $\widehat{q}_{1-\alpha} \leftarrow$ approx. $(1- \alpha)$-quantile 
    of $\{ \widetilde{\operatorname{ACMMD}}\vphantom{a}_{b}^{2} \}_{b=1}^{B}$
    \IF {$ \widehat{ \operatorname{ACMMD} }\vphantom{a}^{2}\leq \widehat{q}_{1-\alpha}$}
        \STATE {Fail to reject $H_0$}
    \ELSE
        \STATE {Reject $H_0$}
    \ENDIF
\end{algorithmic}
\end{algorithm}

\section{Assessing Reliability with ACMMD}\label{sec:reliability}
In practice, our model $ Q_{|} $ may not fit the data perfectly, and it is important
to distinguish (at a given level of inaccuracy) models that remain
consistent with their training data from ones that fail more drastically.
In this section, we show how the ACMMD can be used to evaluate model reliability,
a statistical property capturing model and data consistency.

{\bfseries Problem Setting}
A model $Q_{|}$ is said to be reliable 
\citep{brocker2008some,vaicenavicius2019evaluating,widmann2022calibration} if
the distribution of the target $ Y $ given that the model made a specific prediction $ q $
\emph{is} this prediction $ q $ itself, e.g. if:
\begin{equation}\label{eq:calibration} 
\begin{aligned}
    q = \mathbb{P}\left(Y \in \cdot \mid Q_{|X} = q \right)\qquad \mathbb{ P }(Q_{|X})\text{--a.e.}
\end{aligned}
\end{equation}
Here, $ Q_{|X} \in \mathcal  P(\mathcal  Y) $ (the space of probability distributions on $ \mathcal  Y $)
is the random variable obtained by evaluating the model $Q_{|}$ at a random value of the input variable $X$.
Reliability differs from accuracy in that it does not require
the model to learn all the information between $ X $ and $ Y $,
but only to make truthful predictions on average --- thus, by assessing reliability,
one may be able to detect models that hallucinate non-realistic sequences
(such as repeats of the same token) in regions of the input space where
they are inaccurate, instead of making a conservative guess, such as 
falling back to the prior disitribution.
In particular, reliability can be used as an additional criterion
to discriminate between models that are equally accurate.
From a theoretical perspective, reliability and accuracy can be handled in 
a unified manner: indeed, Equation \ref{eq:calibration} shows
that reliability is defined as an equality between
the conditional distribution of $ Y $ given a model prediction $ q $,
$ \mathbb{ P }^{Q}_{|q} \coloneqq \mathbb{ P }(Y=\cdot|Q_{|X} = q) $
and a ``model'' of this conditional distribution mapping $ q \in \mathcal  P(\mathcal  Y) $ 
to itself, e.g. $ Q^{\mathrm{Rel}}_{|}: q  \longmapsto Q^{\mathrm{Rel}}_{|q} = q $.
We thus propose to measure reliability using 
the ACMMD (a distance between conditional distributions)
between $ Q^{\mathrm{Rel}}_{|} $ and $ \mathbb{ P }^{Q}_{|} $.
\begin{definition}[ACMMD for Reliability]\label{def:acmmd}
    The Augmented Conditional $ \operatorname{MMD} $
    for reliability ($\operatorname{ACMMD--Rel}$) between $ \mathbb{ P }_{|} $ and $ Q_{|}$ as:
    \begin{equation}\label{eq:acmmd-rel}
    \begin{aligned}
        \operatorname{ACMMD--Rel}(\mathbb P_{|}, Q_{|})\hspace{-0.3em} &\coloneqq \mathrm{ACMMD}(\mathbb{ P }^{Q}_{|}, Q^{\mathrm{Rel}}_{|}) \\
                                                        &= \mathrm{MMD}(\mathbb P_{|Q_{|X}} \otimes \mathbb P^{Q}_{|}, \mathbb P_{|Q_{|X}} \otimes Q_{|})
    \end{aligned}
    \end{equation}
    where the $\operatorname{ACMMD}$ is evaluated with a user-specified kernel
    $ k_{\mathcal  P(\mathcal  Y) \times \mathcal  Y} $ on $ \mathcal  P(\mathcal  Y) \times \mathcal  Y $.
\end{definition}
As for the ACMMD, we will restrict our attention to the case where 
$ k_{\mathcal  P(\mathcal  Y) \times \mathcal  Y} $ is a tensor product kernel
$ k_{\mathcal  P(\mathcal  Y)} \otimes k_{\mathcal  Y} $ between a kernel on
$ \mathcal  P(\mathcal  Y) $ and a kernel on $ \mathcal  Y $.
Comparing the ACMMD--Rel with the ACMMD, we see that the former requires specifying
a kernel \emph{on the space of probability measures} on sequences $ \mathcal P(\mathcal  Y) $
instead of a kernel on $ \mathcal  X $.
Two important points must be addressed when working with such kernels.
First, in order to have $ \operatorname{ACMMD--Rel}(\mathbb P_{|}, Q_{|}) = 0$
if and only if $ Q_{|} $ is reliable, we must find universal kernels defined on $ \mathcal P(\mathcal  Y) $.
Second, as many kernels on probabilities are intractable, we must design an
approximation strategy to estimate the ACMMD--Rel from data.
\paragraph{ACMMD--Rel can detect any pattern of unreliability}
Our first goal is to ensure that ACMMD--Rel can detect any pattern of unreliability.
As ACMMD--Rel is a specific instance of the ACMMD, we can apply \cref{lemma:cgof_mmd},
stating that if the kernels
$ k_{\mathcal  P(\mathcal  Y)}$ and $ k_{\mathcal  Y} $ are universal, then
\begin{equation} \label{}
\begin{aligned}
\operatorname{ACMMD--Rel}(\mathbb P_{|}, Q_{|}) = 0 \iff Q_{|} \text{ is reliable}.
\end{aligned}
\end{equation}
The task of finding a universal kernel $ k_{\mathcal  Y} $ on $ \mathcal Y $
was addressed in \cref{sec:augmented_conditional_mmd}; thus,
it remains to  find a universal kernel $ k_{\mathcal  P(\mathcal  Y)} $ 
on $ \mathcal P(\mathcal  Y) $.
However, to the best of our knowledge, none of the existing kernels defined on probability distributions
\cite{carmeli2010vector, szabo2015two, szaboSPG16, meunier2022pc, Glaser2023-sj}
have been shown to be universal when $ \mathcal  Y $ is the space of arbitrary-length sequences.
In the next proposition, we show that many such kernels can be
constructed by following a simple recipe.
\begin{proposition}\label{prop:universal_kernel_on_PYS}
    Let $k_{\mathcal  Y}$ be a kernel on $ \mathcal  Y$
    vanishing at infinity (on $ \mathcal  Y \times \mathcal  Y $).
    Suppose that $ k_{\mathcal  Y} $ has discrete
    masses, i.e. that $\delta_y \in \mathcal H_{\mathcal  Y}$ for all sequences $y \in
    \mathcal  Y$, where $\delta_y$ is the Dirac function at $y$, and let $ \sigma > 0 $.
    Then the kernel $ k_{\mathcal P(\mathcal  Y)} $ on  $ \mathcal P(\mathcal  Y) $ defined as
    \begin{equation} \label{eq:kernels-on-dists}
    \begin{aligned}
      &k_{\mathcal  P(\mathcal  Y)}(q, q') \coloneqq e^{-\frac 1 {2\sigma^2} \operatorname{MMD}^2(q, q')}, \\
    \end{aligned}
    \end{equation}
    (where the MMD is computed in $\mathcal H_{\mathcal  Y}$) 
    is a $ C_{0} $--universal kernel on the space of probability
    distributions $\mathcal P(\mathcal  Y)$ (under the topology of
    convergence in distribution or Total Variation, which are identical, see
    \cite{amin2021generative}).
\end{proposition}
The proof, provided in \cref{proof:universal_kernel_on_PYS},
relies on an argument similar to prior work for universal kernels on
probability measures \cite{carmeli2010vector},
but tailored to the special case of sequences.
\cref{prop:universal_kernel_on_PYS} guarantees that any kernel on $ \mathcal Y
$ vanishing at infinity with the discrete mass property \cite{Amin2023-er} can be used to construct a
universal kernel on $ \mathcal  P(\mathcal  Y) $.
Kernels with discrete masses are studied in detail in \cite{Amin2023-er}.
In particular, the tilted Exponentiated Hamming
Kernel $\frac{1}{|y||y'|} e^{-\lambda d_H(y, y')}$ (where $ |y| $ is the length of the sequence $ y $)
is a kernel with discrete masses vanishing at infinity on $ \mathcal  Y \times \mathcal  Y $,
and can thus be used to construct a universal kernel on $ \mathcal  P(\mathcal  Y) $.

\paragraph{Estimating ACMMD--Rel from data}
To estimate ACMMD--Rel from the data $ \{ X_{i}, Y_{i} \}_{i=1}^{N}  $
and samples from the model $ \{ \tilde{Y}_{i} \sim Q_{|X_{i}} \}_{i=1}^{N} $,
one may try to use the general ACMMD estimator
proposed in Lemma \ref{lemma:cgof_mmd_estimator}, which, specialized 
to the reliability setting, is given by:
\begin{equation*}
\begin{aligned}
    & \quad \quad \quad
    \frac{2}{N(N-1)}\hspace{-0.5em} \sum\limits_{ \substack{ 1 \leq i < j \leq  N } } \hspace{-1em}
    h(Z_{i}, Z_{j}) \\
    &h(Z_{i}, Z_{j}) \coloneqq k_{\mathcal  P(\mathcal  Y)}(Q_{|X_{i}}, Q_{|X_{j}}) g((Y_{i}, \tilde{Y}_{i}), (Y_{j}, \tilde{Y}_{j}))
\end{aligned}
\end{equation*}
This estimator requires evaluating
$ k_{\mathcal  P(\mathcal  Y)}(Q_{|X_i}, Q_{|X_j}) $
for pairs $ i, j $.
Unfortunately, 
exact evaluation of these quantities for the universal kernels proposed in 
\cref{prop:universal_kernel_on_PYS} is in general impossible,
as $ \operatorname{MMD}^{2}(Q_{|X_i}, Q_{|X_j}) $
contains intractable expectations under $ Q_{|X_i} $ and $ Q_{|X_j}$.
However, MMDs can be unbiasedly estimated using samples from
$ Q_{|X_i} $ and $ Q_{|X_j} $ \cite{gretton2012kernel, schrab2022efficient}.
Inspired by this fact, we propose the following estimator:
\begin{equation} \label{eq:acmmd-rel-estimator}
\begin{aligned}
    \widehat{\operatorname{ACMMD}}\operatorname{--Rel}\vphantom{a}^{2} \coloneqq \frac{2}{N(N-1)}\hspace{-0.3em} \sum\limits_{ \substack{ 1 \leq i < j \leq  N } } \hspace{-1em}
    \widehat{h}(Z_{i}, Z_{j}) \\
    \widehat{h}(Z_{i}, Z_{j}) \coloneqq \widehat{k}_{ij} \times g((Y_{i}, \tilde{Y}_{i}), (Y_{j}, \tilde{Y}_{j}))
\end{aligned}
\end{equation}
Here, $ \widehat{  k}_{ij} $ 
is an approximation of 
$ k_{\mathcal  P(\mathcal  Y)}(Q_{|X_{i}}, Q_{|X_{j}}) $
obtained by drawing $ R $ samples
$ \{\tilde{Y}^{r}_{i} \}_{r=1}^{R}$ and $ \{\tilde{Y}^{r}_{j} \}_{r=1}^{R} $ from $ Q_{|X_{i}} $
and $ Q_{|X_{j}} $, and replacing the $\operatorname{MMD}^{2}(Q_{|X_{i}}, Q_{|X_{j}}) $
term in $ k_{\mathcal  P} $ 
by an unbiased 
estimate $ \widehat{ \operatorname{MMD} }\vphantom{a}^2 _{ij}$ computed from these samples.
The full estimation procedure is provided in Algorithm
\ref{alg:ACMMD-Rel}.
This additional approximation step has several implications: first, unlike
$ \widehat{ \operatorname{ACMMD }}\vphantom{a}^2  $, $ \widehat{ \operatorname{ACMMD }}\operatorname{--Rel}\vphantom{a}^2  $
is not unbiased. However, the bias of this estimator can be controlled by increasing the number of samples
$ R $ used to estimate the MMD. Moreover, we show in the next proposition that
the estimator $ \widehat{ \operatorname{ACMMD }}\operatorname{--Rel}\vphantom{a}^2  $ is still consistent
provided that $ R $ is chosen appropriately.
\begin{proposition}\label{prop:acmmd-rel-estimator-consistency}
    Assume that $ k_{\mathcal  Y} $ is bounded.
    Then, if $ R \equiv R(N) $, with $ \lim\limits_{ N  \to \infty }R(N) = +\infty $,
    $ \widehat{\operatorname{ACMMD}}\operatorname{--Rel}\vphantom{a}^2 $
    converges in probability to $
    \operatorname{ACMMD--Rel}\vphantom{a}^2 $ as  $ N  \to \infty $.
\end{proposition}
\begin{algorithm}
    \caption{Estimating $\operatorname{ACMMD--Rel}$}
   \label{alg:ACMMD-Rel}
\begin{algorithmic}
   \STATE {\bfseries Input:} $\{X_{i}, Y_{i}, \tilde{Y}_{i}\}_{i=1}^N \stackrel{\text{i.i.d.}}{\sim} \mathbb{ P }_{X} \otimes \mathbb{ P }_{|} \otimes Q_{|}$, model $ Q_{|} $
   \STATE \textbf{Parameters:} kernel $k_{\mathcal  Y}$
   \FOR{$i$ \textbf{in}  1 to $N$}
   \STATE \lbrack $\tilde{Y}^{r}_{i} \sim Q_{|X_{i}}$ \textbf{for} $r$ \textbf{in} 1 to $R$ \rbrack 
   \ENDFOR
   \FOR{$i, j $ \textbf{in}  1 to $N$}
   \STATE \texttt{\%Use, e.g.} \citet[Equation 4]{gretton2012kernel}
   \STATE $ \widehat{ \operatorname{MMD} }^2_{ij} \coloneqq  $ \texttt{estimate\_mmd}$(\{\tilde{Y}^{r}_{i}\}_{r=1}^{R}, \{\tilde{Y}^{r}_{j}\}_{r=1}^{R})$
   \STATE $ \widehat{h}_{ij} \coloneqq e^{-\frac{1}{2 \sigma^2} \widehat{ \operatorname{MMD} }^2_{ij}}
   \times g((Y_{i}, \tilde{Y}_{i}), (Y_{j}, \tilde{Y}_{j}))$
   \ENDFOR
   \vspace{-2em}
   \STATE 
    \begin{flalign*}
    \textbf{return }
      &\frac{2}{N(N-1)}\hspace{-0.2em} \sum\limits_{ \substack{ 1 \leq i < j \leq  N } } \hspace{-1em} \widehat{h}_{ij} &\\
    \end{flalign*}
    \vspace{-3em}
\end{algorithmic}
\end{algorithm}

\paragraph{Testing for reliability with ACMMD--Rel}
As an ACMMD, ACMMD--Rel has the potential to be used to test whether a model
is reliable given some available data: to do so, one can use Algorithm
\ref{alg:mmd_cgof_test}, replacing $ \widehat{ \operatorname{ACMMD} }\vphantom{a}^2  $ by $
\widehat{ \operatorname{ACMMD} }\operatorname{--Rel}\vphantom{a}^2  $, and performing quantile estimation
using the $ \widehat{ h }(Z_{i}, Z_{j}) $ instead of the $ h(Z_{i}, Z_{j}) $.
A full description of the algorithm is provided in Appendix \ref{sec:proof_approx-test-validity}.
An important question to answer is whether
the approximation of using $ \widehat{ h } $ instead of $ h $
affects the false-rejection rate
of the test. We show in the next proposition that this is not the case.

\begin{proposition}\label{prop:acmmd-rel-test-validity}
    Assume that $ k_{\mathcal  Y} $ is bounded, and $ k_{\mathcal
    P(\mathcal  Y)} $ is a kernel of the form of Equation \ref{eq:kernels-on-dists}.
    Then a reliability test using $ \widehat{h}(Z_{i}, Z_{j}) $ instead of
    $ h(Z_{i}, Z_{j}) $ to estimate $ \operatorname{ACMMD--Rel} $ and its 
    $(1-\alpha)$--quantile under $ H_0 $ has a false-rejection rate of
    exactly $ \alpha $.
\end{proposition}
\section{Related Work}\label{sec:related-work}

{\bfseries Goodness-of-fit methods}
The \textit{goodness-of-fit} problem is a well-studied problem in
the statistics and machine learning literature, for which many methods
were developed \cite{chwialkowski2016kernel,
Gorham2017-rt,grathwohl2020learning, amin2023kernelized,baum2023kernel}.
Impressively, these methods can operate directly from the model's analytical
form, without requiring access to samples from the model -- which may be hard
to generate. In these works, goodness-of-fit is defined as
the problem of evaluating the fit of \emph{unconditional} models to their data,
which is unlike the conditional goodness-of-fit problem we consider here.
Evaluating conditional goodness-of-fit with kernels was recently studied in 
\citet{jitkrittum2020testing}. 
However, the proposed method requires the output space $ \mathcal  Y $ to be a subset of
$ \mathbb  R^{d} $, and is thus unsuitable for conditional sequence models.
The use of conditional goodness-of-fit metrics to evaluate reliability 
was also done in \citet{Glaser2023-sj}, in a method also limited to continuous output spaces.
Finally, we note that $ \operatorname{ACMMD--Rel}\vphantom{a}^2$
recovers an existing calibration metric, the Squared Kernel Calibration Error (SKCE) of
\citet{widmann2022calibration}. However, the latter did not study the problems of
universality, tractability and test validity in the case of sequence-valued outputs.

{\bfseries Deep Protein Design Models}
(Deep Learning--powered) conditional probability models have gained significant
momentum in computational biology during the last decade. In particular, such
models have revolutionized the protein design field
\citep{Johnson2023YangScoring, Bennett2023binderdesign}.
Inverse folding models are trained on protein
structures and sequences in the protein data bank (PDB)
\citep{ingraham2019generative, Hsu2022esmif1, Dauparas2022pmpnn}. They
condition a sequence distribution on an input protein structure --- thus learning
what sequences would likely fold into that structure. The designs from these
methods have been shown to be highly stable and retain function
\citep{Sumida2024improvePMPNN}. However, many of the leaps made using these
models have used small, simple structural scaffolds (like loop-helix-loop
motifs) \citep{Bennett2023binderdesign, Watson2023RFdiffusion}. Protein
engineers interested in leveraging these tools for novel scaffolds need to know
how accurate and reliable the model is on average. If the model is too imprecise, one
might wish to gather more data and train more bespoke models before using the
method to design experiments.

\section{Experiments}\label{sec:experiments}
We now investigate the behavior and utility of the ACMMD and ACMMD--Rel metrics and tests
in practice. We start with a synthetic example showing that ACMMD is a natural measure of model
distance. We then perform an extended analysis of a state-of-the-art inverse folding model, ProteinMPNN.
We show that ACMMD can detect small perturbations in the model, and that it can be used to tune
its temperature parameter. Finally, we analyze the absolute performance of ProteinMPNN.

\subsection{A toy synthetic setting} \label{sec:toy-synthetic}
We first study the behavior of ACMMD and ACMMD--Rel
in a synthetic setting where the data distribution and the model are simple
generative models on sequences.
We set the input variable $ X $ to be a single scalar $ p $
drawn from some distribution $ \mathbb{ P }_{X} $ with support in $ (0.3, 0.5) $.
$ Y $ is a sequence of arbitrary length with alphabet $ \mathcal  A = \left \{ A, B, \mathrm{STOP} \right \}  $.
We set the conditional distribution of $ Y $ given $ p $ to be:
\begin{equation*} \label{eq:toy-data}
\begin{aligned}
        p(y_n | y_{0:n-1}, x=p) =
        \begin{cases}
            A & \text{with probability } p \\
            B & \text{with probability } p \\
            \text{STOP} & \text{with probability } 1 - 2p
        \end{cases}
\end{aligned}
\end{equation*}
so long as $ y_{n-1} \neq \text{STOP} $. The model distribution
$ Q_{|} $ is the same as the data, except for the fact that
the first factor $ Q_{|p}(y_0) $ is perturbed by a parameter $ \Delta p $:

\begin{equation*} \label{eq:toy-model}
\begin{aligned}
	Q_{|p}(y_0)
        &= \begin{cases}
	A & \text{with probability } p - \Delta p  \\
        B & \text{with probability } p + \Delta p \\
        \text{STOP} & \text{with probability } 1 - 2p
	\end{cases}  \\
\end{aligned}
\end{equation*}

We set the kernel on $ \mathcal  Y $ to be the exponentiated
Hamming distance kernel $ k_{\mathcal  Y}(y, y') = e^{-d_H(y, y')} $,
where $ d_H(y, y') $ is the Hamming distance between $ y $ and $ y' $, and
$ k_{\mathcal  X}  $ to be the Gaussian kernel $ k_{\mathcal  X}(p, p') = e^{-\frac{1}{2}(p - p')^2} $.
With such choices, it is possible to show that:
\begin{equation*} \label{}
\begin{aligned}
    \operatorname{ACMMD}(\mathbb{ P }_{|}, Q_{|}) &= C \lvert \Delta p \rvert 
\end{aligned}
\end{equation*}
for some $ C > 0 $ that does not depend on $ \Delta p $, and is computable
in closed form for discrete priors on $ X $. The proof and expression of $ C $
are given in \cref{sec:acmmd-proofs-synthetic}. 
From this expression, we immediately see that ACMMD is 0 only if $ \Delta p = 0 $, a manifestation of 
\cref{lemma:cgof_mmd} which guarantees that the ACMMD can detect any mismatch
between the model and the data when $ k_{\mathcal  X} $ and $ k_{\mathcal  Y} $ are universal.
Moreover, in this case, the ACMMD depends monotonically on the shift $ |\Delta p| $.
Since $ |\Delta p| $ represents a natural measure of how different
the model is from the data, this fact suggests that the ACMMD is a natural,
well-behaved measure of model distance. Additionally, we plot the average rejection rate
of the ACMMD test for various number of samples and shifts in \cref{fig:acmmd-calibration-main}.
The results for $ \Delta p = 0 $ confirm that our test has the correct specified type-I error rate (0.05).
Moreover, we see that the power of the test increases with the number of samples, and the shift $ |\Delta p| $.

\begin{figure}[htbp]
    \includegraphics[width=.23\textwidth]{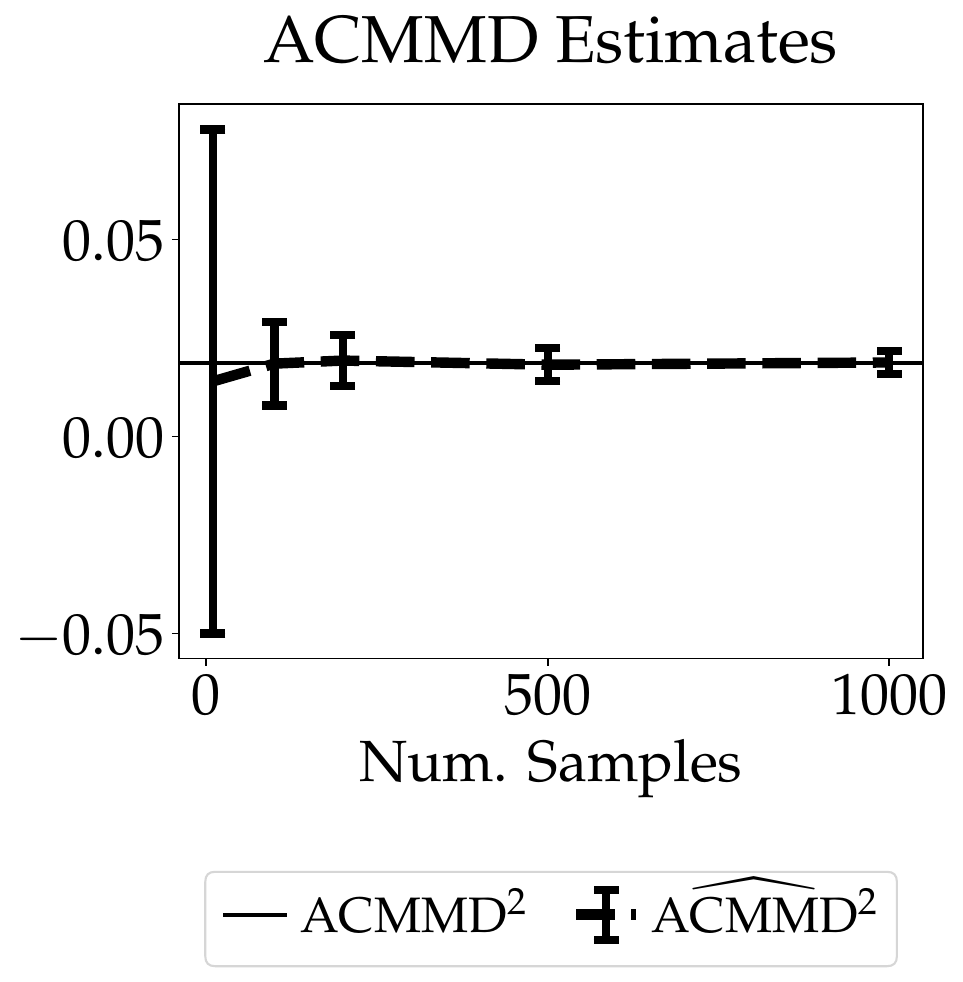}
    \includegraphics[width=.23\textwidth]{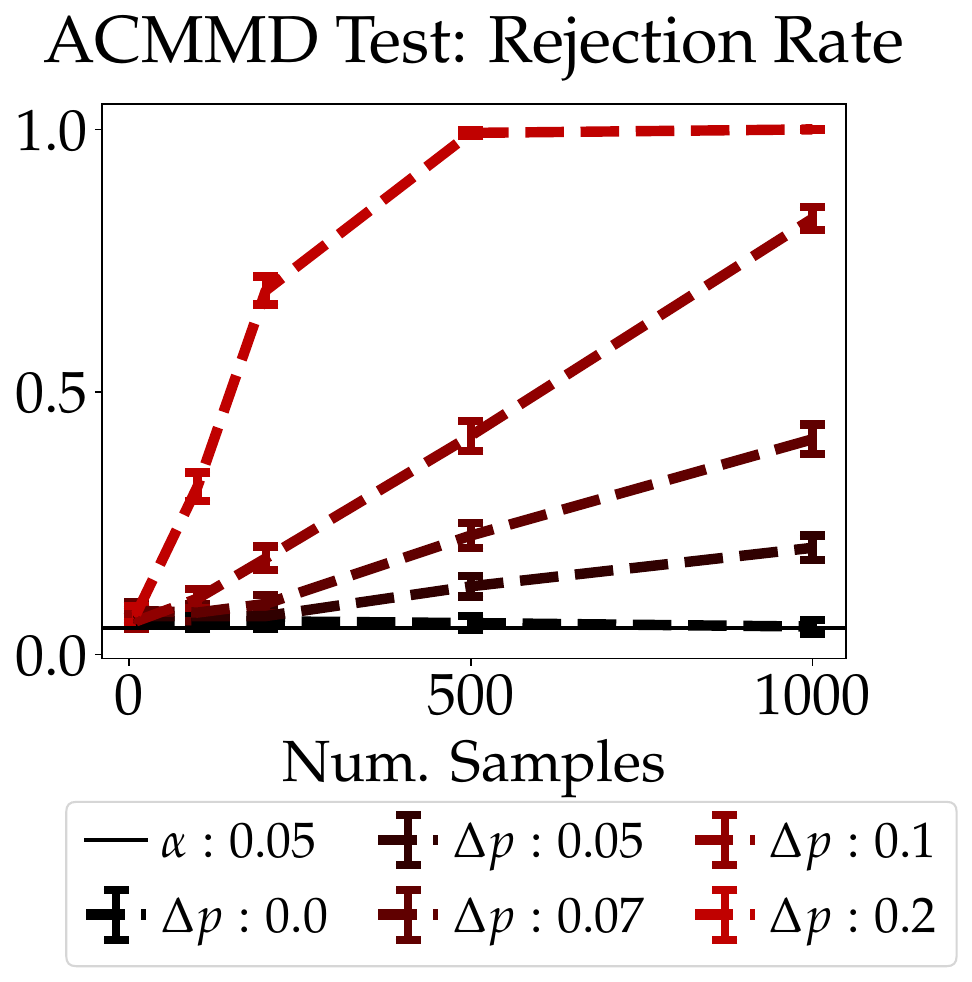}
    \caption{
    Left panel: ACMMD estimates for a fixed shift value $ \Delta p = 0.25 $
    and various number of samples in the synthetic example of \cref{sec:toy-synthetic}.
    The analytic ACMMD value is given by the horizontal line.
    Right panel: ACMMD test average rejection rate
    for various number of samples and shifts in the same setting.
}
    \label{fig:acmmd-calibration-main}
\end{figure}

\subsection{ACMMD Case Study: Inverse Folding Models}

To demonstrate the utility of the ACMMD measures and tests, we apply them to
evaluate inverse folding models, a popular model framework used in protein
design. Inverse folding models seek a distribution of amino acid sequences
that are likely to fold into a given input three-dimensional structure, as discussed
in Section \ref{sec:augmented_conditional_mmd}.
We focus our experiments on
evaluating ProteinMPNN \citep{Dauparas2022pmpnn}, a sampleable, commonly used
model in this class.
The sampling temperature $ T $ of ProteinMPNN can also be varied, letting the user
control the trade-off between accuracy and diversity of the generated
sequences.

\textbf{Data} \quad We leveraged the CATH taxonomy to select a set of diverse
(in sequence and structural topologies) protein structures to perform our ACMMD
test on. CATH is a taxonomy of protein structures that categorizes proteins
according to a hierarchy of structural organization \citep{Sillitoe2021CATH}.
We used the S60 redundancy filtered set which includes proteins that are at
least 60\% different in sequence identity from each other. Of these, we
selected all single domain monomers (proteins where only one topological domain
is found in the monomer), and removed any topologies that had fewer than 10
chains in its classification. This left us with 17,540 structures.

\textbf{Choice of kernel} \quad
Key to the performance of our metrics is the choice of the kernels
$ k_{\mathcal  X} $, $ k_{\mathcal  Y} $ and $ k_{\mathcal  P(\mathcal  Y)} $.
For $ k_{\mathcal  Y} $, we propose to use kernels that first embed each element
-- or \emph{residue} -- of a sequence $ y $ 
using an embedding function $ \phi_{\mathcal  Y}: \mathcal  A \times \mathcal  Y \longmapsto
\mathbb{R}^{d_{\mathcal  Y}} $, and evaluating a euclidean kernel on
$ \mathbb{R}^{d_{\mathcal  Y}} $ on the mean of the resulting embeddings,
yielding a kernel of the form:
\begin{equation*} \label{}
\begin{aligned}
k_{\mathcal  Y}(y, y') = k_{\mathbb{R}^{d_{\mathcal  Y}}}\!\!\left(\frac{1}{|y|}\sum\limits_{ i=1 }^{ |y| }\phi_{\mathcal  Y}(y_i, y), \frac{1}{|y'|}\sum\limits_{ i=1 }^{ |y'| }\phi_{\mathcal  Y}(y'_i, y')\right)
\end{aligned}
\end{equation*}
where we noted $ y = (y_1, \ldots, y_{|y|}) $. As the input space $ \mathcal  X $ is
also sequence-valued, we follow the same recipe to construct our a kernel $ k_{\mathcal  X} $,
using an embedding function 
$ \phi_{\mathcal  X}:  \mathbb{ R }^{3} \times \mathcal  X \longmapsto \mathbb{R}^{d_{\mathcal  X}} $.
Finally, for the kernel on $ \mathcal  P(\mathcal  Y) $, we will use a kernel of the form of 
\cref{eq:kernels-on-dists}, with kernel $ k_{\mathcal  Y} $ described above to compute the inner MMD.
We set our embedding functions $ \phi_{\mathcal X} $ and $ \phi_{\mathcal  Y} $ to
a pair of recent pre-trained neural networks that are commonly used for
representation learning of protein sequences and structures: Gearnet
\cite{zhang2023enhancing} for structures, and ESM-2 \cite{lin2023evolutionary}
for sequences.
Such two-step kernels allow us to
instill the complex structure present in the distribution of protein structures
and sequences within the ACMMD maximizing the performance and meaningfulness of
our evaluation pipeline. Whether \cref{prop:universal_kernel_on_PYS} holds
for these kernels is an open question, but we find that they perform well in
our experiments.

\begin{figure}[htbp]
    \includegraphics[width=0.46\textwidth]{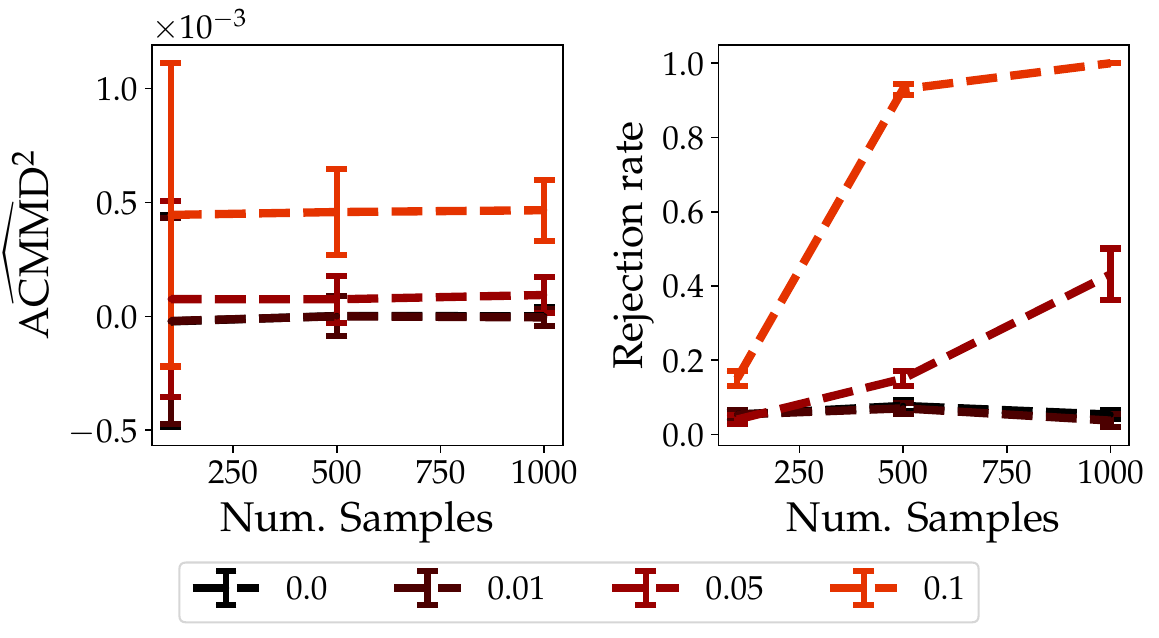}
    \caption{
        Values of $ \widehat{ \operatorname{ACMMD} }\vphantom{a}^2 $,
        (left) and of the average rejection rate of the ACMMD test (right) in the setting
        described in \ref{sec:discriminative}.
        Each line corresponds to a different value for $ \delta T $.
    } \label{fig:synthetic_exp_pmpnn_cgof}
\end{figure}

\subsubsection{The Discriminative power of \texorpdfstring{$ \operatorname{ACMMD}$}{ACMMD}}\label{sec:discriminative}

We first propose to evaluate the behavior of the ACMMD and its associated test
when comparing a known ground truth and a model distribution differing from the
ground truth in a controlled manner. To this end, we set the ground truth to be
a pre-trained ProteinMPNN model $ Q^{T}_{|} $ with temperature $ T $, and the model
to be the same model $ Q^{T+  \delta T}_{|} $ with temperature $ T + \delta T $. As ProteinMPNN's probability
distribution is a continuous function of $T$, small changes in $ T $ result in
small changes in the predicted distribution which will be hard to detect,
translating into ``low'' values for $ \widehat{ \operatorname{ACMMD} }\vphantom{a}^2 $
relative to larger temperature changes. Conversely, we posit
that large changes in $ T $ will result in large changes in the model
distribution, and will be simpler to detect by the ACMMD. To test these
hypotheses, we performed an estimation of $ \operatorname{ACMMD}\vphantom{a}^2 $
for a ground truth temperature $ T=0.1 $ (the default in the
ProteinMPNN documentation) and $ \delta T \in \left \{ 0, 0.01, 0.05,
0.1 \right \} $. We used the winged helix-like DNA binding domain superfamily
(CATH ID: 1.10.10.10), and performed bootstrap sampling to produce dataset
sizes ranging from $ 100 $ to  $ 1000 $, and $ 100 $ different random seeds in
order to obtain confidence intervals of our estimates.

\begin{figure}[htbp]
    \includegraphics[width=0.46\textwidth]{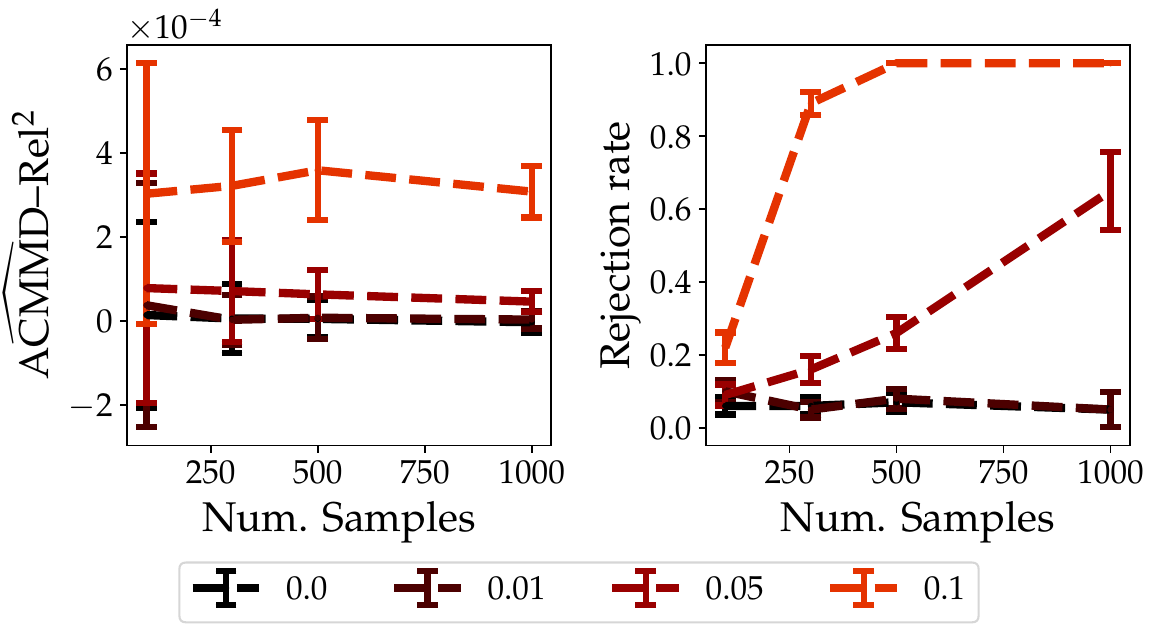}
    \caption{
        Values of $ \widehat{ \operatorname{ACMMD} }\operatorname{--Rel}\vphantom{a}^2 $,
        (left) and of the average rejection rate of the ACMMD--Rel test (right)
        in the setting described in \cref{sec:discriminative}.
        Each line corresponds to a different value for $ \delta T $.
    } \label{fig:synthetic_exp_pmpnn_calibration}
\end{figure}

The results are shown in Figure \ref{fig:synthetic_exp_pmpnn_cgof}. As expected, the value of $
\operatorname{ACMMD}\vphantom{a}^2(Q^{T}_{|}, Q^{T + \delta T}_{|}) $ robustly increases with
increasing values of $ \delta T $.
Additionally, we performed the ACMMD test of Section \ref{sec:testing}
with a target type-I error rate of $  \alpha=0.05 $, and $ 100 $ permutations
to estimate the $ 1-\alpha $ quantile of the null distribution for the same
values of $ N $ and $ \delta T $, and computed the average rejection
rate of the null hypothesis $ H_0: \mathrm{ACMMD}\vphantom{a}^2(Q^{T}_{|}, Q^{T + \delta T}_{|}) > 0 $,
which, if $k _{\mathcal  X}$ and $ k_{\mathcal  Y} $ are universal, is equivalent to $ H_0: \delta T = 0 $.
The results, shown in Figure
\ref{fig:synthetic_exp_pmpnn_cgof} (right), empirically confirm that the ACMMD
test controls its type-I error rate and is able to detect differences in
temperatures of an order relevant for ProteinMPNN. Similarly,
we evaluate the behavior of the ACMMD--Rel, which is used to assess
the (lack of) reliability of between model $ Q^{T+\delta T}_{|} $ w.r.t the data 
$ \mathbb{ P }_{|X} \otimes  Q^{T}_{|} $, the assumption being that $ Q^{T + \delta T}_{|} $ 
is not reliable when $ \delta T \ne 0 $. The results are shown in Figure
\ref{fig:synthetic_exp_pmpnn_calibration}, and exhibits similar behavior.

\subsubsection{Evaluation of ProteinMPNN on the CATH dataset}\label{sec:proteinmpnn-whole-data-evaluation}
Now that we have confirmed the discriminative power of the ACMMD 
on semi-synthetic data, we use our tests to evaluate ProteinMPNN
against real-world protein structures and sequences from the CATH dataset. We
perform a whole-data evaluation, using samples of 5000 proteins across all
families in the dataset. Then we perform a fine-grain evaluation on a
subset of CATH superfamilies.

\paragraph{Whole-data Evaluation}

We first study the deviation of ProteinMPNN from the true data by computing $
\widehat{ \operatorname{ACMMD} }\vphantom{a}^2  $ and estimating
its mean and variance by bootstrapping over 10 random seeds. We find that
ProteinMPNN with no temperature adjustment ($ T=1.0$) has an
$ \widehat{ \operatorname{ACMMD} }\vphantom{a}^2  $ value of 0.0916 (and a p-value  $ < 0.01 $).
Comparing this to the criterion values obtained on similar dataset sizes in the toy
data experiments demonstrates that the model does not fit the test data. This suggests that
there is still much room for improvement on solving the inverse folding
problem.

\vspace{-1.em}

\paragraph{On optimal temperature choices for ProteinMPNN}
Practitioners vary the sampling temperature as a heuristic method for sampling
more certain sequences from ProteinMPNN; lower temperature settings have been
found to generate sequences with fewer unrealistic artifacts (e.g. runs of
alanines) which fold to more stable structures \citep{Sumida2024improvePMPNN}.
However, the relationship between
sampling temperature, model reliability, and design accuracy has not been fully
established. To thoroughly evaluate this, designs from different sampling
temperatures conditioned on a diverse set of backbone structures would need to
be experimentally characterized, which is resourse intensive in practice.
We leverage the ACMMD to understand at what sampling temperature ProteinMPNN
best fits the data, which gives insight as to what temperature is optimum, by
computing $ \widehat{ \operatorname{ACMMD} }\vphantom{a}^2  $ and
$ \widehat{ \operatorname{ACMMD} }\operatorname{--Rel}\vphantom{a}^2 $ for varying temperature values
across $ 10 $ seeds for each
temperature value. The results are shown in Figure
\ref{fig:whole_dataset_cgof_calibration}.

First, we observe that reducing the temperature below $1.0$ improves both the
model's goodness-of-fit and its reliability. This corroborates the empirical
design success of lowering the sampling temperature, suggesting that greater
model fit may increase the quality of samples from the model.
The decrease in reliability at higher temperature shows that even though
increasing the temperature increases the diversity of the model's predictions,
this diversity does not necessarily capture the one of the data distribution,
as for instance the prior would.
The optimal temperature from the perspective of
goodness-of-fit is $ 0.4 $ (which lies outside the suggested temperature
range of 0.1-0.3 in the ProteinMPNN documentation \cite{Dauparas2022pmpnn}).
However, we notice that model reliability continues to improve with even lower sampling
temperatures while accuracy slightly increases, suggesting a trade-off between reliability and accuracy.
Further experiments will determine how this trade-off 
manifests in the quality of designs from low-temperature settings.

\begin{figure}[htbp]
    \begin{center}
    \includegraphics[width=0.45\textwidth]{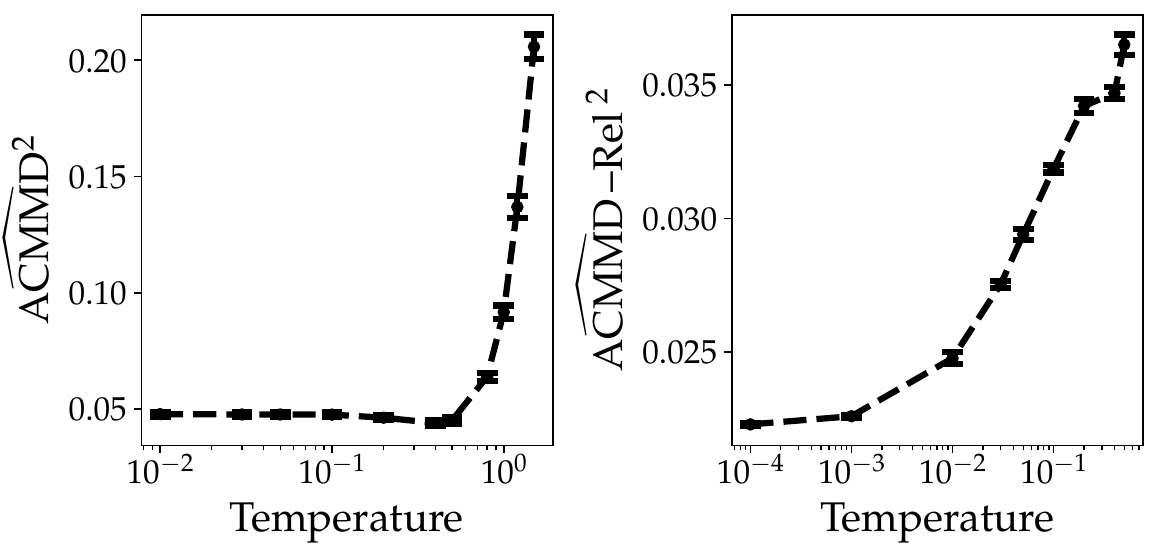}
    \end{center}
    
    \vspace{-2em}
    \caption{
   Evolution of $ \widehat{ \operatorname{ACMMD} }\vphantom{a}^2  $ (left)
   and $ \widehat{ \operatorname{ACMMD}}\operatorname{--Rel}\vphantom{a}^2  $ (right)
   between a pre-trained ProteinMPNN model and the CATH S60 reference dataset, for
   varying temperature values
 } \label{fig:whole_dataset_cgof_calibration}
\end{figure}

\vspace{-1em}
\paragraph{Structural superfamily evaluation}\label{sec:superfamily-evaluation}
The (H)omologous superfamily tier within the CAT(H) hierarchy groups proteins with the same
similar folds and sequence identity. While
ProteinMPNN has shown great performance in designing particular structural scaffolds, a
practitioner aiming to leverage this model on a yet untested structural family
may want some insight as to how well ProteinMPNN may fit the distribution of
proteins they are interested in. Thus, we performed ACMMD evaluation separately on
individual superfamilies contained in our dataset to gain insights on what types of structures
ProteinMPNN does or does not fit well. We filtered the superfamilies for groupings with at least
500 proteins under a length of 100, yielding 11 families.
The results are shown in Figure \ref{fig:superfamily_T_bar}. 
We find that the model fit varies across families and the fit ranking
is largely maintained at different temperatures. With no temperature adjustment
($T=1.0$) the best fit superfamily (lowest $ \widehat{ \operatorname{ACMMD}}\vphantom{a}^2 $)
is the Homeodomain-like proteins (CATH ID:
1.10.10.60). These structures are largely dominated by helical bundles - a
class of proteins that ProteinMPNN has demonstrated success on designing
\citep{Dauparas2022pmpnn, Watson2023RFdiffusion, Bennett2023binderdesign}.
While the Immunoglobulin superfamily has the highest fit at lower sampling
temperatures, we note that most of an immunoglobulin structure consists of
the beta sandwich of the framework, while, for antibody design, engineers are
often most interested in the unstructured complementarity determining regions
(CDRs) of antibodies \citep{Kunik2013CDRloops, Liu2020Gifford, Jin2022HERN}. As
the criterion is calculated across the entire sequence, this may not reflect
that ProteinMPNN has learned the distribution of CDR loops well. Further work
will extend these tests to focus on subsequences of a domain to answer
specific questions of model fit on regions of interest.

\begin{figure}[htbp]
    \begin{center}
    \includegraphics[width=.43\textwidth]{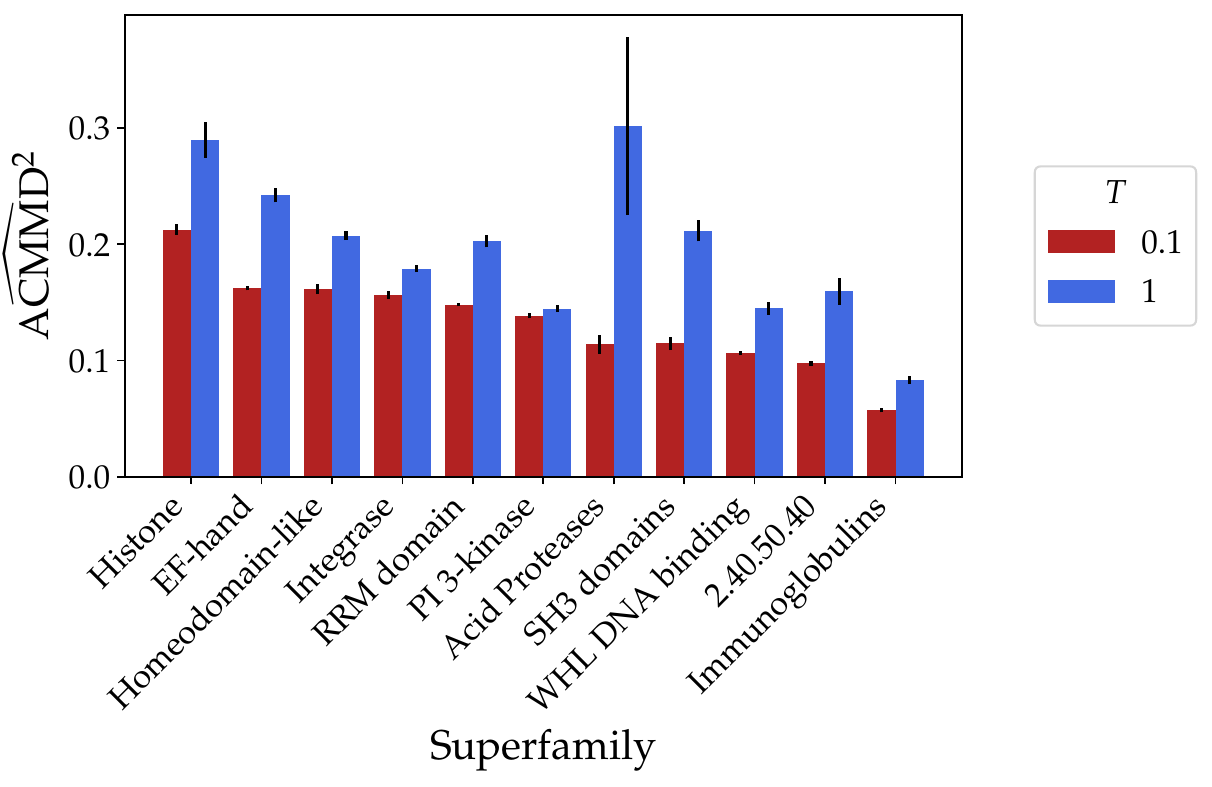}
    \end{center}
    \vspace{-2em}
    \caption{ Estimated value $ \widehat{ \operatorname{ACMMD} }\vphantom{a}^2 $ between ProteinMPNN and
        the CATH S60 reference dataset on a subset of 10 superfamilies for two 
        different temperatures $ T=1.0 $ and $ T=0.1 $.
    }
    \label{fig:superfamily_T_bar}
\end{figure}

\vspace{-1em}
\section{Discussion}
Advancing the computational evaluation of conditional sequence models is crucial for accelerating the development of these methods for protein engineering. Given the limitations of current evaluation methods, and leveraging recent advancements in kernel methods for designing tests of goodness-of-fit and calibration, we propose a criterion and its associated test to principledly evaluate protein sequence models for how well they have learned input-conditioned sequence distributions.
We discuss the statistical properties of our metrics and develop testing frameworks from them.
Finally, we leverage them to investigate the performance of inverse folding models under default and temperature-adjusted settings. We develop novel insights on ideal temperature settings for ProteinMPNN and discuss the trade-off between design accuracy and model calibration that our tests demonstrate for lower temperatures.
Future work can perform a more fine-grained evaluation, for example investigating which structures in particular cause the model to make unreliable predictions and what features of the model's predictions do not match the data through the use of witness functions, a by-product of MMDs \citep{Lloyd2015-vh}.
We also note that protein engineering goals may differ from pure modeling goals, and whether performance under our metrics reflect experimental design success rates requires further investigation to determine. Yet, barring orthogonal \textit{in silico} validation data or experimental testing, our methods offer a powerful framework to test conditional sequence models for desirable statistical properties.

\section* {Impact Statement}
The tools developed in this work assess the quality of sequences predictors. As such, they have the potential to influence various procedures in protein design, and, on longer timescales, healthcare. However, the conclusions that they provide are only statistical: while they are guaranteed to hold on average, they will not hold every time. Such tools should thus be used with caution, and in conjunction with external help from domain experts to ensure that the real-world actions they will influence remain beneficial to society.

\bibliography{bibpierre,bibsteffanie,smallbib}

\begin{thebibliography}{47}
\providecommand{\natexlab}[1]{#1}
\providecommand{\url}[1]{\texttt{#1}}
\expandafter\ifx\csname urlstyle\endcsname\relax
  \providecommand{\doi}[1]{doi: #1}\else
  \providecommand{\doi}{doi: \begingroup \urlstyle{rm}\Url}\fi

\bibitem[Amin et~al.(2021)Amin, Weinstein, and Marks]{amin2021generative}
Amin, A., Weinstein, E.~N., and Marks, D.
\newblock A generative nonparametric bayesian model for whole genomes.
\newblock \emph{NeurIPS}, 34:\penalty0 27798--27812, 2021.

\bibitem[Amin et~al.(2023{\natexlab{a}})Amin, Weinstein, and
  Marks]{Amin2023-er}
Amin, A.~N., Weinstein, E.~N., and Marks, D.~S.
\newblock Biological sequence kernels with guaranteed flexibility.
\newblock \emph{arXiv preprint arXiv:2304.03775}, 2023{\natexlab{a}}.

\bibitem[Amin et~al.(2023{\natexlab{b}})Amin, Weinstein, and
  Marks]{amin2023kernelized}
Amin, A.~N., Weinstein, E.~N., and Marks, D.~S.
\newblock A kernelized stein discrepancy for biological sequences.
\newblock In \emph{Proceedings of the 40th ICML}, 2023{\natexlab{b}}.

\bibitem[Arcones \& Giné(1992)Arcones and Giné]{arcones1992bootstrap}
Arcones, M.~A. and Giné, E.
\newblock On the bootstrap of {U} and {V} statistics.
\newblock \emph{The Annals of Statistics}, 1992.

\bibitem[Baum et~al.(2023)Baum, Kanagawa, and Gretton]{baum2023kernel}
Baum, J., Kanagawa, H., and Gretton, A.
\newblock A kernel stein test of goodness of fit for sequential models.
\newblock In \emph{ICML}, 2023.

\bibitem[Bennett et~al.(2023)Bennett, Coventry, Goreshnik, Huang, Allen,
  Vafeados, Peng, Dauparas, Baek, Stewart, DiMaio, De~Munck, Savvides, and
  Baker]{Bennett2023binderdesign}
Bennett, N.~R., Coventry, B., Goreshnik, I., Huang, B., Allen, A., Vafeados,
  D., Peng, Y.~P., Dauparas, J., Baek, M., Stewart, L., DiMaio, F., De~Munck,
  S., Savvides, S.~N., and Baker, D.
\newblock Improving de novo protein binder design with deep learning.
\newblock \emph{Nat. Commun.}, 2023.

\bibitem[Berlinet \& Thomas-Agnan(2011)Berlinet and
  Thomas-Agnan]{berlinet2011reproducing}
Berlinet, A. and Thomas-Agnan, C.
\newblock \emph{Reproducing kernel Hilbert spaces in probability and
  statistics}.
\newblock 2011.

\bibitem[Br{\"o}cker(2008)]{brocker2008some}
Br{\"o}cker, J.
\newblock Some remarks on the reliability of categorical probability forecasts.
\newblock \emph{Monthly Weather Review}, 2008.

\bibitem[Carmeli et~al.(2010)Carmeli, De~Vito, Toigo, and
  Umanit{\'a}]{carmeli2010vector}
Carmeli, C., De~Vito, E., Toigo, A., and Umanit{\'a}, V.
\newblock Vector valued reproducing kernel {H}ilbert spaces and universality.
\newblock \emph{Analysis and Applications}, 2010.

\bibitem[Christmann \& Steinwart(2010)Christmann and
  Steinwart]{Christmann2010-ox}
Christmann, A. and Steinwart, I.
\newblock Universal kernels on {Non-Standard} input spaces.
\newblock In \emph{NeurIPS}, 2010.

\bibitem[Chwialkowski et~al.(2016)Chwialkowski, Strathmann, and
  Gretton]{chwialkowski2016kernel}
Chwialkowski, K., Strathmann, H., and Gretton, A.
\newblock A kernel test of goodness of fit.
\newblock In \emph{ICML}, 2016.

\bibitem[Cuturi \& Blondel(2017)Cuturi and Blondel]{cuturi2017soft}
Cuturi, M. and Blondel, M.
\newblock Soft-dtw: a differentiable loss function for time-series.
\newblock In \emph{ICML}, 2017.

\bibitem[Cuturi et~al.(2007)Cuturi, Vert, Birkenes, and
  Matsui]{cuturi2007kernel}
Cuturi, M., Vert, J.-P., Birkenes, O., and Matsui, T.
\newblock A kernel for time series based on global alignments.
\newblock In \emph{ICASSP}, 2007.

\bibitem[Dauparas et~al.(2022)Dauparas, Anishchenko, Bennett, Bai, Ragotte,
  Milles, Wicky, Courbet, de~Haas, Bethel, Leung, Huddy, Pellock, Tischer,
  Chan, Koepnick, Nguyen, Kang, Sankaran, Bera, King, and
  Baker]{Dauparas2022pmpnn}
Dauparas, J., Anishchenko, I., Bennett, N., Bai, H., Ragotte, R.~J., Milles,
  L.~F., Wicky, B. I.~M., Courbet, A., de~Haas, R.~J., Bethel, N., Leung, P.
  J.~Y., Huddy, T.~F., Pellock, S., Tischer, D., Chan, F., Koepnick, B.,
  Nguyen, H., Kang, A., Sankaran, B., Bera, A.~K., King, N.~P., and Baker, D.
\newblock Robust deep learning-based protein sequence design using
  {ProteinMPNN}.
\newblock \emph{Science}, 2022.

\bibitem[Dinculeanu(2000)]{dinculeanu2000vector}
Dinculeanu, N.
\newblock \emph{Vector integration and stochastic integration in Banach
  spaces}.
\newblock John Wiley \& Sons, 2000.

\bibitem[Domingo-Enrich et~al.(2023)Domingo-Enrich, Dwivedi, and
  Mackey]{Domingo-Enrich2023-eu}
Domingo-Enrich, C., Dwivedi, R., and Mackey, L.
\newblock Compress then test: Powerful kernel testing in near-linear time.
\newblock \emph{AISTATS}, 2023.

\bibitem[Gao et~al.(2022)Gao, Tan, Chac{\'o}n, and Li]{Gao2022PiFold}
Gao, Z., Tan, C., Chac{\'o}n, P., and Li, S.~Z.
\newblock {PiFold}: Toward effective and efficient protein inverse folding.
\newblock \emph{arXiv preprint arXiv:2209.12643}, 2022.

\bibitem[Glaser et~al.(2023)Glaser, Widmann, Lindsten, and
  Gretton]{Glaser2023-sj}
Glaser, P., Widmann, D., Lindsten, F., and Gretton, A.
\newblock Fast and scalable score-based kernel calibration tests.
\newblock In \emph{UAI}, 2023.

\bibitem[Gorham \& Mackey(2017)Gorham and Mackey]{Gorham2017-rt}
Gorham, J. and Mackey, L.
\newblock Measuring sample quality with kernels.
\newblock In \emph{ICML}, 2017.

\bibitem[Grathwohl et~al.(2020)Grathwohl, Wang, Jacobsen, Duvenaud, and
  Zemel]{grathwohl2020learning}
Grathwohl, W., Wang, K.-C., Jacobsen, J.-H., Duvenaud, D., and Zemel, R.
\newblock Learning the stein discrepancy for training and evaluating
  energy-based models without sampling.
\newblock In \emph{ICML}, 2020.

\bibitem[Gretton et~al.(2012)Gretton, Borgwardt, Rasch, Sch{\"o}lkopf, and
  Smola]{gretton2012kernel}
Gretton, A., Borgwardt, K.~M., Rasch, M.~J., Sch{\"o}lkopf, B., and Smola, A.
\newblock A kernel two-sample test.
\newblock \emph{JMLR}, 2012.

\bibitem[Hoeffding(1994)]{hoeffding1994sequences}
Hoeffding, W.
\newblock \emph{On sequences of sums of independent random vectors}.
\newblock Springer, 1994.

\bibitem[Hsu et~al.(2022)Hsu, Verkuil, Liu, Lin, Hie, Sercu, Lerer, and
  Rives]{Hsu2022esmif1}
Hsu, C., Verkuil, R., Liu, J., Lin, Z., Hie, B., Sercu, T., Lerer, A., and
  Rives, A.
\newblock Learning inverse folding from millions of predicted structures.
\newblock \emph{ICML}, 2022.

\bibitem[Ingraham et~al.(2019)Ingraham, Garg, Barzilay, and
  Jaakkola]{ingraham2019generative}
Ingraham, J., Garg, V., Barzilay, R., and Jaakkola, T.
\newblock Generative models for graph-based protein design.
\newblock \emph{NeurIPS}, 2019.

\bibitem[Jin et~al.(2022)Jin, Barzilay, and Jaakkola]{Jin2022HERN}
Jin, W., Barzilay, R., and Jaakkola, T.
\newblock {Antibody-Antigen} docking and design via hierarchical equivariant
  refinement.
\newblock \emph{ICML}, 2022.

\bibitem[Jitkrittum et~al.(2020)Jitkrittum, Kanagawa, and
  Sch{\"o}lkopf]{jitkrittum2020testing}
Jitkrittum, W., Kanagawa, H., and Sch{\"o}lkopf, B.
\newblock Testing goodness of fit of conditional density models with kernels.
\newblock In \emph{UAI}, 2020.

\bibitem[Johnson et~al.(2023)Johnson, Fu, Viknander, Goldin, Monaco, Zelezniak,
  and Yang]{Johnson2023YangScoring}
Johnson, S.~R., Fu, X., Viknander, S., Goldin, C., Monaco, S., Zelezniak, A.,
  and Yang, K.~K.
\newblock Computational scoring and experimental evaluation of enzymes
  generated by neural networks.
\newblock 2023.

\bibitem[Kriege et~al.(2020)Kriege, Johansson, and Morris]{Kriege2020-xc}
Kriege, N.~M., Johansson, F.~D., and Morris, C.
\newblock A survey on graph kernels.
\newblock \emph{Applied Network Science}, 2020.

\bibitem[Kunik \& Ofran(2013)Kunik and Ofran]{Kunik2013CDRloops}
Kunik, V. and Ofran, Y.
\newblock The indistinguishability of epitopes from protein surface is
  explained by the distinct binding preferences of each of the six
  antigen-binding loops.
\newblock \emph{Protein Engineering, Design \& Selection}, 2013.

\bibitem[Lin et~al.(2023)Lin, Akin, Rao, Hie, Zhu, Lu, Smetanin, Verkuil,
  Kabeli, Shmueli, et~al.]{lin2023evolutionary}
Lin, Z., Akin, H., Rao, R., Hie, B., Zhu, Z., Lu, W., Smetanin, N., Verkuil,
  R., Kabeli, O., Shmueli, Y., et~al.
\newblock Evolutionary-scale prediction of atomic-level protein structure with
  a language model.
\newblock \emph{Science}, 2023.

\bibitem[Liu et~al.(2020)Liu, Zeng, Mueller, Carter, Wang, Schilz, Horny,
  Birnbaum, Ewert, and Gifford]{Liu2020Gifford}
Liu, G., Zeng, H., Mueller, J., Carter, B., Wang, Z., Schilz, J., Horny, G.,
  Birnbaum, M.~E., Ewert, S., and Gifford, D.~K.
\newblock Antibody complementarity determining region design using
  high-capacity machine learning.
\newblock \emph{Bioinformatics}, 2020.

\bibitem[Lloyd \& Ghahramani(2015)Lloyd and Ghahramani]{Lloyd2015-vh}
Lloyd, J.~R. and Ghahramani, Z.
\newblock Statistical model criticism using kernel two sample tests.
\newblock \emph{Adv. Neural Inf. Process. Syst.}, 2015-Janua:\penalty0
  829--837, 2015.

\bibitem[Meunier et~al.(2022)Meunier, Pontil, and Ciliberto]{meunier2022pc}
Meunier, D., Pontil, M., and Ciliberto, C.
\newblock Distribution regression with sliced {Wasserstein} kernels.
\newblock In \emph{{ICML}}, 2022.

\bibitem[Park \& Muandet(2020)Park and Muandet]{park2020measure}
Park, J. and Muandet, K.
\newblock A measure-theoretic approach to kernel conditional mean embeddings.
\newblock \emph{NeurIPS}, 2020.

\bibitem[Saigo et~al.(2004)Saigo, Vert, Ueda, and Akutsu]{saigo2004protein}
Saigo, H., Vert, J.-P., Ueda, N., and Akutsu, T.
\newblock Protein homology detection using string alignment kernels.
\newblock \emph{Bioinformatics}, 2004.

\bibitem[Schrab et~al.(2022)Schrab, Kim, Guedj, and
  Gretton]{schrab2022efficient}
Schrab, A., Kim, I., Guedj, B., and Gretton, A.
\newblock Efficient aggregated kernel tests using incomplete $ u $-statistics.
\newblock \emph{Advances in Neural Information Processing Systems},
  35:\penalty0 18793--18807, 2022.

\bibitem[Serfling(2009)]{serfling2009approximation}
Serfling, R.~J.
\newblock \emph{Approximation theorems of mathematical statistics}.
\newblock John Wiley \& Sons, 2009.

\bibitem[Sillitoe et~al.(2021)Sillitoe, Bordin, Dawson, Waman, Ashford,
  Scholes, Pang, Woodridge, Rauer, Sen, Abbasian, Le~Cornu, Lam, Berka,
  Varekova, Svobodova, Lees, and Orengo]{Sillitoe2021CATH}
Sillitoe, I., Bordin, N., Dawson, N., Waman, V.~P., Ashford, P., Scholes,
  H.~M., Pang, C. S.~M., Woodridge, L., Rauer, C., Sen, N., Abbasian, M.,
  Le~Cornu, S., Lam, S.~D., Berka, K., Varekova, I.~H., Svobodova, R., Lees,
  J., and Orengo, C.~A.
\newblock {CATH}: increased structural coverage of functional space.
\newblock \emph{Nucleic acids research}, 2021.

\bibitem[Sriperumbudur et~al.(2010)Sriperumbudur, Fukumizu, and
  Lanckriet]{sriperumbudur2010relation}
Sriperumbudur, B., Fukumizu, K., and Lanckriet, G.
\newblock On the relation between universality, characteristic kernels and rkhs
  embedding of measures.
\newblock In \emph{AISTATS}, 2010.

\bibitem[Sumida et~al.(2024)Sumida, N{\'u}{\~n}ez-Franco, Kalvet, Pellock,
  Wicky, Milles, Dauparas, Wang, Kipnis, Jameson, Kang, De~La~Cruz, Sankaran,
  Bera, Jim{\'e}nez-Os{\'e}s, and Baker]{Sumida2024improvePMPNN}
Sumida, K.~H., N{\'u}{\~n}ez-Franco, R., Kalvet, I., Pellock, S.~J., Wicky, B.
  I.~M., Milles, L.~F., Dauparas, J., Wang, J., Kipnis, Y., Jameson, N., Kang,
  A., De~La~Cruz, J., Sankaran, B., Bera, A.~K., Jim{\'e}nez-Os{\'e}s, G., and
  Baker, D.
\newblock Improving protein expression, stability, and function with
  {ProteinMPNN}.
\newblock \emph{Journal of the American Chemical Society}, 2024.

\bibitem[Szab{\'{o}} et~al.(2015)Szab{\'{o}}, Gretton, P{\'{o}}czos, and
  Sriperumbudur]{szabo2015two}
Szab{\'{o}}, Z., Gretton, A., P{\'{o}}czos, B., and Sriperumbudur, B.~K.
\newblock Two-stage sampled learning theory on distributions.
\newblock In \emph{{AISTATS}}, 2015.

\bibitem[Szab{\'{o}} et~al.(2016)Szab{\'{o}}, Sriperumbudur, P{\'{o}}czos, and
  Gretton]{szaboSPG16}
Szab{\'{o}}, Z., Sriperumbudur, B.~K., P{\'{o}}czos, B., and Gretton, A.
\newblock Learning theory for distribution regression.
\newblock \emph{JMLR}, 2016.

\bibitem[Vaicenavicius et~al.(2019)Vaicenavicius, Widmann, Andersson, Lindsten,
  Roll, and Sch{\"o}n]{vaicenavicius2019evaluating}
Vaicenavicius, J., Widmann, D., Andersson, C., Lindsten, F., Roll, J., and
  Sch{\"o}n, T.
\newblock Evaluating model calibration in classification.
\newblock In \emph{{AISTATS}}, 2019.

\bibitem[Vert et~al.(2004)Vert, Saigo, and Akutsu]{vert2004local}
Vert, J.-P., Saigo, H., and Akutsu, T.
\newblock Local alignment kernels for biological sequences.
\newblock \emph{Kernel methods in computational biology}, 2004.

\bibitem[Watson et~al.(2023)Watson, Juergens, Bennett, Trippe, Yim, Eisenach,
  Ahern, Borst, Ragotte, Milles, Wicky, Hanikel, Pellock, Courbet, Sheffler,
  Wang, Venkatesh, Sappington, Torres, Lauko, De~Bortoli, Mathieu, Ovchinnikov,
  Barzilay, Jaakkola, DiMaio, Baek, and Baker]{Watson2023RFdiffusion}
Watson, J.~L., Juergens, D., Bennett, N.~R., Trippe, B.~L., Yim, J., Eisenach,
  H.~E., Ahern, W., Borst, A.~J., Ragotte, R.~J., Milles, L.~F., Wicky, B.
  I.~M., Hanikel, N., Pellock, S.~J., Courbet, A., Sheffler, W., Wang, J.,
  Venkatesh, P., Sappington, I., Torres, S.~V., Lauko, A., De~Bortoli, V.,
  Mathieu, E., Ovchinnikov, S., Barzilay, R., Jaakkola, T.~S., DiMaio, F.,
  Baek, M., and Baker, D.
\newblock De novo design of protein structure and function with {RFdiffusion}.
\newblock \emph{Nature}, 2023.

\bibitem[Widmann et~al.(2021)Widmann, Lindsten, and
  Zachariah]{widmann2022calibration}
Widmann, D., Lindsten, F., and Zachariah, D.
\newblock Calibration tests beyond classification.
\newblock In \emph{{ICLR}}, 2021.

\bibitem[Zhang et~al.(2023)Zhang, Xu, Chenthamarakshan, Lozano, Das, and
  Tang]{zhang2023enhancing}
Zhang, Z., Xu, M., Chenthamarakshan, V., Lozano, A., Das, P., and Tang, J.
\newblock Enhancing protein language models with structure-based encoder and
  pre-training.
\newblock \emph{arXiv preprint arXiv:2303.06275}, 2023.

\end{thebibliography}
\bibliographystyle{icml2024}

\newpage
\appendix
\onecolumn
\allowdisplaybreaks

\begin{center}
    \textbf{\Large Supplementary Material of the paper \emph{Kernel-Based Evaluation of Conditional Biological Sequence Models}}
\end{center}

\section{Proof of Lemma \ref{lemma:cgof_mmd}}\label{sec:proof_cgof_mmd}
Let us first re-state the lemma in its complete form.
\begin{lemma*}[Complete form of Lemma \ref{lemma:cgof_mmd}]
    Assume that $ \mathcal  X$ is locally-compact and second countable.
    Moreover, assume that $ k_{\mathcal  X \times \mathcal  Y} =
    k_{\mathcal  X} \otimes k_{\mathcal  Y}$, and that $ k_{\mathcal  X} $,
    $ k_{\mathcal Y} $ satisfy the integrability conditions $
    \mathbb{ E }\left \lbrack k_{\mathcal  X}(X, X)k_{\mathcal  Y}(Y, Y) \right \rbrack   < +\infty$ and
    $ \mathbb{ E } \left \lbrack k_{\mathcal  Y}(Y, Y) \right \rbrack  <
    +\infty$ (and similarly for $ \tilde{Y}$).  Then,
    \begin{equation*} 
    \begin{aligned}
    \operatorname{ACMMD}^2(\mathbb P_{|}, Q_{|}) =  \left \| T_{K_{\mathcal
    X}}(\mu_{\mathbb{ P }_{|}} - \mu_{Q_{|}}) \right \|^{2}_{\mathcal
    H_{\mathcal  X, \mathcal  H_\mathcal  Y}}
    \end{aligned}
    \end{equation*}
    Where $ \mu_{\mathbb{  P}_{|}} $  and $ \mu_{Q_{|}} $ are the conditional mean
    embeddings \cite{park2020measure} of $ \mathbb P_{|} $
    and $ Q_{|} $, given by:
    $ \mu_{\mathbb{ P }_{|}}: x  \longmapsto \mathbb{ E }_{y \sim \mathbb{ P }_{|x}}
    k_{\mathcal  Y}(y, \cdot) $ (and similarly for $ Q_{|} $), 
    $ K_{\mathcal  X}(x, x') \coloneqq
    k_{\mathcal X}(x, x') I_{\mathcal  H_{\mathcal  Y}} $ is an
    operator-valued kernel with associated vector-valued RKHS $ \mathcal
    H_{\mathcal X, \mathcal  H _\mathcal   Y} \subset L^{2}_{\mathbb{ P }_X}(\mathcal  X, \mathcal
    H_{\mathcal Y}) $, and $ T_{K_{\mathcal  X}} $ is its associated integral
    operator from $ L^{2}_{\mathbb{ P }_X}(\mathcal  X, \mathcal  H_{\mathcal  X, \mathcal  H_\mathcal Y}) $
    to $ \mathcal  H_{\mathcal  X, \mathcal H_{\mathcal  Y}} $, defined as
    \begin{equation*} 
        T_{K_{\mathcal  X}} f(x) = \int_{\mathcal  X} K_{\mathcal  X}(x, x') f(x') \mathbb P_X(\mathrm{d}x') \in \mathcal  H_{\mathcal  Y} \\
    \end{equation*}
    for all $ f \in L^2_{\mathbb{ P }_{X}}(\mathcal  X, \mathcal  H_{\mathcal Y}) $
    and $ x \in \mathcal  X $. Moreover, if $ k_{\mathcal  X} $ and $
    k_{\mathcal  Y} $ are $ C_0 $-universal \footnote{A kernel $ k $ is $ C_0
        $-universal if the associated RKHS $ \mathcal  H_k $ is dense in $
    C_0(\mathcal  X) $, the space of continuous functions on $ \mathcal  X $
    vanishing at infinity \cite{sriperumbudur2010relation} } 
    \begin{equation*}
        \operatorname{ACMMD}(\mathbb{ P }_{|}, Q_{|}) = 0 \iff \mu_{\mathbb{ P }_{|x}} = \mu_{Q_{|x}}, \quad  \mathbb{ P }_X \text{-a.e.}
    \end{equation*}
\end{lemma*}
\begin{proof}
    Let us introduce the notations used in this proof.
    Let $ (\Omega, \mathcal  F, \mathbb{ P }) $ be the sample space, $ X(\omega), Y(\omega), \tilde{Y}(\omega) $
    being random variables on $ \Omega $ corresponding to the input, target and the model.
    When clear, we will identify the measure $ \mathbb{ P } $ and the push-forwards
    $ Y_{\#}\mathbb{ P } $, $ \tilde{Y}_{\#}\mathbb{ P } $ and drop the dependence of 
    $Y, \tilde{Y} $ on $ \omega $. Given $ x \in \mathcal  X $, we write $ K_{x}  $
    the linear operator from $ \mathcal H_{\mathcal  Y}$ to  $ \mathcal  L(\mathcal  X, \mathcal  H_{\mathcal  Y}) $,
    the space of linear operators from $ \mathcal  X $ to $ \mathcal  H_{\mathcal  Y} $, such that
    $ (K_{x} f)(x') = K_{\mathcal  X}(x, x') f \in \mathcal  H_{\mathcal  Y} $ for all $ f \in \mathcal  H_{\mathcal  Y} $.
    When no confusion is possible, we may identify the notations $ k_{\mathcal  Y}(y, \cdot) $ and $ k_{y}  $.

    The existence of the conditional mean embeddings $ \mu_{\mathbb{ P }_{|}}$
    and $ \mu_{Q_{|}} $ is guaranteed by
    \citep[Definition~3.1]{park2020measure} under the integrability assumption
    $ \int_{  }^{  } k_{\mathcal  Y}(y, y) \text{d} \mathbb P(y) < +\infty $ and 
    $ \int_{  }^{  } k_{\mathcal  Y}(\tilde{y}, \tilde{y}) \text{d} \mathbb P(\tilde{y}) < +\infty $.
    The second integrability assumption $ \int_{  }^{  } k_{\mathcal  X}(x, x)k_{\mathcal  Y}(y, y)
    \text{d} (\mathbb P_X \otimes \mathbb{ P }_{|})(x, y) = \int_{  }^{ } k_{\mathcal  X} \otimes k_{\mathcal
    Y}((x, y), (x, y)) \text{d} (\mathbb P_X \otimes \mathbb{ P }_{|})(x, y) < +\infty $ guarantees the
    existence of the mean embedding $ \mu_{ \mathbb{  P}_{X} \otimes \mathbb{ P }_{|}} $, defined as:
    \begin{equation*} 
        \mint k_{\mathcal  X}(x, \cdot) \otimes k_{\mathcal  Y}(y, \cdot) 
            \text{d} (\mathbb P_X \otimes \mathbb{ P }_{|})(x, y) 
            \in \mathcal  H_{\mathcal  X} \otimes \mathcal  H_{\mathcal  Y}
    \end{equation*}
    by \citep[Lemma~3]{gretton2012kernel} (and respectively for $ Q_{|} $).
    Here, $ \mathcal  H_{\mathcal  X} \otimes \mathcal  H_{\mathcal  Y} $ is the tensor product
    Hilbert space of $ \mathcal  H_{\mathcal  X} $ and $ \mathcal  H_{\mathcal  Y} $, with 
    kernel $ k_{\mathcal  X} \otimes k_{\mathcal  Y} $.
    We actually prove a stronger form of the lemma, given by removing the norm
    from both hands of the equality and replacing it with a suitable isometric
    isomorphism $ \phi: \mathcal  H_{\mathcal  X} \otimes \mathcal  H_{\mathcal  Y} 
    \longmapsto \mathcal  H_{\mathcal  X, \mathcal  H_\mathcal  Y} $
    \begin{equation*} 
        \phi \left (\mint k_{\mathcal  X}(x, \cdot) \otimes k_{\mathcal  Y}(y, \cdot) 
            \text{d} (\mathbb{  P}_{X} \otimes \mathbb{ P }_{|})(x, y)\right) 
        = \mint K_{x} \mu_{\mathbb{  P}_{|}} \text{d} \mathbb P_X(x)
    \end{equation*}
    This isometric isomorphism is shown to exist in the ``Currying lemma'' of
    \citet[Example~6]{carmeli2010vector} regarding tensor product kernels (note
    that both $ \mathcal  X $ -- by assumption -- and $ \mathcal  Y $ are
    locally compact and second-countable). This lemma shows that the mapping:
    \begin{align}
        \phi: \mathcal  H_{\mathcal  X} \otimes \mathcal  H_{\mathcal  Y} &\longrightarrow \mathcal  F(\mathcal  X, \mathcal  H_{\mathcal  Y})  \\
        f \otimes g &\longmapsto \phi(f \otimes g ) = (x \in \mathcal  X  \longmapsto f(x) g \in \mathcal  H_{\mathcal Y})
    \end{align}
    is an isometric isomorphism between $ \mathcal  H_{\mathcal  X} \otimes
    \mathcal  H_{\mathcal  Y} $ and $ \mathcal  H_{\mathcal  X, \mathcal  H_\mathcal  Y} $.
    This lemma gives both a representation formula for elements of $ \mathcal
    H_{\mathcal  X, \mathcal  H_{\mathcal  Y}} $, and a way to formalize the currying operation,
    (e.g. the transformation of a function of two variables into a higher-order
    function of one variable and returning a function of one variable) on
    tensor-product spaces, since $ (f \otimes g)(x, y) = (\phi(f \otimes
    g)(x))(y) $. We refer to \citet[Example~6]{carmeli2010vector} for a proof.
    Proceeding with the proof of \cref{lemma:cgof_mmd}, when $ f $ and $ g $ are kernel
    functions $ k_{\mathcal  X}(x, \cdot) $ and $ k_{\mathcal  Y}(y, \cdot) $,
    the right-hand side of the equality can be related to $ K_{x} $ as
    \begin{align*}
        \phi(k_{\mathcal  X}(x, \cdot) \otimes k_{\mathcal  Y}(y, \cdot))(x')
                &= k_{\mathcal  X}(x, x') k_{\mathcal  Y}(y, \cdot) \\ 
                &= K_{x'}^{\star} K_{x} k_{\mathcal  Y}(y, \cdot) \\ 
                &= K_{x} k_{y}(x')
    \end{align*}
    where the second to last equality follows from the reproducing property of
    $ \mathcal  H_{\mathcal  X, \mathcal  H_{\mathcal  Y}} $. Since $ \phi $ is
    linear and unitary, it commutes with the mean embedding operation:
    \citep[Theorem~36]{dinculeanu2000vector}, yielding:
    \begin{align*}
        \phi \left (\mint (k_{\mathcal  X}(x, \cdot) \otimes k_{\mathcal  Y}(y, \cdot))  \text{d} (\mathbb P_{X} \otimes \mathbb{ P }_{|})(x, y) \right ) 
            &= \mint \phi(k_{\mathcal  X}(x, \cdot) \otimes k_{\mathcal  Y}(y, \cdot)) \text{d} (\mathbb P_{X} \otimes \mathbb{ P }_{|})(x, y) \\
            &= \mint K_{x}k_{y} \text{d} (\mathbb P_{X} \otimes \mathbb{ P }_{|})(x, y)
    \end{align*}
    To complete the proof, it remains to relate the right-hand side to the
    conditional mean embedding $ \mu_{\mathbb{ P }_{|}} $, using
    \begin{align*}
        \int_{}^{  } K_{x}k_{y} \text{d} (\mathbb P_{X} \otimes \mathbb{ P }_{|})(x, y) 
            &= \miint K_{x}k_{y} \text{d} \mathbb P_{X}(x) \text{d} \mathbb P_{|x}(y) \\
            &= \mint K_{x} \int_{  }^{  } k_{y} \text{d} \mathbb P_{|x}(y) \text{d} \mathbb P_{X}(x) \\
            &= \mint K_{x} \mu_{\mathbb{ P }_{|}}(x) \text{d} \mathbb P_X(x)
    \end{align*}
    as $ K_{x} $ is a bounded linear operator.
    We thus have that:
    \begin{equation*} 
    \phi \left (\mint k_{\mathcal  X}(x, \cdot) \otimes k_{\mathcal Y}(y, \cdot) \text{d} (\mathbb{ P }_{|X} \otimes \mathbb P_{|})(x, y) \right ) = \mint K_{x}
    \mu_{\mathbb{ P }_{|}} \text{d} \mathbb P_X(x)
    \end{equation*}
    Combining this with the analogue of this result holding for $ \mu_{Q_{|}} $ allows
    showing the stronger form of \cref{lemma:cgof_mmd}.
    Let us now prove the second part of the lemma. Assume
    $ \operatorname{ACMMD}(\mathbb{P}_{|}, Q_{|}) = 0 $, meaning
    \begin{equation*} 
    T_{K_{\mathcal  X}} (\mu_{\mathbb{ P }_{|}} - \mu_{Q_{|}}) = 0
    \end{equation*}
    By \citet[Theorem 2]{carmeli2010vector} $ K_{\mathcal X } $ is a $ C_0 $-universal
    operator-valued kernel, the operator $ T_{K_{\mathcal  X}} $ is injective.
    This implies that the conditional mean embeddings of $ \mathbb{ P }_{|x} $
    and $ Q_{|x} $ are equal $ \mathbb{ P }_{X} $--almost everywhere. By
    \citet[Theorem 5.2]{park2020measure} applied to the case where the
    marginals are equal, and since $ k_{\mathcal  X} \otimes k_{\mathcal  Y} $ is $ C_{0} $--universal,
    this implies that $ \mathbb{ P }_{|x} = Q_{|x} $, $ \mathbb{ P }_{X}
    $--almost everywhere, and in summary, $ \operatorname{ACMMD}(\mathbb{ P }_{|}, Q_{|}) = 0 $
    implies $ \mathbb{ P }_{|x} = Q_{|x} $, $ \mathbb{ P }_{X} $--almost
    everywhere. To prove the reverse direction, assume that $ \mathbb{ P }_{|x}
    = Q_{|x} $, $ \mathbb{ P }_{X} $--almost everywhere Since \citet[Theorem
    5.2]{park2020measure} also prove the reverse direction of the statement
    relied upon in the previous argument, we have that conversely
    $ \mu_{\mathbb{ P }_{|}}(x) = \mu_{Q_{|}}(x) $, $ \mathbb{ P }_X $--almost everywhere.
    By linearity of $ T_{K_{\mathcal  X}} $, we thus have that
    $ T_{K_{\mathcal  X}} (\mu_{\mathbb{ P }_{|}} - \mu_{Q_{|}}) = 0 $, and
    therefore $ \operatorname{ACMMD}(\mathbb{ P }_{|}, Q_{|}) = 0 $.
\end{proof}

\section{Asymptotic distribution of \texorpdfstring{$ \widehat{\operatorname{ACMMD}}\vphantom{a}^2 $}{ACMMD² estimates}}
\label{app-sec:cgof_mmd_estimator_asymptotic_distribution}
As discussed in the main text, it is possible to characterize the asymptotic distribution of $  N
\widehat{ \operatorname{ACMMD} }\vphantom{a}^2   $. When $ \mathbb{ P }_{|} = Q_{|} $,
and $ \sqrt {N} (\widehat{ \operatorname{ACMMD} }\vphantom{a}^2 - \operatorname{ACMMD}^2) $ when $ \mathbb{ P }_{|} \neq Q_{|} $. This characterization is given in the next lemma.
\begin{lemma}\label{lemma:cgof_mmd_estimator_asymptotic_distribution}
    Assume that the integrability assumptions of \cref{lemma:cgof_mmd} hold, and
    that $ \mathbb{E}_{ Z_1, Z_2 } h(Z_1, Z_2)^2 < +\infty $, and that 
    \begin{itemize}
        \item if $ \mathbb{ P }_{|x} = Q_{|x} $ $ \mathbb{ P }(X) $--a.s, then
            $ \mathbb{ E } \lbrack \widehat{\operatorname{ACMMD}} \rbrack = 0 $ and 
            \begin{equation*} 
                N \widehat{ \operatorname{ACMMD} }\vphantom{a}^2 \xrightarrow{ d } \sum\limits_{ j=1 }^{ \infty } \lambda_j (\chi_{1j}^2 - 1)
            \end{equation*}
            where $ \{ \chi^2_{1j} \}_{j=1}^{\infty}  $ are independent random $ \chi_1^2 $ variables, and 
            $ \lambda_j $ are the eigenvalues of the operator  defined as:
            \begin{equation*} 
            \phi  \longmapsto \int_{  }^{  } h(z, \cdot) \phi(z) \text{d} \mathbb P(z)
            \end{equation*}
        \item Assume moreover that $ k_{\mathcal  X} $ and $ k_{\mathcal  Y} $ are $ C_0 $-universal kernels, and that
            $ \sigma^2_h = 4 \mathbb{ V }_{Z_2} \lbrack \mathbb{ E }_{Z_1}h(Z_1, Z_2) \rbrack > 0  $. Then
            \begin{equation*} 
                \sqrt {N}  ( \widehat{ \operatorname{ACMMD} }\vphantom{a}^2 - \operatorname{ACMMD}\vphantom{a}^2  \rbrack ) \xrightarrow{ d } \mathcal{ N }(0, \sigma^2_{h})
            \end{equation*}
    \end{itemize}
\end{lemma}

\begin{proof}
    Since $ \mathbb{ E }_{Z_1, Z_2} h(Z_1, Z_2)^2 < +\infty $ we have that $
    \mathbb{ V }_{Z_1, Z_2}\mathbb{ E }_{Z_1, Z_2} h(Z_1, Z_2)^2 < +\infty $. Let us define, as in 
    \citet[Section 5.1.5]{serfling2009approximation}, the function
    $ h(z) = \mathbb{ E }_{Z_2} h(z, Z_2) $, and define $ \zeta \coloneqq
    \mathbb{ V }_{z} h $. For the first point, we will show that if $ \mathbb{
    P }_{|} =  Q_{|} $, $ \mathbb{ P }$--a.s, then $ \zeta = 0 $, and the result
    will follow from \citet[Setion 5.5.2]{serfling2009approximation}. Indeed,
    noting $ k = k_{\mathcal  X} \otimes k_{\mathcal  Y} $,
    \begin{align*}
        h(z) &= \mathbb{ E }_{Z_2} h(z, Z_2) \\
             &= \mathbb{ E }_{Z_2} \left \langle k((x, y), \cdot) - k((x, y'), \cdot), k((X_2, Y_2), \cdot) - k(X_2, \tilde{Y}_2) \right \rangle  \\
             &= \left \langle k((x, y), \cdot), - k((x, y'), \cdot), \mathbb{ E }_{z_2}  \left \lbrack  k((X_2, Y_2), \cdot) - k((X_2, \tilde{Y}_2), \cdot)\right \rbrack  \right \rangle  \\
    \end{align*}
    Where we exchanged the order of integration and inner product, which is
    possible since $ h  \longmapsto \left \langle k((x, y), \cdot) - k((x,
    \tilde{y}), \cdot), h \right \rangle  $ is a bounded linear functional for all
    $ (x, y, \tilde{y}) $. Now,
     \begin{equation*} 
        \mathbb{ E }_{z_2} k((X_2, Y_2), \cdot) - k((X_2, \tilde{Y}_2), \cdot) 
        = \mathbb{ E }_{\mathbb{ P }_X} \left \lbrack 
            \mathbb{ E }_{\mathbb{ P }|} k((X_2, Y_2), \cdot ) - \mathbb{ E }_{Q_{|}} k((X_2, \tilde{Y}_2), \cdot)
          \right \rbrack  
        = 0
    \end{equation*}
    since $ \mathbb{ P }_{|x} = Q_{|x} $ $ \mathbb{ P }_X $--a.s.
    Thus, $ h(z) $ is a constant function, and $ \zeta = 0 $.
    The second case follows by assumption from \citet[Section
    5.1.1]{serfling2009approximation}.
\end{proof}

\subsection{Proof of Lemma \ref{lemma:cgof_mmd_double_expectation}}\label{sec:proof_cgof_mmd_double_expectation}
    Let $ \mathcal  H_{\mathcal  X \times \mathcal  Y} \coloneqq \mathcal  H_{\mathcal  X} \otimes \mathcal  H_{\mathcal  Y} $ be the tensor-product RKHS of functions from
    $ \mathcal  X \times \mathcal  Y $ with kernel $ k_{\mathcal  X} \otimes k_{\mathcal  Y} $.
    The result can be obtained by applying a ``coupling'' argument, and starting from the following object:
    \begin{equation} \label{eq:coupling}
    \begin{aligned}
        \mu_{\mathbb{ P }_{X}\otimes \mathbb{ P }_{|} - \mathbb{ P }_{X}\otimes \mathbb{ Q }_{|}} 
            &\coloneqq \int_{  }^{  } \left ( k((x(\omega), y(\omega)), \cdot) - k((x(\omega), \tilde{y}(\omega)), \cdot)\right ) \text{d} \mathbb P(\omega) \\ 
            &= \mathbb{E}_{ x, y, \tilde{y} }\left [ k((x, y), \cdot) - k((x, \tilde{y}), \cdot) \right ] 
    \end{aligned}
    \end{equation}
    We first show that $ \mu_{\mathbb{ P }_{X}\otimes \mathbb{ P }_{|} - \mathbb{ P }_{X}\otimes \mathbb{ Q }_{|}} $
    is a well-defined element of $ \mathcal  H_{\mathcal  X \times \mathcal  Y} $.
    Indeed, the following operator
    \begin{equation*} 
        T:  f \in \mathcal H \longmapsto \mathbb{ E  }_{z} f((x, y)) - f((x, \tilde{y}))
    \end{equation*}
    satisfies 
    \begin{align*}
        |Tf| &\leq \mathbb{E}_{  }\left [ |f(x, y)| + |f(x, \tilde{y})| \right ]  \\
             &\leq\left \|f \right \|_{\mathcal  H_{\mathcal  X \times \mathcal  Y}}( \mathbb{ E }\sqrt { k((x, y), (x, y))}  \\
             & \quad \quad + \mathbb{ E }\sqrt { k((x, \tilde{y}), (x, \tilde{y}))} ) \\
    \end{align*}
    and is bounded thanks to the integrability assumptions of \cref{lemma:cgof_mmd}. Applying the same argument as \citep[Lemma~3]{gretton2012kernel},
    it follows that the object in \cref{eq:coupling} is well-defined and belongs to $ \mathcal  H_{\mathcal  X \times \mathcal  Y} $.
    Furthermore, by linearity of integration, we have that:
    \begin{align*}
    &\int_{  }^{  } \left ( k((x(\omega), y(\omega)), \cdot) - k((x(\omega), \tilde{y}(\omega)), \cdot)\right ) \text{d} \mathbb P(\omega) \\ 
    &= \int_{  }^{  } k((x(\omega), y(\omega)), \cdot) \text{d} \mathbb P(\omega) - \int_{  }^{  } k((x(\omega'), \tilde{y}(\omega')), \cdot) \text{d} \mathbb P(\omega') \\
    &= \mu_{\mathbb{ P }_{X} \otimes \mathbb{ P }_{|}} - \mu_{\mathbb{ P }_{X} \otimes Q_{|}}
    \end{align*}
    To conclude, note that:
    \begin{align*}
    \operatorname{ACMMD}(\mathbb{ P }_{|}, Q_{|})^2  &= \left \| \mu_{\mathbb{ P }_{X} \otimes \mathbb{ P}_{|} - \mathbb{ P }_{X} \otimes Q_{|}} \right \|_{\mathcal  H_{\mathcal  X \times \mathcal  Y}}^2 \\ 
                                  &= \left \langle  \mu_{\mathbb{ P }_{X} \otimes \mathbb{ P}_{|} - \mathbb{ P }_{X} \otimes Q_{|}}, \mu_{\mathbb{ P }_{X} \otimes \mathbb{ P}_{|} - \mathbb{ P }_{X} \otimes Q_{|}} \right \rangle_{\mathcal  H_{\mathcal  X \times \mathcal  Y}} \\
                                  &= \left \langle \mathbb{ E }_{x, y, \tilde{y}} \left [ k((x, y), \cdot) - k((x, \tilde{y}), \cdot) \right ],  \mathbb{ E }_{x, y, \tilde{y}} \left [ k((x, y), \cdot) - k((x, \tilde{y}), \cdot) \right ] \right \rangle_{\mathcal  H_{\mathcal  X \times \mathcal  Y}} \\
                                  &= \mathbb{ E }_{x_1, y_1, \tilde{y}_1} \mathbb{ E }_{x_2, y_2, \tilde{y}_2} h((x_1, y_1, \tilde{y}_1), (x_2, y_2, \tilde{y}_2)) \\
    \end{align*}
    Where the last equality was obtained by exchanging the order of integration
    and dot product, possible thanks to the integrability assumptions of
    \cref{lemma:cgof_mmd}, by using the bilinearity of the inner product
    and the reproducing property of the kernel $ k $. The symmetry
    of $ h $ in $ (Z_1, Z_2) $ follows from the symmetry of $ k_{\mathcal  X} \otimes k_{\mathcal  Y} $.

\subsection{Proof of Lemma \ref{lemma:cgof_mmd_estimator}}\label{lemma:proof_cgof_mmd_estimator}
\begin{proof}
    The proof of the unbiasedness of $ \widehat{ \operatorname{ACMMD} }\vphantom{a}^2 $ follows by linearity of the expectation,
    and that each $h((X_{i}, Y_{i}, \tilde{Y}_{i}), (X_{j}, Y_{j}, \tilde{Y}_{j}))$ is an unbiased estimator
    of $ \operatorname{ACMMD}\vphantom{a}^2 (\mathbb{ P }_{|}, Q_{|}) $.
\end{proof}

\section{Type-I error control of the ACMMD test}\label{app-sec:type_I_error_control}
The goal of this section is to show that the ACMMD test is guaranteed to control its type-I error rate at level $ \alpha $.

\subsection{Quantile estimation and Decision Rule} \label{sec:decision_rule}
We first fully specify the way we compute our quantile estimate  $ \widehat{ q }_{1-\alpha} $.
Let $b_{\alpha} \coloneqq \lceil { (1 - \alpha) (B+1) } \rceil $.
Given $ B $ bootstrap samples $ \{ \widetilde{ \mathrm{ACMMD} }\vphantom{a}^{2}_{b} \}_{b=1}^{B} $
and an $ \widehat{ \operatorname{ACMMD} }\vphantom{a}^{2} $ estimate, we order them in increasing order
in a sequence of size $ B+1 $, with ties broken arbitrarily.
Let
$ m = \min_{  } \{ b \in [\![1, B+1 ]\!] \, | \, \widetilde{ \mathrm{ACMMD} }\vphantom{a}^{2}_{b} = \widetilde{ \mathrm{ACMMD} }\vphantom{a}^2_{b_{\alpha}} \} $, and
$ M = \max_{  } \{ b \in [\![1, B+1 ]\!]  \, | \, \widetilde{ \mathrm{ACMMD} }\vphantom{a}^{2}_{b} = \widetilde{ \mathrm{ACMMD} }\vphantom{a}^{2}_{b_{\alpha}} \} $.
We set $ \widehat{q}_{1-\alpha} $
to be the $ (m - 1) $-th element
with probability $ (b_{\alpha} - (1 - \alpha) (B+1)) / (M - m + 1)$
(with the convention that the $ 0 $-th element is $ -\infty $),
and the $ b_{\alpha} $-th element otherwise
The decision rule is then to reject the null hypothesis if $ \widehat{ \mathrm{ACMMD} }\vphantom{a}^{2} >  q_{1-\alpha} $.

\subsection{Wild-bootstrap and permutation-based approaches are equivalent in the ACMMD test}\label{sec:proof-wild-bootstrap-acmmd}
To show that the ACMMD test is guaranteed to control its type-I error rate at level $ \alpha $,
we show that the use of a wild bootstrap procedure in the ACMMD test can be cast as a
computationally efficient way to approximate the quantiles of the random variable $ \widehat{ \operatorname{ACMMD} }\vphantom{a}^2 $
when $ \mathbb{ P }_{|x} = Q_{|x} $ $ \mathbb{ P }_X $--a.e.
\begin{lemma}\label{lemma:wild-bootstrap}
    Let  $ \{ Z_{i} \}_{i=1}^{N}  $ be i.i.d realizations of $ \mathbb{ P }_X
    \otimes \mathbb{ P }_{|} \otimes Q_{|} $, and let $ \{ W_{i}^{b} \}^{b=1 \dots B}_{i=1
    \dots N}$ be i.i.d. Rademacher random variables independent
of the data. 
    Given a function $ \sigma: [\![ 1, N]\!]  \longmapsto \{-1, 1\}  $,
    define $\{Z^{\sigma}_{i} \}_{i=1}^{N} \coloneqq \{ X_{i}, Y_{i}^{\sigma}, \tilde{Y}^{\sigma}_{i} \}_{i=1}^{N}  $,
    where $ (Y^{\sigma}_{i}, \tilde{Y}^{\sigma}_{i}) = (Y_{i}, \tilde{Y}_{i}) $ if $ \sigma(i) = 1$, and $ (\tilde{Y}_{i}, Y_{i}) $ otherwise.
    Then we have:
    \begin{equation*} 
    \begin{aligned}
        \widetilde{\operatorname{ACMMD}}\vphantom{a}^{2}_b 
        = \frac{2}{N(N-1)} \sum\limits_{ \substack{ i, j = 1 \\ i < j } }^{ N } 
            h(Z_i^{\sigma_b}, Z_j^{\sigma_b})
        \coloneqq \widehat{ \operatorname{ACMMD} }\vphantom{a}^{2}_{\sigma_b}
    \end{aligned}
    \end{equation*}
    for $ \sigma_b(i) \coloneqq W^b_i$.
\end{lemma}
The $ W^{b}_i $ should be understood as elements of a random swap $ \sigma_b $, which for each $ i $,
swaps $ Y_{i} $ and $ \tilde{Y}_{i} $ with probability $ 1/2 $.
\begin{proof}
    Without loss of generality, we fix $ i = 1 $ and $ j = 2 $, and fix $ b $, dropping the $ b $ index.
    Note that $ h(Z_1, Z_2)$ and $ h(Z^{\sigma}_1, Z^{\sigma}_2) $
    share the same $ k_{\mathcal  X}(X_1, X_2) $. The only differing term is
    \begin{equation*}  
    \begin{aligned}
        g((Y_1, \tilde{Y}_1), (Y_2, \tilde{Y}_2)) 
        \coloneqq
        k_{\mathcal  Y}(Y_1, Y_2)) + k_{\mathcal  Y}(\tilde{Y}_1, \tilde{Y}_2) 
            - k_{\mathcal  Y}(Y_1, \tilde{Y}_2) - k_{\mathcal  Y}(\tilde{Y}^1, Y^2)
    \end{aligned}
    \end{equation*}
    and we only need to show that
    $ W_1W_2g((Y_1, \tilde{Y}_1), (Y_2, \tilde{Y}_2)) = g((Y_1^{\sigma}, \tilde{Y}^{\sigma}_1), (Y_2^{\sigma}, \tilde{Y}^{\sigma}_2)) $.
    \paragraph{Case $ W^1 = W^2 = 1 $}
    In that case,
    $ Z_1 = Z^{\sigma}_1 $ and $ Z_2 = Z^{\sigma}_2 $, and $ W_1W_2h(Z_1, Z_2) = h(Z_1, Z_2)= h(Z^{\sigma}_1, Z_2^{\sigma}) $
    by definition of $ \sigma $.
    \paragraph{Case $ W_1 = W_2 = -1 $}
    In that case, we have:
    \begin{equation*} 
        g((Y_1^{\sigma}, \tilde{Y}^{\sigma}_1), (Y_2^{\sigma}, \tilde{Y}^{\sigma}_2)) = 
        k(\tilde{Y}_1, \tilde{Y}_2) + k(Y_1, Y_2) k(\tilde{Y}_1, Y_2) - k(Y_1, \tilde{Y}_2))  = h(Z_1, Z_2)
    \end{equation*}
    implying again $ W_1W_2h(Z_1, Z_2) = h(Z_1, Z_2) = h(Z^{\sigma}_1, Z^{\sigma}_2) $.

    \paragraph{Case $ W_1 = 1 $ and $ W_2 = -1 $}
    In that case, we have:
    \begin{align*}
        h(Z^{\sigma}_1, Z^{\sigma}_2) 
        &= k((X_1, Y_1), (X_2, \tilde{Y}_2)) + k((X_1, \tilde{Y}_1), (X_2, Y_2)) 
            - k((X_1, Y_1), (X_2, Y_2)) - k((X_1, \tilde{Y}_1), (X_2, \tilde{Y}_2)) \\ 
        & =-h(Z_1, Z_2)
    \end{align*}
    meaning again $ W_1W_2h(Z_1, Z_2) = -h(Z_1, Z_2) = h(Z^{\sigma}_1, Z^{\sigma}_2) $,
    and the last case is proved similarly.
\end{proof}

\subsection{Level of the ACMMD test} \label{sec:level_acmmd}
We now show that the ACMMD test has the desired type-I error rate.
\begin{lemma}\label{lemma:level_acmmd}
    Assume that $ \mathbb{ P }_{|x} = Q_{|x} $ $ \mathbb{ P }_X $--a.s.
    Then  the probability that the ACMMD test rejects the null hypothesis is exactly $ \alpha $.
\end{lemma}
The proof consists in 2 steps. First, we show that the decision rule is equivalent to a simpler one.
Then, we analyze the latter decision rule.
\paragraph{An equivalent decision rule}
This decision rule is equivalent
to the one rejecting $ H_0 $ if the position $ Q $ (with ties broken uniformly at random)
of  $ \widehat{ \mathrm{ACMMD} }\vphantom{a}^{2} $ in that sequence satisfies
$ Q > b_{\alpha} $, accepting it if   $ Q < b_{\alpha} $,
and rejecting it with probability $ b_{\alpha} - (1 - \alpha) (B + 1) $ if $ Q = b_{\alpha} $:
Indeed, 
$ Q > M \iff \widehat{ \mathrm{ACMMD} }\vphantom{a}^{2} > q_{1-\alpha} $ (we always reject),
$ Q < m \iff \widehat{ \mathrm{ACMMD} }\vphantom{a}^{2} \leq q_{1-\alpha}$ (we never reject), and
for both rules, when the random position $ Q$ is in  $ [\![m, M]\!]$, the null is rejected 
with probability $ (b_{\alpha} - \alpha (B+1)) / (M - m + 1) $.

\paragraph{Analysis of the decision rule}
We derive the type-I error of our decision rule by analyzing the equivalent, latter one. 
Our analysis follows a similar argument, in flavor, as \citet[Appendix C]{Domingo-Enrich2023-eu}.
Now, recall that from \cref{lemma:wild-bootstrap}, the wild bootstrap quantile estimation
are draws of $ \widehat{ \mathrm{ACMMD} } $ on swapped samples $ Z^{\sigma} $, e.g.
$ \{ X_i, Y_i^{\sigma}, \tilde{Y}^{\sigma}_i \}_{i=1}^{N}  $ parameterized by
$ \sigma: [\![ 1, N]\!]  \longmapsto \{-1, 1\}  $
where $ Y_i^{\sigma} = Y_i$ if $ \sigma(i) \coloneqq w_i $ and $ Y_i^{\sigma} = \tilde{Y}_i $ otherwise:
\begin{equation*} 
    (\widetilde{\operatorname{ACMMD}}\vphantom{a}^{2}_b)_{b=1}^{B} = (\widehat{ \operatorname{ACMMD} }\vphantom{a}^{2}_{\sigma_b})_{b=1}^{B}
\end{equation*}
using the notation of \cref{lemma:wild-bootstrap}.
Note that $ \sigma$ is a random swap operator such that $ \sigma(i) = 1 $ with probability $ 0.5 $, and
$ \sigma(i) = -1 $ with probability $ 0.5 $.
If $ \mathbb{ P }_{|x} = Q_{|x} $ a.e., then
since the $ B $ swap maps $ \sigma_1, \dots, \sigma_B  $ are i.i.d.
let us note $ \widehat{ \mathrm{ACMMD} }\vphantom{a}^{2}_{\sigma_0} = \widehat{ \mathrm{ACMMD} }\vphantom{a}^{2} $, e.g. $ \sigma_0(i) = 1 $.
Then the random sequence $ (\widehat{ \mathrm{ACMMD} }\vphantom{a}^{2}_{\sigma_b})_{b=0}^{B}  $ is exchangeable.
Since Q is the position of $ \widehat{\operatorname{ACMMD}}\vphantom{a}^{2} $
within that sorted sequence, and that all positions are equally likely under exchangeability,
we have:
\begin{align*}
& \mathbb{P}\left[Q < m \right]=1 /(B+1) \\
& \mathbb{P}\left[Q>b_\alpha\right]=\left(B+1-b_\alpha\right) /(B+1) \\
& \mathbb{P}\left[Q<b_\alpha\right]=\left(b_\alpha-1\right) /(B+1)
\end{align*}

Noting $ \Delta((X^{i}, Y^{i}, \tilde{Y}^{i})_{i=1}^{N}) $ the event that the null hypothesis is rejected, we have:
\begin{align*}
    \mathbb{P}\left[\Delta((X^{i}, Y^{i}, \tilde{Y}^{i})_{i=1}^{N})  \right] & =\mathbb{P}\left[Q>b_\alpha\right]+\mathbb{P}\left[Q=b_\alpha\right] \mathbb{ P }[\text{Reject} | Q = b_{\alpha}] \\
& =\left(B+1-b_\alpha\right) /(B+1)+\left(b_\alpha-(1-\alpha)(B+1)\right) /(B+1)=\alpha,
\end{align*}

thus showing that the ACMMD test has the desired type-I error rate. \qed

\section{Proofs related to ACMMD--Rel}

\subsection{Differences between the SKCE U-statistics and the ACMMD U-statistic}\label{app-sec:skce-vs-acmmd}

We recall the definition of the SKCE U-statistics estimator from \citep[Lemma 2]{widmann2022calibration}:
\begin{equation} \label{eq:SKCE-u-statistics-estimator}
    \widehat{\operatorname{SKCE}} = \frac{2}{N(N-1)} \sum\limits_{1 \leq i < j \leq N}G((Q_{|X_i}, Y_i), (Q_{|X_j}, Y_j))
\end{equation}
where
\begin{equation}\label{eq:G}
\begin{aligned}
    G((q, y), (q', y')) 
        &\coloneqq  k_{\mathcal  P(\mathcal  Y) \times \mathcal  Y}((q, y), (q', y')) 
            - \mathbb E_{Y \sim q} k_{\mathcal  P(\mathcal  Y) \times \mathcal  Y}((q, Y), (q', y'))\\
        &
            - \mathbb E_{Y' \sim q'} k_{\mathcal  P(\mathcal  Y) \times \mathcal  Y}((q, y), (q', Y')) 
            + \mathbb E_{Y \sim q} \mathbb E_{Y' \sim q'} k_{\mathcal  P(\mathcal  Y) \times \mathcal  Y}((q, Y), (q', Y')). \\
        &= k_{\mathcal  P(\mathcal  Y)}(q, q') \times \left (
             k_{ \mathcal  Y}(y, y')- \mathbb E_{Y \sim q} k_{\mathcal   Y}(Y, y') 
             - \mathbb E_{Y' \sim q'} k_{\mathcal  Y}(y, Y') 
             + \mathbb E_{Y \sim q} \mathbb E_{Y' \sim q'} k_{\mathcal  Y}(Y, Y') 
           \right ) \\
\end{aligned}
\end{equation}
Where the second equality holds when focusing on tensor product kernels.
Comparing \cref{eq:SKCE-u-statistics-estimator} and \cref{eq:G} with the
expression of the ACMMD U-statistics estimator given in
\cref{eq:cgof_mmd_estimator} and \cref{lemma:cgof_mmd_double_expectation}, we
see that the SKCE population criterion equals the ACMMD. However, the SKCE
U-statistics estimator is different from the ACMMD U-statistics estimator:
while the ACMMD U-statistics only requires samples the conditional
distributions $Q_{|X}$, the SKCE U-statistics contains expectations over the
conditional distributions $Q_{|X}$, which are rarely available in practice.

\subsection{Proof of Proposition \ref{prop:universal_kernel_on_PYS}}\label{proof:universal_kernel_on_PYS}
\begin{proof}
    We will show that the image of $q\mapsto \mu_{q}$ is compact, and the
    result will follow from \cite{Christmann2010-ox}. Let $\mathcal M(\mathcal
    Y)$ be the Banach space of measures of sequences endowed with the total
    variation norm:
    \begin{equation*} 
    \|q\|_{\mathrm{TV}} \coloneqq  q_+(\mathcal  Y) + q_{-}(\mathcal  Y)
    \end{equation*}
    We recall that by the Riesz-Markov theorem, ${(\mathcal M(\mathcal  Y),\|\cdot\|_{\textrm{TV}})}$
    can be identified with the topological dual  of
    $C_0(\mathcal  Y)$, $(C_0(\mathcal  S)^\star, \|\cdot\|_{\textrm{op}})$  through an isometric
    isomorphism $q \in \mathcal M(\mathcal  Y)  \longmapsto \tilde q \in C_0(\mathcal  Y)^\star
    $, and for which the following holds:
    \begin{equation*} 
    \tilde q(f) = \int_{\mathcal  Y} f \mathrm{d} q, \quad \forall f \in C_0(\mathcal  Y).
    \end{equation*}
    Let $B\coloneqq\{q\in\mathcal M(\mathcal  Y)\ |\ \Vert q\Vert_{\mathrm{TV}}\leq 1\}$.
    As a unit ball, by the Banach-Alaoglu theorem, $B$ is compact under the
    weak--$\star$ topology and contains all distributions on sequences. We
    will show that $q\mapsto \mu_{q}$ is continuous on $B$ and the result
    will follow. Given that this mapping is linear, it is sufficient to show
    continuity at 0. Moreover, since $\{B(0_\mathcal H, r)\}_{r>0}$ is a
    neighborhood basis of $(\mathcal H, \|\cdot\|_\mathcal H)$, it suffices to
    show that there is a neighborhood $\mathcal V$ of the null measure in the
    weak-$\star$ topology such that $\|\int k_{\mathcal  Y}(y, \cdot) dq(y)\|^2_{\mathcal H} = \int k_{\mathcal  Y}(y, y') \mathrm{d} (q \otimes q)(y, y') < 1 $ for all $q$ in $\mathcal V$.  Since the family
    \begin{equation*} 
    \{q \in \mathcal M(\mathcal  Y), \int f_i (x) \mathrm{d} q(x) < \epsilon,\,\, i\in 1, \dots k, f_i \in C_0(\mathcal  Y)\}
    \end{equation*}
    form a neighborhood basis of the weak-$\star$ topology, we can consider
    candidates of this form for $\mathcal V$. In particular, let us set
    $\{f_i\} = \{x \longmapsto \sqrt{k_{\mathcal  Y}(x, x)}  \in \mathcal C_0(\mathcal  Y)\}$,
    since $ k_{\mathcal  Y} \in C_0(\mathcal  Y \times \mathcal  Y)$, and
    $\epsilon = 0.5$. On this neighborhood, we have:
    \begin{equation*}
    \int k(y, y') \mathrm{d} q(y) \mathrm{d} q(y') \leq \int \sqrt{k(y, y) \times k(y', y')} \mathrm{d} q(y) \mathrm{d} q(y') \leq 0.5^2 < 1,
    \quad \forall q \in \mathcal V
        \end{equation*}
    showing the
    continuity of the map in question.
    As a consequence, the image of $ B_\mathcal M(S)(0, 1) $ by the map $q \longmapsto \mu_{q}$
    is compact, implying from \cite{Christmann2010-ox} that the kernel 
    \begin{equation*}
        \tilde{k}(f, g) \coloneqq \exp(-\frac{1}{2 \sigma^2}\|f - g\|_{{\mathcal  H}}^2)
    \end{equation*}
    is universal on that set.
    Thus, we have shown that $\tilde{k}$ is universal on $\mathcal  H$ under the strong topology (e.g. the norm topology in $\mathcal H$).
    This is equivalent to the $ \mathrm{TV} $ topology of $\mathcal P(S)$ since $k$ has discrete masses by proposition 9 of \cite{Amin2023-er},
    and thus $ k_{\mathcal  P(\mathcal  Y)} $ is universal on $ \mathcal  P(\mathcal  Y) $.
\end{proof}

\subsection{Proofs regarding the impact of approximate kernels}\label{sec:proofs_approx_kernels}

To prove the convergence of the ACMMD--Rel estimator and the validity of its test,
we rely on an augmented U-statistics formulation. Let:
\begin{equation*} 
 U \coloneqq (Q_{|X}, \tilde{Y}^{1}, \dots, \tilde{Y}^{R}, \tilde{Y}, Y) 
    \sim \mathbb{ P }_{Q_{|X}} \otimes \mathbb{ Q }_{|}^{\otimes r} \otimes \mathbb{ Q }_{|} \otimes \mathbb{ P }^{Q}_{|} 
 \coloneqq \mathbb{ U }
\end{equation*}
$ U $ is the random variable which, for each model $ Q_{|X} $, concatenates
the synthetic samples $ (\tilde{Y}^{1}, \dots, \tilde{Y}^{R}) $ used to perform the kernel
approximation $ \widehat{ k }_{\mathcal  P(\mathcal  Y)}(q, q') $,
the synthetic sample $ \tilde{Y} $ used to evaluate $ h $, and $ \tilde{Y} $,
a sample from $ \mathbb{ P }^{Q}_{|} $ the conditional distribution of $ Y $ given $ Q_{|X} $.
Then, given $ N $ realizations $ \{  U_i \}_{i=1}^{N}  $, of $ U $, the estimator
$ \widehat{ \operatorname{ACMMD} }\operatorname{--Rel}\vphantom{a}^2$ can be written as a U-statistics
on the $ \{ U_i \}_{i=1}^{N} $
\begin{equation*} 
\widehat{ \operatorname{ACMMD} }\operatorname{--Rel}\vphantom{a}^{2} = \frac{2}{N(N-1)} \sum_{1 \leq i < j \leq N} h_a(U_i, U_j)
\end{equation*}
where
\begin{equation*} 
    h_a(U_i, U_j) \coloneqq 
        \widehat{ k }(\{ \tilde{Y}^r_i \}_{r=1}^{R}, \{ \tilde{Y}^r_j \}_{r=1}^{R}) 
        \times (
            k_{\mathcal  Y}(Y_i, Y_j) + 
            k_{\mathcal  Y}(\tilde{Y}_i, \tilde{Y}_j) 
            - k_{\mathcal  Y}(Y_i, \tilde{Y}_j) 
            - k_{\mathcal  Y}(\tilde{Y}_i, Y_j)
        )
\end{equation*}
We also will note 
\begin{equation*} 
\operatorname{ACMMD}\vphantom{a}^{2}_\mathrm{a} \coloneqq \mathbb{ E }_{U_1, U_2 \sim \mathbb{ U } \otimes \mathbb{ U }}\,\, h_a(U_1, U_2)
\end{equation*}

\subsubsection{Proof of Proposition \ref{prop:acmmd-rel-test-validity}}\label{sec:proof_approx-test-validity}
\begin{algorithm}[tb]
   \caption{ACMMD--Rel Conditional Goodness of fit Test}
   \label{alg:acmmd-rel-test}
\begin{algorithmic}
   \STATE {\bfseries Input:} $\{X_i, Y_i,
   \tilde{Y}_i\}_{i=1}^N \stackrel{\text{i.i.d.}}{\sim} \mathbb P_{X} \otimes
   \mathbb{ P }_{|} \otimes Q_{|}$
   \STATE \textbf{Parameters:} Level $\alpha$,
   kernel $k_{\mathcal  X}$, kernel $ k_{\mathcal  Y}$
   \STATE \texttt{// Estimate $\operatorname{ACMMD--Rel}$ using Algorithm \ref{alg:ACMMD-Rel}} and collect the $ \widehat{ h }(Z_i, Z_j) $ of \cref{eq:acmmd-rel-estimator}
   \STATE $\widehat{\operatorname{ACMMD}}\operatorname{--Rel}\vphantom{a}^2, \{ h(Z_i, Z_j) \}_{1\leq i < j \leq N} \leftarrow \texttt{estimate\_acmmd\_rel}(\{X_i, Y_i, \tilde{Y}_i\}_{i=1}^{N}, Q_{|X_i})$
   \STATE $ \lbrack W^b_i \sim \text{Rademacher} \texttt{ for } i \in 1, \dots, N \texttt{ for } b \in 1, \dots, B \rbrack $
   \STATE $ \lbrack
   \widetilde{\operatorname{ACMMD}}\vphantom{a}^2_b  \leftarrow  \frac{2}{N(N-1)} \sum\limits_{1 \leq i < j \leq N} W^{b}_i W^{b}_j \widehat{ h }(Z_i, Z_j),
    \texttt{ for } b  \texttt{ in } 1, \dots, B \rbrack $
    \STATE \texttt{// See \cref{sec:decision_rule} for how to compute $ \widehat{q}_{1-\alpha} $}
    \STATE $\widehat{q}_{1-\alpha} \leftarrow$ approx.  $(1- \alpha)$-quantile 
    of $\{ \widetilde{\operatorname{ACMMD}}\vphantom{a}^2_b \}_{b=1}^{B}$
    \IF {$ \widehat{ \operatorname{ACMMD} }\vphantom{a}^2 \leq \widehat{q}_{1-\alpha}$}
        \STATE {Fail to reject $H_0$}
    \ELSE
        \STATE {Reject $H_0$}
    \ENDIF
\end{algorithmic}
\end{algorithm}
With this formalism, we now prove that the ACMMD--Rel test has the specified type-I error rate of $ \alpha \in (0, 1) $,
e.g. rejects $ H_0 $ when $ \mathbb{ P }_{|x} = Q_{|x} $ with probability $ \alpha $.
Indeed, straightforward adaptations of the arguments in \cref{sec:proof-wild-bootstrap-acmmd} show that that doing a wild 
bootstrap using the $ \widehat{ h }(Z_i, Z_j) $ is equivalent to estimating 
\begin{equation*} 
\frac{2}{N(N-1)}\sum\limits_{ 1\leq i < j \leq N} h_a(U^{\sigma}_i, U^{\sigma}_j) \coloneqq \widehat{ \operatorname{ACMMD} }\operatorname{--Rel}\vphantom{a}^2_{\sigma_b}
\end{equation*}
where $ U_i^{\sigma} \coloneqq (Q_{|X_i}, \tilde{Y}^1_i, \dots, \tilde{Y}^R_i, \tilde{Y}^{\sigma}_i, Y_i^{\sigma})$,
where $(Y_i^{\sigma}, \tilde{Y}^{\sigma}_i) = (Y_i, \tilde{Y}_i) $ if $ \sigma(i) = 1 $ and $ (\tilde{Y}_i, Y_i) $ otherwise.
The same argument to show that ACMMD test has the desired type-I error rate follows in this case too:
Under $ \mathbb{ P }_{|x} = Q_{|x} $, the sequence
$ \{ \widehat{ \operatorname{ACMMD} }\operatorname{--Rel}\vphantom{a}^2_{\sigma_b} \}_{b=0}^{B}  $ is exchangeable
(noting $ \widehat{\operatorname{ACMMD} }\operatorname{--Rel}\vphantom{a}^2_{\sigma_0} = \widehat{\operatorname{ACMMD} }\operatorname{--Rel}\vphantom{a}^2 $, 
e.g. $ \sigma_0(i) = 1 $ for all $ 1 \leq i \leq N$), and we can repeat the derivations of the proof 
of \cref{lemma:level_acmmd} to show that the ACMMD--Rel test has the desired type-I error rate.

\subsubsection{Proof of Proposition \ref{prop:acmmd-rel-estimator-consistency}}\label{sec:proof_acmmd-rel-estimator-consistency}
We prove a slightly more general version of the proposition, for
kernels of the form $ \phi(d(q, q')^2) $, where $ \phi $ is a Lipschitz
function and $ d $ is a distance on $ \mathcal P(\mathcal Y) $. Setting $ \phi
= e^{- \frac{ \cdot }{ \sigma^2 }} $, we recover the kernels of
\cref{prop:acmmd-rel-estimator-consistency}, which include the exponentiated MMD kernel.
\begin{proposition}
    Assume that $ k_{\mathcal  Y} $ and $ k_{\mathcal
    P(\mathcal  Y)} $ is a kernel of the form $k_{\mathcal P(\mathcal  Y)}(q,
    q') = \phi(d(q, q'))$, for a Lipschitz function $ \phi $ and a function $
    d(q, q') $  admitting an unbiased estimator  of the form $
    \widehat{ d }(\{y_{1}^{r}\}_{r=1}^{R}, \{y_{2}^{r}\}_{r=1}^{R} ) $ where
    $ \{y_{1}^{r}\}_{r=1}^{r} $ and $ \{y_{2}^{r}\}_{r=1}^{r} $ are $ R $ i.i.d
    samples of $ q $ and $ q' $ respectively, with variance $ O(\frac{1}{R}) $
    (the bound in uniform in $ q $ and $ q' $).
    Then, assuming $ R \equiv R(N) $, with $ \lim\limits_{ N  \to \infty }R(N) = +\infty $,
    $ \widehat{\operatorname{ACMMD}}\operatorname{--Rel}\vphantom{a}^2 $
    converges in probability to $
    \operatorname{ACMMD--Rel}\vphantom{a}^2 $ as  $ N  \to \infty $.
\end{proposition}
\begin{proof}
    As discussed above, the estimator $ \widehat{ \operatorname{ACMMD} }\operatorname{--Rel}\vphantom{a}^2 $ can be written as a U-statistics
    on the $ \{ U_i \}_{i=1}^{N}  $, where $ U_i = (Q_{|X_i}, \tilde{Y}^{1}_i, \dots, \tilde{Y}^R_i, \tilde{Y}_i, Y_i)$,
    and using the kernel $ h_a $ defined as (accounting approximating $ k_{\mathcal  Y} $ through $ d $ directly)
    \begin{equation*} 
    h_a(U_i, U_j) \coloneqq \phi( \widehat{d}(\{ \tilde{Y}^r_i \}_{r=1}^{R}, \{ \tilde{Y}^r_j \}_{r=1}^{R})) \times (
        k_{\mathcal  Y}(Y_i, Y_j) + k_{\mathcal  Y}(\tilde{Y}_i, \tilde{Y}_j) 
        - k_{\mathcal  Y}(Y_i, \tilde{Y}_j) - k_{\mathcal  Y}(\tilde{Y}_i, Y_j)
        )
    \end{equation*}
    e.g. 
    \begin{equation*} 
    \begin{aligned}
        \widehat{\operatorname{ACMMD}}\operatorname{--Rel}\vphantom{a}^2 = \frac{2}{N(N-1)}\sum\limits_{ i < j }^{  } h_a(U_i, U_j)
    \end{aligned}
    \end{equation*}
    To study the convergence in probability
    of $ \widehat{ \operatorname{ACMMD} }\vphantom{a}^2_a $ to $ \mathrm{\operatorname{ACMMD}}\vphantom{a}^2_a $,
    we use finite-sample bounds on $ U $-statistcs \citet{hoeffding1994sequences}:
    \begin{equation*} 
    \begin{aligned}
        P\left ( \lvert \widehat{ \operatorname{ACMMD} }\vphantom{a}^2_a - \operatorname{ACMMD}\vphantom{a}^2_a \rvert > \left \| h_a \right \|_{\infty} \sqrt { \frac{ \log(2/\delta) }{ 2 \lfloor N / 2 \rfloor } } \right ) \leq \delta
    \end{aligned}
    \end{equation*}
    for all $ \delta > 0 $, where, by assumption, $ k_{\mathcal  Y} $ bounded, and $ \widehat{ k
    }_{\mathcal  P(\mathcal  Y)}  $ is of the form $ \phi( \widehat{ d }(\{ y_1^r \}_{r=1}^R, \{ y_2^r \}_{r=1}^R)) $ for some
    bounded function $ \phi $, implying that $ h_a $ is bounded.
    To show the dependence in $ R $, we bound the difference $
    \operatorname{ACMMD}_a $ and $ \operatorname{ACMMD} $.
    \begin{align*}
    \lvert \operatorname{ACMMD}\vphantom{a}^2_{a} - \operatorname{ACMMD}\vphantom{a}^2 \rvert 
        &= \mathbb{ E }_{\mathbb{ U }, \mathbb{ U }}  \left \lbrack   ( 
                \widehat{ k }_{\mathcal  P(\mathcal Y)}(\{ Y_1^r \}_{r=1}^{R}, \{ Y_2^r\}_{r=1}^{R} ) - k_{\mathcal  P(\mathcal  Y)}(Q_{|X_1}, Q_{X_2})
            ) \times (k_{\mathcal  Y}(Y_1, Y_2) + k_{\mathcal  Y}(\tilde{Y}_1, \tilde{Y}_2) \ \right .  \\
        & \quad \quad \quad \quad   \left . 
            - k_{\mathcal  Y}(Y_1, \tilde{Y}_2) - k_{\mathcal  Y}(\tilde{Y}_1, Y_2))\right \rbrack  \\
        &\leq 4 \left \|k_{\mathcal  Y}\right \|_{\infty} \lvert \mathbb{ E }_{\mathbb{ P }_{Q_{|}} \otimes \mathbb{ Q }^{\otimes r} \times \mathbb{ P }_{Q_{|}} \otimes \mathbb{ Q }^{\otimes r}}
            \widehat{ k }_{\mathcal  P(\mathcal  Y)}(\{ Y_1^r \}_{r=1}^{R}, \{ Y_2^r \}_{r=1}^{R} ) - k_{\mathcal P(\mathcal  Y)}(Q_{|X_1}, Q_{X_2}) \rvert \\
        &\leq  4 \left \|k_{\mathcal  Y}\right \|_{\infty} \left \| \phi \right \|_{\mathrm{Lip}} 
            \mathbb{ E }_{\mathbb{ P }_{Q_{|}} \otimes \mathbb{ Q }^{\otimes r} \times \mathbb{ P }_{Q_{|}} \otimes \mathbb{ Q }^{\otimes r}} 
            \lvert \widehat{ d }(\{ Y_1^r \}_{r=1}^{R}, \{ Y_2^r \}_{r=1}^{R} ) - d(Q_{|X_1}, Q_{X_2}) \rvert \\
        &\leq  4 \left \|k_{\mathcal  Y}\right \|_{\infty} \left \| \phi \right \|_{\mathrm{Lip}} \times \\ 
        & \quad \quad 
            \mathbb{ E }_{\mathbb{ P }_{Q_{|}} \times \mathbb{ P }_{Q_{|}}} \left \lbrack  
            \mathbb{ E }_{\mathbb{ Q }^{\otimes r} \times \mathbb{ Q }^{\otimes r}} 
            \lvert \widehat{ d }(\{ Y_1^r \}_{r=1}^{R}, \{ Y_2^r \}_{r=1}^{R} ) - d(Q_{|X_1}, Q_{X_2}) \rvert \,\, \Big |\,\,  Q_{|X_1}, Q_{X_2} 
        \right \rbrack \\
        &\leq  4 \left \|k_{\mathcal  Y}\right \|_{\infty} \left \| \phi \right \|_{\mathrm{Lip}} \mathbb{ E }_{\mathbb{ P }_{Q_{|}} \times \mathbb{ P }_{Q_{|}}}
            \left \lbrack  \sqrt { \mathbb{ V }_{\mathbb{ Q }^{\otimes r} \times \mathbb{ Q }^{\otimes r}} \widehat{ d }(\{ Y_1^r \}_{r=1}^{R}, \{ Y_2^r \}_{r=1}^{R} )}  
                            \Big |\,\,  Q_{|X_1}, Q_{X_2} \right \rbrack
    \end{align*}
    Where the last inequality follows from Jensen's inequality and the
    unbiasedness of $ \widehat{ d } $. The result follows 
    by applying the assumption on the variance of $ \widehat{ d } $
    (a bound which we assume is uniform in $ Q_{|X){1}}, Q_{|X_{2}} $).
\end{proof}
The term $ \mathbb{ V }_{\mathbb{ Q }^{\otimes r} \times
\mathbb{ Q }^{\otimes r}} \left \lbrack   \widehat{ d }(\{ Y_1^{r} \}_{r=1}^{R}, \{
Y_2^{r} \}_{r=1}^{R} ) | Q_{|X_1}, Q_{|X_2}\right \rbrack  $ can be more precisely characterized depending on $
\widehat{ d } $. For instance, we have, when $ \widehat{ d } $ is a
U-statistics (for instance, using the MMD estimator of \citet[Lemma 6, Equation 4]{gretton2012kernel}),
that \citep[section 5.2.1]{serfling2009approximation} $ \mathbb{
V }_{\mathbb{ Q }^{\otimes r} \times \mathbb{ Q }^{\otimes r}} \widehat{ d
}(\{ Y_1^{r} \}_{r=1}^{R}, \{ Y_2^{r} \}_{r=1}^{R})<  \zeta(Q_1,
Q_{|X_2})/R $, where $ \zeta(Q_{|X_1}, Q_{|X_2}) \coloneqq \mathbb{ V }_{(Y_1, \tilde{Y}_1), (Y_2, \tilde{Y}_2) \sim Q_{|X_1} \otimes Q_{|X_2} }(\tilde{h}((Y_1, \tilde{Y}_1), (Y_2, \tilde{Y}_2)) $
and
$ \tilde{h}((Y_1, \tilde{Y}_1), (Y_2, \tilde{Y}_2)) = k_{\mathcal  Y}(Y_1, Y_2) + k_{\mathcal  Y}(\tilde{Y}_1, \tilde{Y}_2) - k_{\mathcal  Y}(Y_1, \tilde{Y}_2) - k_{\mathcal  Y}(\tilde{Y}_1, Y_2) $, which is uniformly bounded by $ 4 \left \|k_{\mathcal  Y}\right \|_{\infty} $ for bounded kernels.
Putting the two parts
together, we thus have that:
\begin{equation*} 
    P \left ( \left \{ \widehat{ \operatorname{ACMMD} }\vphantom{a}^2_{a} - \mathrm{ACMMD}\vphantom{a}^2 \right \} > 
    4 \left \| k_{\mathcal  Y} \right \|_{\infty} \sqrt { \frac{ \log(2/\delta) }{ 2 \lfloor N / 2 \rfloor }}
     + \frac{16 \left \|k_{\mathcal  Y} \right \|_{\infty}^{2} \left \| \phi \right \|_{\mathrm{Lip}}}{\sqrt {R} }  \right ) \leq \delta 
\end{equation*}
for all $ \delta > 0 $, showing the convergence in probability of $ \widehat{ \operatorname{ACMMD} }_{a} $ to $ \mathrm{ACMMD} $.

\subsection{Additional Details for ACMMD and ACMMD--Rel in the synthetic example}\label{sec:acmmd-proofs-synthetic}
\subsubsection{Derivations of ACMMD in the synthetic example}
We first prove that $ \operatorname{ACMMD} $ is proportional to  $ \Delta p $.
\begin{lemma}
    In the setting described in \cref{sec:toy-synthetic}, we have
\begin{equation*} 
        \mathrm{ACMMD}^2(\mathbb{ P }_{|},  Q_{|}) = C \times \Delta p^{2}
\end{equation*}

for 
\begin{equation*} 
        C \coloneqq \iint_{  }^{  }  k_{\mathcal  X}(p, p')2(1 - e^{-\lambda})
        \frac{ (1-2p)(1-2p')}{ 1 - 4pp'(1 + e^{-\lambda}) / 2 } 
         \left (
		  \frac{2p'e^{-\lambda}}{ 1 - 2p'e^{-\lambda} }
		  + \frac{2pe^{-\lambda}}{ 1 - 2pe^{-\lambda} }
		  + 1
		  \right ) \mathrm{d} \mathbb{ P }_X(p) \mathrm{d} \mathbb{ P }_X(p')
\end{equation*}
\end{lemma}

\begin{proof}

Recall that we have
\begin{align*}
	\operatorname{ACMMD}^2 &= \operatorname{MMD}^2(\mathbb{ P }_X \otimes \mathbb{ P }_{|}, \mathbb{ P }_X  \otimes Q_{|})^2
			       &= \int_{  }^{  }k_{\mathcal  X}(p, p')\left ( T_{11} + T_{22} - 2 T_{12} \right ) \text{d} \mathbb P_X(p) \text{d} \mathbb P_X(p')
\end{align*}
\noindent
where
\begin{equation*} 
	T_{12} = \int_{   }^{  }k_{\mathcal  Y}(y, y') p(y | p) q(y' | p')  \mathrm{d}(y) \text{d}(y')
\end{equation*}
and $ T_{22} $ and $ T_{12} $ are defined similarly. For a sequence $ y $, we
define the function $ \mathrm{len}$ given by $\mathrm{len}(y) \coloneqq
\min_{  } \left \{ i \in \mathbb{ N } | y_i = \text{STOP} \right \} $,
which intuitively returns the length of the sequence.
\paragraph{Computing $ T_{ij} $} 
As we will see, a lot of the computations are agnostic to whether we are computing $ T_{11}, T_{22} $ or $ T_{12} $.
Note that the exponentiated hamming distance kernel on $ \mathcal  Y $ writes as a product
\begin{equation*} 
	k_{\mathcal  Y}(y, y') = e^{-\lambda d_H(y, y')} = e^{-\lambda_y \sum_{ i=0 }^{ \infty } \delta(y_i \ne y'_i)} = \prod_{ i=0 }^{ \infty } e^{-\lambda \delta(y_i \ne y'_i)} 
		= \prod\limits_{ i=0 }^{ \max_{  }(\text{len}(y), \text{len}(y'))} e^{-\lambda \delta(y_i \ne y'_i)}
\end{equation*}
let us define the following events
\begin{align*}
	F(m) &\coloneqq \left \{ \min_{ }(\text{len}(y), \text{len}(y')) = m \right \} \\
	G(m, \delta m) &\coloneqq \left \{ \max_{ }(\text{len}(y), \text{len}(y')) = m + \delta m \right \} \\
\end{align*}
which we further break down as 
\begin{align*}
	F_1(m) &= \left \{  \mathrm{len}(y) = m \right \} \cap  \left \{  \text{len}(y') > m \right \} \\
	F_2(m) &= \left \{   \text{len}(y) > m \right \}\cap \left \{ \mathrm{len}(y') = m \right \}  \\
	F_3(m) &= \left \{  \mathrm{len}(y) = m \right \} \cap  \left \{  \mathrm{len}(y') = m \right \} \\
	\implies F(m) &= F_1(m) \cup F_2(m) \cup F_3(m)
\end{align*}
For which the following probabilities hold:
\begin{align*}
P(F_1(m)) &=  P(\text{len}(y) = m) \times P(\text{len}(y') > m) = \left ( (2p)^m \times (1 - 2p) \right) \times (2p')^{m + 1} \\
P(F_2(m)) &=  P(\text{len}(y') = m) \times P(\text{len}(y) > m) = \left ( (2p')^m \times (1 - 2p') \right) \times (2p)^{m + 1} \\
P(F_3(m)) &=  P(\text{len}(y') = m) \times P(\text{len}(y) = m) = \left ( (2p')^m \times (1 - 2p') \right) \times \left ( (2p)^m \times (1 - 2p) \right) \\
P(G(m, \delta m) | F_1(m)) &= P(\text{len}(y') = m + \delta m | \text{len}(y) = m, \text{len}(y') > m) = (2p')^{\delta m - 1} \times (1 - 2p') \delta_{(\delta m \geq  1)}  \\
P(G(m, \delta m) | F_2(m)) &= P(\text{len}(y) = m + \delta m | \text{len}(y') = m, \text{len}(y) > m) = (2p)^{\delta m - 1} \times (1 - 2p) \delta_{(\delta m \geq  1)}  \\
P(G(m, \delta m) | F_3(m)) &= \delta(\delta m = 0)
\end{align*}
Let us note 
\begin{equation*} 
	E(m, \delta m, i) \coloneqq  F_i(m) \cap G(m, \delta m)
\end{equation*}
We have that $ E(m, \delta_m, i) \cap E(m', \delta m', j) = \emptyset $ if $(m, \delta m, i) \ne (m', \delta m', j) $.
\begin{equation*} 
	\Omega = \bigcup_{ m=0 }^{ +\infty }\bigcup_{i=1}^{3}\bigcup_{ \delta m = 0 }^{ +\infty } E(m, \delta m, i)
\end{equation*}
Using the law of total probability, we have that Thus, using the law of total probability:
\begin{align*}
	T_{ij}(p, p') 
            &= \sum\limits_{ m=0 }^{ +\infty }\sum\limits_{ i=1 }^{ 3 }\sum\limits_{ \delta m=0 }^{ +\infty }\mathbb P(E(m, \delta m, i))  \mathbb{ E }(e^{-\lambda d_H(y, y')}  | E(m, \delta m, i)) \\
	    & = \sum\limits_{ m=0 }^{ +\infty }\sum\limits_{ i=1 }^{ 3 }\sum\limits_{ \delta m=0 }^{ +\infty }\mathbb P(F_i(m)\cap G(m, \delta m))  \mathbb{ E }(e^{-\lambda d_H(y, y')}  | E(m, \delta m, i)) \\
            &= \sum\limits_{ m=0 }^{ +\infty }\sum\limits_{ i=1 }^{ 3 }\mathbb P(F_i(m))  \sum\limits_{ \delta m = 0 }^{ +\infty }P(G(m, \delta m) | F_i(m)) \mathbb{ E }(e^{- \lambda d_H(y, y')}  | E(m, \delta m, i)) \\
            &= \sum\limits_{ m=0 }^{ +\infty }\sum\limits_{ i=1 }^{ 3 }\mathbb P(F_i(m))\mathbb{ E }(e^{-\lambda d_H(y_{:m}, y_{:m}')} | F_i(m))   \\ 
            &       
                \times \sum\limits_{ \delta m = 0 }^{ +\infty }P(G(m, \delta m) | F_i(m)) \mathbb{ E }(e^{-  \lambda d_H(y_{m+1:m+\max_{  }(\delta m, 1)}, y'_{m+1:m+\max_{  }(\delta m, 1))}}  | E(m, \delta m, i), p, p') \\
            &= \sum\limits_{ m=0 }^{ +\infty }\sum\limits_{ i=1 }^{ 3 }\mathbb P(F_i(m))\mathbb{  E }(e^{-\lambda d_H(y_{:m}, y_{:m}')} | F_i(m)) \times 
                    \sum\limits_{ \delta m = 0 }^{ +\infty }P(G(m, \delta m) | F_i(m)) e^{-\lambda (\max_{  }(0, \delta m - 1) + \delta(m > 0))} \\
            & = \sum\limits_{m=0}^{ +\infty } \sum\limits_{ i=1 }^{ 3 }\mathbb P(F_i(m))\mathbb{  E }(e^{-\lambda d_H(y_{:m}, y_{:m}')} | F_i(m)) 
                    \left (\prod\limits_{ i=1 }^{ \max(m - 1, 1) } \mathbb{ E }(e^{-\lambda \delta(y_i \ne y'_i)}  |  F_i(m))) \right)^{\delta(m \geq  2)} \\ 
            & 
                \times \sum\limits_{ \delta m = 0 }^{ +\infty }P(G(m, \delta m) | F_i(m)) e^{-\lambda  (\max_{  }(0, \delta m - 1) + \delta(m > 0)) } \\
\end{align*}
where we break down the factorized hamming distance over the sequence into the
sum of the hamming distances over each coordinate, and made use of the fact
that 
\begin{equation*} 
	d_H(y_{m:m+\delta m}, y_{m:m+\delta m}')  = \max_{  }(0, \delta m - 1) + \delta(m > 0)
\end{equation*}
conditioned on $ F_i(m) $ and
$ G(m, \delta m) $. The disjunction of cases is necessary in order to not
count the term $ 0^{th} $ term twice in the event when $ m = 0 $.
This representation is convenient since whenever $ m \geq  2 $, for any $  1 \leq  i \leq m - 1 $, 
\begin{equation*} 
	P(\delta(y_i, y_i') = 1 |  F_i(m)) = \frac{(p p') + (p p')}{(p + p) \times (p' + p')} = \frac{1}{2}  = P(\delta(y_i, y_i') = 0 | F_i(m))
\end{equation*}
meaning we have
\begin{align*}
T_{ij}(p, p') & = \sum\limits_{m=0}^{ +\infty } \sum\limits_{ i=1 }^{ 3 } \mathbb{ E}(e^{-\lambda \delta(y_0 \ne y'_0)} |F_i(m)) \mathbb{ P }(F_i(m)) \left ( \frac{ 1 + e^{-\lambda} }{ 2 } \right )^{\max_{  }(m-1, 0) } \\ 
	      & \times \sum\limits_{ \delta m = 0 }^{ +\infty }P(G(m, \delta m) | F_i(m)) e^{-\lambda (\max_{  }(0, \delta m - 1) + \delta(m > 0)) } \\
\end{align*}
Inserting the relevant event probabilities into the expression for $ T_{ij} $, we have
\begin{align*}
	T_{ij}(p, p') &= \sum\limits_{ m=0 }^{ +\infty } \left ( \frac{ 1 + e^{-\lambda} }{ 2 } \right )^{\max_{  }(m-1, 0)}  \\
		  & \times \left ( \mathbb{ E }(e^{-\lambda \delta(y_0 \ne y'_0)} | F_1(m))(2p)^{m} \times (1 - 2p) (2p')^{m+1}\times (1 - 2p') e^{-\lambda \delta(m>0)}
                    \sum\limits_{ \delta m = 1 }^{ +\infty } e^{- \lambda  (\delta m - 1)} (2p')^{\delta m - 1}  \right .\\
		  & + \left .      \mathbb{ E }(e^{-\lambda \delta(y_0 \ne y'_0)} | F_2(m))(2p')^{m} \times (1 - 2p') (2p)^{m+1} \times (1 - 2p)e^{-\lambda \delta(m>0)}
                    \sum\limits_{ \delta m = 1 }^{ +\infty } e^{- \lambda (\delta m - 1)} (2p)^{\delta m - 1}  \right . \\
		  & + \Big .       \mathbb{ E }(e^{-\lambda \delta(y_0 \ne y'_0)} | F_3(m))  (2p')^{m} \times (1 - 2p') (2p)^{m}(1-2p) \Bigg ) \\
                  &= \sum\limits_{ m=0 }^{ +\infty }  \left ( \frac{ 1 + e^{-\lambda} }{ 2 } \right )^{\max_{  }(m-1, 0)}  \\
		  & \times \left ( \mathbb{ E }(e^{-\lambda \delta(y_0 \ne y'_0)} | F_1(m)) (2p)^{m} \times 
                    (1 - 2p) (2p')^{m+1}\times (1 - 2p') e^{-\lambda \delta(m>0)}\sum\limits_{ \delta m = 0 }^{ +\infty } e^{- \lambda \delta m} (2p')^{\delta m}  \right .\\
		  & 
                    + \left .      \mathbb{ E }(e^{-\lambda \delta(y_0 \ne y'_0)} | F_2(m))(2p')^{m} \times 
                        (1 - 2p') (2p)^{m+1} \times (1 - 2p) e^{-\lambda \delta(m>0)}\sum\limits_{ \delta m = 0 }^{ +\infty } e^{- \lambda \delta m} (2p)^{\delta m}  \right . \\
		  & + \Big .       \mathbb{ E }(e^{-\lambda \delta(y_0 \ne y'_0)} | F_3(m))	(2p')^{m} \times (1 - 2p') (2p)^{m}(1-2p) \Bigg ) \\
		  &= \sum\limits_{ m=0 }^{ +\infty } \left ( \frac{ 1 + e^{-\lambda} }{ 2 } \right )^{\max_{  }(m-1, 0)}  \\
		  & \times \left ( \mathbb{ E }(e^{-\lambda \delta(y_0 \ne y'_0)} | F_1(m))(2p)^{m} \times (1 - 2p) (2p')^{m+1}\times (1 - 2p') 
                    \times \frac{e^{-\lambda  \delta(m > 0) }}{1 - 2p'e^{-\lambda} } \right .\\
		  & + \left .      \mathbb{ E }(e^{-\lambda \delta(y_0 \ne y'_0)} | F_2(m))(2p')^{m} \times (1 - 2p') (2p)^{m+1} \times (1 - 2p) 
                    \times \frac{e^{-\lambda \delta(m > 0) }}{1 - 2pe^{-\lambda} } \right . \\
		  & + \Big .       \mathbb{ E }(e^{-\lambda \delta(y_0 \ne y'_0)} | F_3(m))(2p')^{m} \times (1 - 2p') (2p)^{m}(1-2p) \Bigg ) \\
\end{align*}
Now, some simplifications arise when $ m \geq  1 $. Indeed, in that case, $
\mathbb{ E }(e^{-\lambda \delta(y_0, y'_0)} | F_i(m))$ is independent of $ i $.
Noting $ T_{ij}^{1}(p, p') $ the sum of the terms for $ m \geq 1 $, we thus have
\begin{align*}
T_{ij}^{1}(p, p') 
		  &= \sum\limits_{ m=1 }^{ +\infty }  \mathbb{ E }(e^{-\lambda \delta(y_0 \ne y'_0)} | F(m))\left ( \frac{ 1 + e^{-\lambda} }{ 2 } \right )^{m - 1}  \\
		  & \times \left ( (2p)^{m} \times (1 - 2p) (2p')^{m+1}\times (1 - 2p') \times \frac{e^{-\lambda}}{1 - 2p'e^{-\lambda} } \right .\\
		  & + \left . (2p')^{m} \times (1 - 2p') (2p)^{m+1} \times (1 - 2p) \times \frac{e^{-\lambda}}{1 - 2pe^{-\lambda} } \right . \\
		  & + \Big . (2p')^{m} \times (1 - 2p') (2p)^{m}(1-2p) \Bigg ) \\
\end{align*}
Noting $ A_{ij} $ the term  $ \mathbb{ E } (e^{-\lambda \delta(y_0 \ne y'_0)} | F(m)) $, which is constant for all $ m \geq 1  $
\begin{align*}
T_{ij}^{1}(p, p') 
		  &= A_{ij} \sum\limits_{ m=1 }^{ +\infty }  \left ( \frac{ 1 + e^{-\lambda} }{ 2 } \right )^{m - 1}  \\
		  & \times \left ( (2p)^{m} \times (1 - 2p) (2p')^{m+1}\times (1 - 2p') \times \frac{e^{-\lambda}}{1 - 2p'e^{-\lambda} } \right .\\
		  & + \left . (2p')^{m} \times (1 - 2p') (2p)^{m+1} \times (1 - 2p) \times \frac{e^{-\lambda}}{1 - 2pe^{-\lambda} } \right . \\
		  & + \Big . (2p')^{m} \times (1 - 2p') (2p)^{m}(1-2p) \Bigg ) \\
		  &= A_{ij}(1 - 2p) (1 - 2p') 4pp'(
		  \frac{2p'e^{-\lambda}   }{ 1 - 2p'e^{-\lambda} }
		  + \frac{2p e^{-\lambda}  }{ 1 - 2pe^{-\lambda} }
		  + 1
		  )\sum\limits_{ m =0 }^{ +\infty } (4pp'(1 + e^{-\lambda})/{2})^m \\
		  &= A_{ij} \times \frac{ (1-2p)(1-2p')4pp' }{ 1 - 4pp'(1 + e^{-\lambda}) / 2 } \left (
		  \frac{2p'e^{-\lambda}}{ 1 - 2p'e^{-\lambda} }
		  + \frac{2pe^{-\lambda}}{ 1 - 2pe^{-\lambda} }
		  + 1
		  \right ) \\
		  &= C \times A_{ij}
\end{align*}
where 
\begin{equation*} 
	C(p, p') = \frac{ (1-2p)(1-2p')4pp' }{ 1 - 4pp'(1 + e^{-\lambda}) / 2 } \left (
		  \frac{2p'e^{-\lambda}}{ 1 - 2p'e^{-\lambda} }
		  + \frac{2pe^{-\lambda}}{ 1 - 2pe^{-\lambda} }
		  + 1
		  \right )
\end{equation*}
is a constant that does not depend on $ i, j $.
We compute the $ m=0 $ sum, noted $ T_{ij}^{0}(p, p') $. We have
 \begin{align*}
	 T_{ij}^{0}(p, p')
		  & = \left ( \mathbb{ E }(e^{-\lambda \delta(y_0 \ne y'_0)} | F_1(0)) \times (1 - 2p) (2p')\times (1 - 2p') \times \frac{1}{1 - 2p'e^{-\lambda} } \right .\\
		  & \quad \quad + \left .     \mathbb{ E }(e^{-\lambda \delta(y_0 \ne y'_0)} | F_2(0)) \times (1 - 2p') (2p) \times (1 - 2p) \times \frac{1}{1 - 2pe^{-\lambda} } \right . \\
		  & \quad \quad + \Big .      \mathbb{ E }(e^{-\lambda \delta(y_0 \ne y'_0)} | F_3(0)) \times (1 - 2p') (1-2p) \Bigg ) \\
 \end{align*}
 And we need to compute the terms $ \mathbb{ E }(e^{-\lambda \delta(y_0 \ne y'_0)} | F_i(0)) $ individually.
\paragraph{i=1, i=2} 
For $ i=1 $, we must have $ y_0 \ne y'_0 $, since $ y_0 = \mathrm{STOP} $, and
$ \mathrm{len}(y') > 0 $. Thus, $ \mathbb{ E }(e^{-\lambda \delta(y_0 \ne y'_0)} |
F_1(0)) = e^{-\lambda} $. Similarly, $ \mathbb{ E }(e^{-\lambda \delta(y_0 \ne y'_0)} |
F_2(0)) = e^{-\lambda} $.
\paragraph{i=3}
In that case, we must have $ y_0 = y'_0 = \text{STOP} $, since $ \text{len}(y) = \text{len}(y') = 0 $.
Thus, $ \mathbb{ E }(e^{-\lambda \delta(y_0 \ne y'_0)} | F_3(0)) = 1 $.

\noindent
Putting this together, we have
\begin{equation*} 
        T_{ij}^{0}(p, p')
       	   = (1-2p)(1-2p')\left (   \frac{2p'e^{-\lambda}}{1 - 2p'e^{-\lambda} } + \frac{2pe^{-\lambda}}{1 - 2pe^{-\lambda} } + 1 \right )
\end{equation*}
With that notation, we have:
\begin{align*}
	\operatorname{ACMMD}^2(\mathbb{ P }_{|}, \mathbb{ Q }_{|}) &= \int_{  }^{  } k_{\mathcal  X}(p, p')C(p, p')(A_{11} + A_{22} - 2A_{12}) \text{d} \mathbb P_X(p) \text{d} \mathbb P_X(p') \\
							   & + \int_{  }^{  } k_{\mathcal  X}(p, p')(T^{0}_{11}(p, p') + T^{0}_{22}(p, p') - 2T^{0}_{12}(p, p')) \text{d} \mathbb P_X(p) \text{d} \mathbb P_X(p') \\
	                                                   &= \int_{  }^{  } k_{\mathcal  X}(p, p')C(p, p')(A_{11} + A_{22} - 2A_{12}) \text{d} \mathbb P_X(p) \text{d} \mathbb P_X(p') \\
\end{align*}
since $ T^{0}_{ij} $ does not depend on $ i, j $.
We can narrow the variation down even further: by noting
$ p_{ij}^{A} = P(\delta(y_i \ne y'_i) = 0 | F(m)) $ (resp $ p_{ij}^{B} = P(\delta(y_i \ne y'_i) = 0 | F(0)) $),
since $ \mathbb{  E }({e^{-\lambda \epsilon }}) = p(\epsilon = 0)(1 -e^{-\lambda}) + e^{-\lambda} $ if $ \epsilon $ is a Bernoulli random variable,
\begin{equation*} 
\operatorname{ACMMD}^2(\mathbb{ P }_{|}, \mathbb{ Q }_{|}) = \int_{  }^{  } k_{\mathcal  X}(p, p')C(p, p')(1 - e^{-\lambda})(p_{11}^{A} + p_{22}^{A} - 2p_{12}^{A}) \text{d} \mathbb P_X(p) \text{d} \mathbb P_X(p') \\
\end{equation*}
We now compute the probabilities $ p_{ij}^{A} $ for $ i, j \in \left \{ 1, 2 \right \} $.
In every case, such $ p_{ij}^{A} $ can be written as:

\begin{align*}
	p_{ij}^A &= \frac{ P(y_0 = y_0' = A) + P(y_0 = y_0' = B) }{ P(\{y_0 \in \left \{ A, B \right \}\}\cap \{y_0' \in \left \{ A, B \right \}\} ) }
		 &\frac{ P(y_0 = y_0' = A) + P(y_0 = y_0' = B) }{ 4pp'}
\end{align*}
and we have
\begin{align*}
p_{11}^{A} &= \frac{ pp' + pp' }{ 4pp' } = \frac{1}{2} \\
p_{22}^A &= \frac{ (p + \Delta p)(p' + \Delta p)  + (p - \Delta p)(p' - \Delta p) }{ 4pp'} = \frac{2pp' + 2 \Delta p^2}{4pp'} \\
p_{12}^A  &= \frac{ (p)(p' + \Delta p)  + (p)(p' - \Delta p) }{ 4pp'} = \frac{1}{2} \\
\implies p_{11}^{A} + p_{22}^{A} - 2p_{12}^{A} &= \frac{2pp' + 2 \Delta p^2}{4pp'} - \frac{1}{2} = \frac{2 \Delta p^2}{4pp'}
\end{align*}

\paragraph{Putting it together} 
We thus have
\begin{equation*} 
\mathrm{ACMMD}(\mathbb{ P }_{|}, \mathbb{ Q }_{|})  = \int_{  }^{  } C(p, p') k(p, p')(1 - e^{-\lambda})\frac{2 \Delta p^2}{4pp'}\text{d} \mathbb P_X(p) \text{d} \mathbb P_X(p') \\
\end{equation*}
Recalling that
\begin{equation*} 
	C(p, p') = \frac{ (1-2p)(1-2p')4pp' }{ 1 - 4pp'(1 + e^{-\lambda}) / 2 } \left (
		  \frac{2p'e^{-\lambda}}{ 1 - 2p'e^{-\lambda} }
		  + \frac{2pe^{-\lambda}}{ 1 - 2pe^{-\lambda} }
		  + 1
		  \right )
\end{equation*}
yields the desired result.
\end{proof}

\subsubsection{Closed-form \texorpdfstring{$\operatorname{ACMMD--Rel}$}{ACMMD-Rel} evaluation}

Assuming the same model, it is also possible to evaluate $ \operatorname{ACMMD--Rel}(\mathbb{ P }_{|},  Q_{|}) $ in closed form.
Indeed, $ \operatorname{ACMMD--Rel} $ becomes a special case of the ACMMD formula given above,
with the conditioned variable $ X $ set to be the models $ Q_{|X} $. It is thus possible to show:

\begin{lemma} We have
\begin{equation*} 
        \operatorname{ACMMD--Rel}^2(\mathbb{ P }_{|},  Q_{|}) = C \times \Delta p^{2}
\end{equation*}
for 
\begin{equation*} 
	C = \iint  k_{\mathcal  P(\mathcal  Y)}(q_{|p}, q_{|p'})2 (1 - e^{-\lambda}) \frac{ (1-2p)(1-2p')}{ 1 - 4pp'(1 + e^{-\lambda}) / 2 } \left (
		  \frac{2p'e^{-\lambda}}{ 1 - 2p'e^{-\lambda} }
		  + \frac{2pe^{-\lambda}}{ 1 - 2pe^{-\lambda} }
		  + 1
		  \right ) \mathrm{d} \mathbb P_X(p) \mathrm{d} \mathbb P_X(p')
\end{equation*}
\end{lemma}
\noindent 
The above lemma leaves the choice of the kernel $ k_{\mathcal  P(\mathcal  Y)} $ open: the tractability
of this expression will follow only if such kernel  can be tractably computed. In the next lemma,
we derive a closed form solution for $ k_{\mathcal  P(\mathcal  Y)}(q, q')$ when 
$ k_{\mathcal  P(\mathcal  Y)}(q, q') = e^{-  \frac{ \text{MMD}^2(q, q') }{ 2 \sigma^2 } } $,
where the MMD is computed with an Exponentiated Hamming kernel on $ \mathcal Y $.
\begin{lemma}
	We have
\begin{equation*} 
	\operatorname{MMD}^2(q_{|p}, q_{|p'})  = T(p, p) + T(p', p') - 2T(p, p')
\end{equation*}
Where
\begin{align*}
	T(p, p') &= C(p, p') A(p, p') + T^{0}(p, p') \\
	C(p, p') &= \frac{ (1-2p)(1-2p')4pp' }{ 1 - 4pp'(1 + e^{-\lambda}) / 2 } \left ( \frac{2p'e^{-\lambda}}{ 1 - 2p'e^{-\lambda} } + \frac{2pe^{-\lambda}}{ 1 - 2pe^{-\lambda} } + 1 \right ) \\
	A(p, p') &= \frac{ 2pp' + 2 \Delta p^2}{4pp'} \times (1 - e^{-\lambda}) + e^{-\lambda} \\
        T^{0}(p, p') & = (1-2p)(1-2p')\left (   \frac{2p'e^{-\lambda}}{1 - 2p'e^{-\lambda} } + \frac{2pe^{-\lambda}}{1 - 2pe^{-\lambda} } + 1 \right )
\end{align*}
\end{lemma}
\noindent
Combining the two lemmas allows us to obtain a computable expression for $ \operatorname{ACMMD--Rel}(\mathbb{ P }_{|},  Q_{|}) $.


%
\begin{figure}[htpb]
    \includegraphics[width=0.5\textwidth]{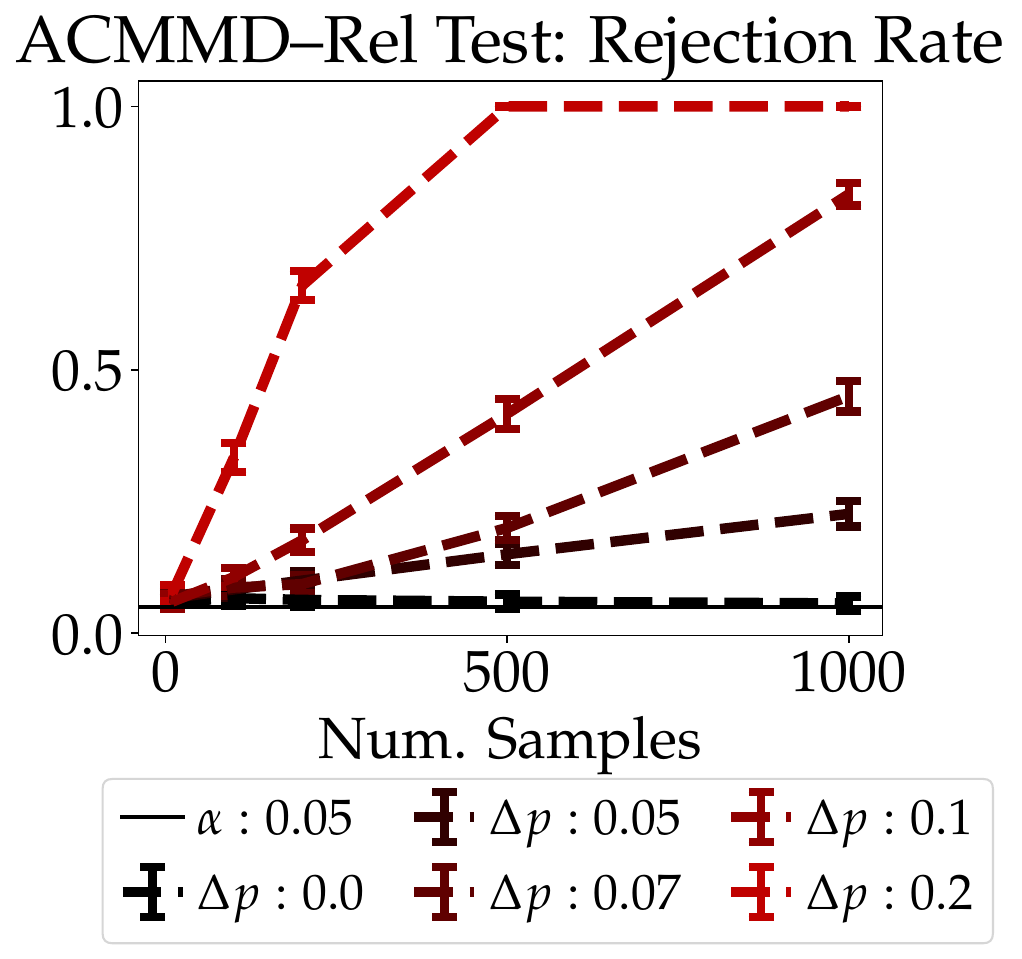}
    \includegraphics[width=0.46\textwidth]{synthetic_exps_plots_rebuttals/synthetic_mmd_cgof_power_cgof_composite.pdf}
    \includegraphics[width=0.5\textwidth]{synthetic_exps_plots_rebuttals/synthetic_mmd_cgof_unbiasedness_cgof_composite.pdf}
    \includegraphics[width=0.5\textwidth]{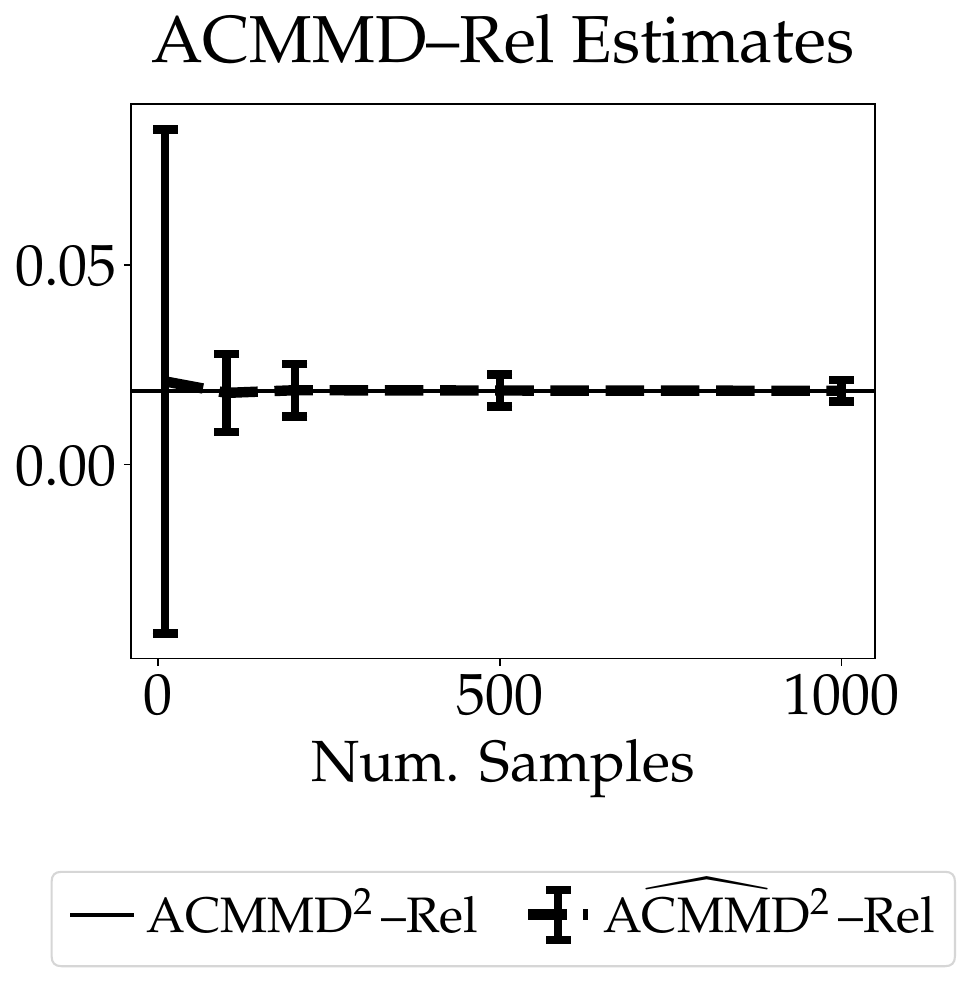}
    \caption{Top left: Rejection Rate of the ACMMD test as a function of the dataset size, for different values of $ \Delta p $. Top right: Rejection Rate of the ACMMD test as a function of $ \Delta p $, for different dataset sizes.
Bottom left: Estimated ACMMD as a function of the dataset size. Bottom right: Estimated ACMMD--Rel as a function of $ \Delta p $.
To compute these estimates, we use dataset sizes of
$ \left \{ 10, 100, 200, 500, 1000 \right \}  $, used $ m=5 $
atoms for the prior on $ p $ between $ p_1=0.3 $, $ p_2=0.45
$, used $ \lambda = 1 $, $ \Delta p = 0.25 $, and average over $ 300 $ runs.
In addition, we plot the true value  $ \operatorname{ACMMD}(\mathbb{ P }_{|}, Q_{|}) $
using the closed-form expression derived above. 
    }
    \label{fig:toy-rejection_rate-and-estimate}
\end{figure}

\newpage

\section{Additional Experiments}
\subsection{Additional Experiments for the semisynthetic ProteinMPNN data}

In addition to the figures of \cref{sec:discriminative}, which use $ T=0.1 $ to plot
the estimates and rejection rates of ACMMD and ACMMD--Rel on the ProteinMPNN
synthetic data, we provide here the same plots for $ T=1.0$ the value used to train ProteinMPNN.
We notice that detecting a given change in temperature is slightly simpler for $ T=1.0 $ than
for $ T=0.1 $.
\begin{figure}[htpb]
    \includegraphics[width=0.5\textwidth]{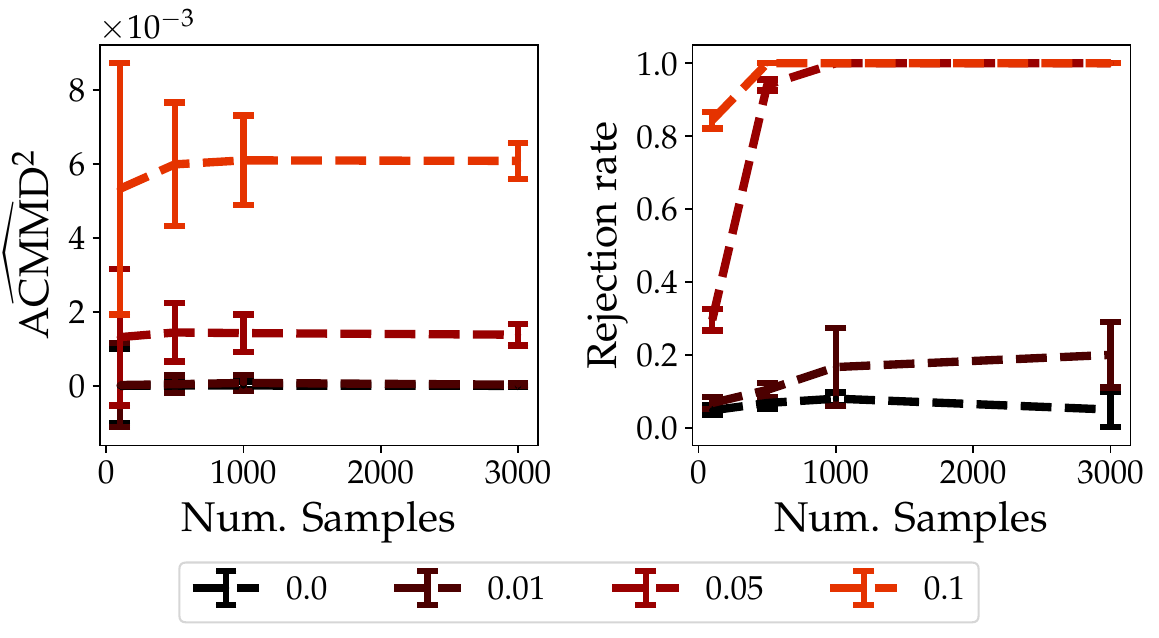}
    \includegraphics[width=0.5\textwidth]{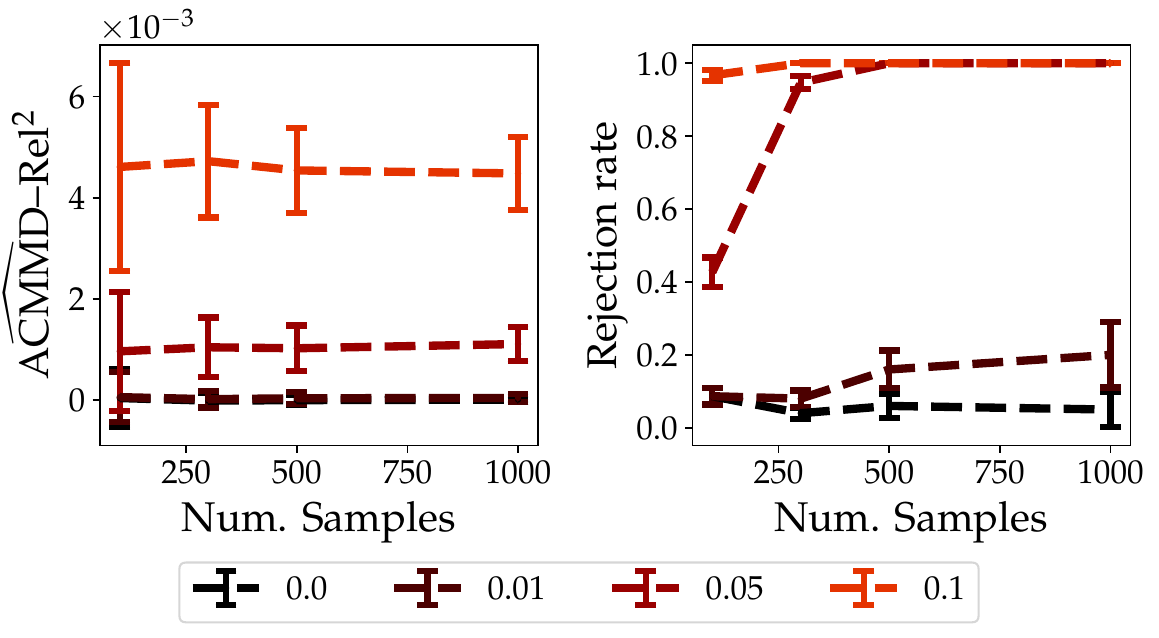}
    \includegraphics[width=0.5\textwidth]{figures/synthetic_exps_plots/temperature_0.1_type_cgof.pdf}
    \includegraphics[width=0.5\textwidth]{figures/synthetic_exps_plots/temperature_0.1_type_calibration.pdf}
    \caption{ACMMD and ACMMD--Rel estimates and rejection rate in the semisynthetic setting 
    of \cref{sec:discriminative}. The different lines indicate a different temperature shift
    between the two MPNN models. Top panel shows uses a base temperature of 
    $ T=1.0 $, while the bottom panel uses $ T=0.1 $.
    }
\end{figure}

\subsection{Additional Experiments for the structural superfamily evaluation}
We include in \cref{fig:superfamily_T_bar_cgof_calibration}
the values of $ \widehat{ \operatorname{ACMMD} }\operatorname{--Rel}\vphantom{a}^2 $
for different superfamilies (which was not included in \cref{sec:superfamily-evaluation}),
and compare it with the values of $ \widehat{ \operatorname{ACMMD} }\vphantom{a}^2 $.
In line with the hyperparameter tuning results of \cref{sec:proteinmpnn-whole-data-evaluation},
we notice that high temperature are highly detrimental from a reliability perspective.
Intuitively, increasing the temperature of ProteinMPNN makes the model
``underconfident''. Since a reliable model is neither over-- nor underconfident,
this decrease of confidence is penalized by ACMMD--Rel. This also shows that increasing
the temperature of a model does not make the model fallback to its prior distribution
(otherwise the model would be more reliable). Instead, it just increases the uncertainty
of the model in a detrimental fashion.
\begin{figure}[htbp]
    \includegraphics[width=.48\textwidth]{figures/superfamily_plots/superfamily_T_bar_cgof.pdf}
    \includegraphics[width=.48\textwidth]{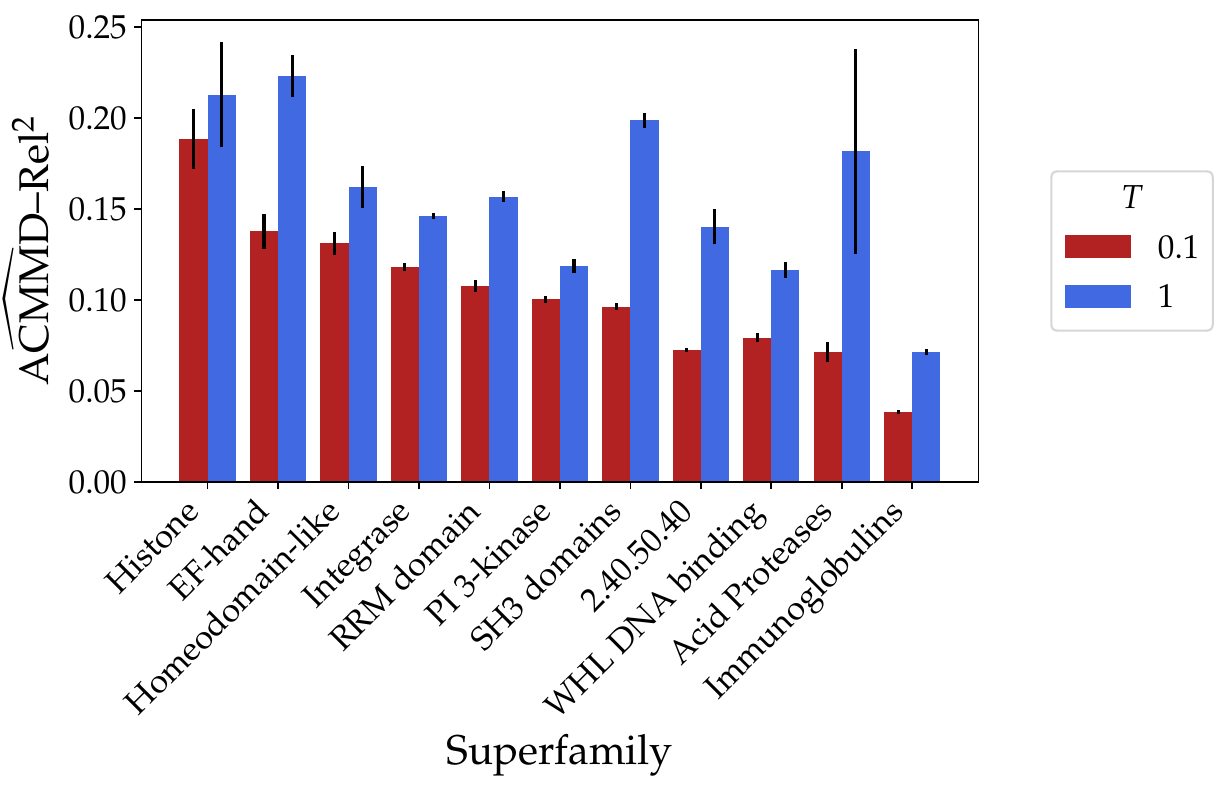}
    \caption{ Value of 
        $ \widehat{ \operatorname{ACMMD} }\vphantom{a}^2 $ (left)  and
        $ \widehat{ \operatorname{ACMMD} }\operatorname{--Rel}\vphantom{a}^2 $ (right)
        between ProteinMPNN and the CATH S60 reference dataset on a subset of 10 superfamilies for two 
        different temperatures $ T=1.0 $ and $ T=0.1 $.
    }
    \label{fig:superfamily_T_bar_cgof_calibration}
\end{figure}

\section{Known Kernels for protein sequences and structures} In the context of
inverse folding, computing the ACMMD requires a kernel on sequences $k_\mathcal
Y$ and a kernel on protein structures. This section contains a brief overview
of non neural-network based, known kernels for protein sequences and
structures. The main desiderata to achieve when choosing kernels for computing
goodness-of-fit criterion is to find kernels that are able to detect (up to
statistical noise) any deviation from a perfect fit between the model and the
data. Such kernels are referred to as \emph{universal}.

\paragraph{The protein formalism}

The most general formalism for the space protein structures is the set of
equivalence classes of graphs where the equivalence relationship is defined to
the existence of a graph isomorphism. The need for equivalence classes stems
from the fact that different labelling policies exist for a given protein,
meaning that a single protein can be associated to multiple graphs. However,
this labelling will in practice not be completely arbitrary: first, the
set of candidate labelling can be restricted to the ones consistent with
covalent bounds. But in the inverse folding problem, the setting is even
simpler: the protein structure is restricted to its backbone, which is
sequential by nature. This limits the set of covalent-bound consistent
labelling policies to two (the forward and the backward one), and my vague
understanding is that there is a terminal atom in protein, which suggests the
existence of a canonical direction: thus, only one labelling policy remain, and
protein structures can thus be associated to the set of atom locations
$\bigcup_{i=1}^{+\infty} \mathbb R^i$. This set differs from the set of protein
sequences $\bigcup_{i=1}^{+\infty} \mathcal A$ in that the ``alphabet'' is the
real line instead of a finite set of symbols. The restriction from the space of
graphs to the space of variable-length
sequences since there it is known that no graph kernels commonly in use are
even characteristic \cite{Kriege2020-xc}. The space $\bigcup_{i=1}^{+\infty}
\mathcal X$ (for arbitrary $\mathcal X$ have been investigated by the
time series community), which have developed a set of kernels to carry out data
analysis on it. We provide some background on such kernels below.

\paragraph{Background: alignment kernels for real-valued sequences of arbitrary length}
Alignment kernels \cite{cuturi2007kernel, cuturi2017soft, saigo2004protein,vert2004local} refer to a diverse set of variety of kernels
constitute a family of kernels on $\bigcup_{i=1}^{+\infty} \mathcal X^i$ that are computed based on aggregating the similarities between all possible ``alignment candidates'' between two inputs $x_1$ and $x_2$. There are two main subfamilies of alignment kernels, which both use slightly different alignment definitions: local alignment kernels, and global alignment kernels.

\paragraph{Local alignment kernels}
Local alignment kernels \cite{saigo2004protein, vert2004local} are kernels of the form
\begin{equation}
    k_{\text{LA}}(x, y) = \sum_{\pi \in \Pi(x, y)} \exp(\beta s(x, y, \pi))
\end{equation}
Where 
$$s(x, y, \pi) = \sum_{i=1}^{|\pi|}s(x_1^{(\pi_1(i))}, x_2^{(\pi_2(i))}) + \sum_{i=1}^{|\pi|-1} g(\pi_1(i+1) - \pi_1(i)) + g(\pi_2(i+1) - \pi_2(i))$$
and $ \Pi(x, y)$ is the set of all possible \emph{alignments} of $x$ and $y$, e.g. the set of all 2-tuple of $p$-long sequences
\begin{equation*}
\pi \coloneqq ((\pi_1(1), \dots, \pi_1(p)), (\pi_2(1), \dots, \pi_2(p))
\end{equation*}
where
\begin{equation*}
1 \leq \pi_1(1) < \pi_1(2) \dots < \pi_2(p) \leq n \\
1 \leq \pi_2(1) < \pi_1(2) \dots < \pi_2(p) \leq m \\
\end{equation*}
Importantly, local alignment kernels involve a gap function, and thus give a specific status to insertions and deletions, unlike global alignment kernels, as we will see below. The local alignment kernel can be seen as computing the (soft) minimum
of a discrepancy within the set of all possible alignments. The use of a soft minimum (and not a hard one) is crucial to ensure positive definiteness. Local alignment kernels seem to have been designed initially for finite alphabets target. When $g=0$, the necessary and sufficient condition on $s$ to ensure that $k_{\text{LA}}$ is a positive definite is for $s$ to be a conditionally positive definite kernel \footnote{A kernel is c.p.d if $\sum_{i, j=1}^{n} c_i c_j s(x^{(i)}, x^{(j)}) \geq 0 \forall c_1, \dots, c_n, c_1 + \dots + c_n = 0$.}.
This is in particular verified if $(s(x^i, y^i))_{1 \leq i, j \leq |\mathcal A|}$ is positive definite. I need further reading to investigate whether the case of infinite $\mathcal X$ was studied.

\paragraph{Global alignment kernels}

Global alignment kernels \cite{cuturi2007kernel, cuturi2017soft} also perform a softmin over alignment, but do not incorporate gaps in their score, and use a slightly different notion of alignment, namely:
\begin{equation*}
\pi \coloneqq ((\pi_1(1), \pi_2(1), \dots, (\pi_1(p), \pi_2(p))
\end{equation*}
where now, the constraints on $\pi_1$ and $\pi_2$ are
\begin{equation*}
\begin{aligned}
1 = \pi_1(1) < \pi_1(2) \dots < \pi_2(p) = n \\
1 = \pi_2(1) < \pi_1(2) \dots < \pi_2(p) = m \\
\pi_1(i+1) \leq \pi_1(i) + 1 \quad \text{unitary increments} \\
(\pi_1(i+1) - \pi_1(i))  + (\pi_2(i+1) - \pi_2(i)) \geq 1 \quad \text{no repetitions}
\end{aligned}
\end{equation*}
Unlike the previous alignment definition, this one explicitly maps each item in each sequence with another item in the other sequence, and does not try to account for potential gaps. Let us call $\mathcal A$ the set of all alignment. The final definition for a global alignment kernel is then:

\begin{equation}\label{eq:roto-translation-invariance}
    k_{\text{GA}}(x, y) = \sum_{\pi \in \mathcal A(x, y)} \exp(\sum_{i=1}^{\pi} s(x_1^{(\pi_1(i))}, x_2^{(\pi_2(i))})
\end{equation}
As stated in \citet[Theorem 1]{cuturi2007kernel}, $k_{\text{GA}}$ will be positive definite if $k(x, y) \coloneqq \exp(s(x, y))$ is a positive definite kernel such that $\frac{k}{(1 - k)}$ is positive definite.

\end{document}